\documentclass[twoside,12pt]{article}

\usepackage{fullpage}
\usepackage{natbib}
\usepackage[hypertexnames=false]{hyperref}
\hypersetup{
	colorlinks,
	linkcolor={red!50!black},
	citecolor={blue!50!black},
	urlcolor={blue!80!black}
}
\usepackage{times}
\usepackage{lastpage}
\usepackage{amsmath,amsfonts,bm}
\usepackage{amsthm}
\usepackage{xcolor}

\usepackage{sty/OCO}
\usepackage[algo2e,ruled]{algorithm2e}
\usepackage{algorithm,algcompatible}
\usepackage{pdflscape}
\usepackage{subfigure}
\usepackage{tabularx}
\usepackage{multirow}

\let\origlog\log  
\let\log\ln 
\let\Tr\tr  

\DeclareRobustCommand{\VAN}[3]{#2} 

\newcommand{\metagradfull}{MetaGrad Full}
\newcommand{\metagradsketch}{MetaGrad Sketch}
\newcommand{\metagraddiag}{MetaGrad Coord}
\newcommand{\rtrick}{\tilde{r}}     %
\newcommand{\Rtrick}{\tilde{R}}     %
\newcommand{\activeset}{\mathcal{A}}     %
\renewcommand{\U}{\mathcal{W}}      %
\newcommand{\approxSigma}{\widetilde{\Sigma}}
\newcommand{\approxgrad}{\grad}  %
\newcommand{\clipgrad}{\bar \grad}  %
\newcommand{\clipsurr}{\bar \surr}  %
\newcommand{\Intersect}{\bigcap}    %

\newcommand{\lambdamax}{\lambda_\text{max}}

\DeclareBoldMathCommand{\X}{X}
\DeclareBoldMathCommand{\V}{V}
\DeclareBoldMathCommand{\W}{W}
\DeclareBoldMathCommand{\bLambda}{\Lambda}
\DeclareBoldMathCommand{\H}{H}
\DeclareBoldMathCommand{\B}{B}
\DeclareBoldMathCommand{\F}{F}
\DeclareBoldMathCommand{\G}{G}
\DeclareBoldMathCommand{\r}{r}
\DeclareBoldMathCommand{\q}{q}
\DeclareBoldMathCommand{\g}{g}
\DeclareBoldMathCommand{\reallyU}{U}
\DeclareBoldMathCommand{\0}{0}
\DeclareMathOperator{\sign}{sign}
\DeclareMathOperator{\shrink}{s}

\newcommand{\sumT}{\sum_{t=1}^T}

\makeatletter
\newcommand{\multiline}[1]{%
  \begin{tabularx}{\dimexpr\linewidth-\ALG@thistlm-1.5cm}[t]{@{}X@{}}
    #1
  \end{tabularx}
}
\makeatother
\newlength{\commentWidth}
\newcommand{\simplecomment}[1]{\ensuremath{\triangleright}~\textit{#1}}
\newcommand{\spacycomment}[1]{\hfill
\makebox[\commentWidth][l]{\simplecomment{#1}}\hspace{0.6cm}\mbox{}}

\newtheorem{theorem}{Theorem}
\newtheorem{lemma}{Lemma}
\newtheorem{definition}{Definition}
\newtheorem{corollary}{Corollary}
\newtheorem{proposition}{Proposition}

\begin{document}

\title{MetaGrad: Adaptation using Multiple Learning Rates\newline in
Online Learning\footnote{An earlier conference version of this paper
appeared at NeurIPS 2016 \citep{metagrad}.}}

\author{Tim van Erven \textit{tim@timvanerven.nl} \\
       Korteweg-de Vries Institute for Mathematics\\
       University of Amsterdam\\
       Science Park 107, 1098 XG Amsterdam, The Netherlands
       \and 
       Wouter M. Koolen \textit{wmkoolen@cwi.nl} \\
       Centrum Wiskunde \& Informatica\\
       Science Park 123, 1098 XG Amsterdam, The Netherlands
       \and
       Dirk van der Hoeven \textit{dirk@dirkvanderhoeven.com} \\
       Mathematical Institute\\
       Leiden University\\
       Niels Bohrweg 1, 2300 RA Leiden, The Netherlands}
       
\maketitle

\begin{abstract}%
We provide a new adaptive method for online convex optimization,
MetaGrad, that is robust to general convex losses but achieves faster
rates for a broad class of special functions, including exp-concave and
strongly convex functions, but also various types of stochastic and
non-stochastic functions without any curvature. We prove this by drawing
a connection to the Bernstein condition, which is known to imply fast
rates in offline statistical learning. MetaGrad further adapts
automatically to the size of the gradients. Its main feature is that it
simultaneously considers multiple learning rates, which are weighted
directly proportional to their empirical performance on the data using a
new meta-algorithm. We provide three versions of MetaGrad. The full
matrix version maintains a full covariance matrix and is applicable to
learning tasks for which we can afford update time quadratic in the
dimension. The other two versions provide speed-ups for high-dimensional
learning tasks with an update time that is linear in the dimension: one
is based on sketching, the other on running a separate copy of the basic
algorithm per coordinate. We evaluate all versions of MetaGrad on
benchmark online classification and regression tasks, on which they
consistently outperform both online gradient descent and AdaGrad.
\end{abstract}

\section{Introduction}

Methods for \emph{online convex optimization} (OCO)
\citep{ShalevShwartz2012,Hazan2016} make it possible to optimize
parameters sequentially, by processing convex functions in a streaming
fashion. This is important in time series prediction where the data are
inherently online; but it may also be convenient to process offline data
sets sequentially, for instance if the data do not all fit into memory
at the same time or if parameters need to be updated quickly when extra
data become available.

The difficulty of an OCO task depends on the convex functions
$f_1,f_2,\ldots,f_T$ that need to be optimized. The argument of these
functions is a $d$-dimensional parameter vector $\w$ from a convex
domain $\U$. Although this is abstracted away in the general framework,
each function $f_t$ usually measures the loss of the parameters on an
underlying example $(\x_t,y_t)$ in a machine learning task. For example,
in classification $f_t$ might be the \emph{hinge loss} $f_t(\w) =
\max\{0,1-y_t \w^\top \x_t\}$ or the \emph{logistic loss} $f_t(\w) =
\log\del*{1 + e^{-y_t \w^\top \x_t}}$, with $y_t \in \{-1,+1\}$. Thus
the difficulty depends both on the choice of loss and on the observed
data.

There are different methods for OCO, depending on assumptions that can
be made about the functions. The simplest and most commonly used
strategy is \emph{online gradient descent} (OGD). OGD updates parameters
$\w_{t+1} = \w_t - \eta_t \nabla f_t(\w_t)$ by taking a step in the
direction of the negative gradient, where the step size is determined by
a parameter $\eta_t$ called the \emph{learning rate}. The goal is to
minimize the \emph{regret}
\[
  R_T^\u = \sum_{t=1}^T f_t(\w_t) - \sum_{t=1}^T f_t(\u)
\]
over $T$ rounds, which measures the difference in cumulative loss
between the online iterates $\w_t$ and the best offline parameters $\u$.
For learning rates $\eta_t \propto 1/\sqrt{t}$, OGD guarantees that the
regret for general convex Lipschitz functions is bounded by
$O(\sqrt{T})$ \citep{Zinkevich2003}. Alternatively, if it is known
beforehand that the functions are of an easier type, then better regret
rates are sometimes possible. For instance, if the functions are
\emph{strongly convex}, then logarithmic regret $O(\log T)$ can be
achieved by OGD with much smaller learning rates $\eta_t \propto 1/t$
\citep{ons}, and, if they are \emph{exp-concave}, then logarithmic
regret $O(d \log T)$ can be achieved by the \emph{Online Newton Step}
(ONS) algorithm \citep{ons}.

This partitions OCO tasks into categories, leaving it to the user to
choose the appropriate algorithm for their setting. Such a strict
partition, apart from being a burden on the user, depends on an
extensive cataloguing of all types of easier functions that might occur
in practice. (See Section~\ref{sec:fastRateExamples} for several ways in
which the existing list of easy functions can be extended.) It also
immediately raises the question of whether there are cases in between
logarithmic and square-root regret (there are, see
Theorem~\ref{thm:Bernstein} in Section~\ref{sec:fastRateExamples}), and
which algorithm to use then. And, third, it presents the problem that
the appropriate algorithm might depend on (the distribution of) the data
(again see Section~\ref{sec:fastRateExamples}), which makes it entirely
impossible to select the right algorithm beforehand. 

These issues motivate the development of \emph{adaptive} methods, which
are no worse than $O(\sqrt{T})$ for general convex functions, but also
automatically take advantage of easier functions whenever possible. An
important step in this direction are the adaptive OGD algorithm of
\citet{BartlettHazanRakhlin2007} and its proximal improvement by
\citet{Do2009}, which are able to interpolate between strongly convex
and general convex functions if they are provided with a data-dependent
strong convexity parameter in each round, and significantly outperform
the main non-adaptive method (i.e.\ Pegasos by
\citealt{Shalev-ShwartzEtAl2011Pegasos}) in
the experiments of \citeauthor{Do2009}. Here we consider a significantly richer
class of functions, which includes exp-concave functions, strongly
convex functions, general convex functions that do not change between
rounds (even if they have no curvature), and stochastic functions whose
gradients satisfy the so-called Bernstein condition, which is well-known
to enable fast rates in offline statistical learning
\citep{BartlettMendelson2006,VanErven2015FastRates,bernfast}.
The latter group can again include functions without curvature, like the
unregularized hinge loss. All these cases are covered simultaneously by
a new adaptive method we call \emph{MetaGrad}, for \underbar{m}ultiple
\underbar{eta} \underbar{grad}ient algorithm. Assuming that the radius
of the domain $\U$ and the $\ell_2$-norms of the gradients $\grad_t =
\nabla f_t(\w_t)$ are both bounded by constants,
Theorem~\ref{thm:mainbound} and Corollary~\ref{cor:roughthm} imply that
MetaGrad's regret is simultaneously bounded by
\begin{equation}\label{eqn:gdBigORegret}
  R_T^\u = O(\sqrt{T \log \log T})
\end{equation}
and by
\begin{equation}\label{eqn:roughmainbound}
R_T^\u
~\le~
\sum_{t=1}^T (\w_t - \u)^\top \grad_t
~=~
O\del*{
\sqrt{
  V_T^\u\,
  d \ln(T/d)
}
+ d \ln(T/d)
}
\end{equation}
for any $\u \in \U$, where
\[
  V_T^\u := \sum_{t=1}^T \del*{(\u - \w_t)^\top \grad_t}^2.
\]
The inequality $R_T^\u \leq \Rtrick_T^\u := \sum_{t=1}^T (\w_t -
\u)^\top \grad_t$ is a direct consequence of convexity of the loss and
holds for any learning algorithm, so the important part of
\eqref{eqn:roughmainbound} is that it bounds $\Rtrick_T^\u$ in terms of
a measure of variance $V_T^\u$ that depends on the distance of the
algorithm's choices $\w_t$ to the optimum $\u$, and which, in favorable
cases, may be significantly smaller than $T$. Intuitively, this happens,
for instance, when there is a stable optimum $\u$ that the algorithm's
choices $\w_t$ converge to. Formal consequences are given in
Section~\ref{sec:fastRateExamples}, which shows that this bound implies
faster than $O(\sqrt{T})$ regret rates, often logarithmic in $T$, for
all functions in the rich class mentioned above. In all cases the
dependence on $T$ in the rates matches what we would expect based on
related work in the literature, and in most cases the dependence on the
dimension $d$ is also what we would expect. Only for strongly convex
functions is there an extra factor $d$. It seems that this is a real
limitation of the method as presented here. In
Section~\ref{sec:conclusion} we discuss a recent extension of MetaGrad
by \citet{WangLuZhang2020} that removes this limitation.

The main difficulty in achieving the regret guarantee in
\eqref{eqn:roughmainbound} is tuning a learning rate parameter $\eta$.
In theory, $\eta$ should be roughly proportional to $1/\sqrt{V_T^\u}$,
but this is not possible using any existing techniques, because the
optimum $\u$ is unknown in advance, and tuning in terms of a uniform
upper bound $\max_\u V_T^\u$ ruins all desired benefits. MetaGrad
therefore runs multiple supporting expert algorithms, one for each
candidate learning rate $\eta$, and combines them with a novel
controller algorithm that learns the empirically best learning rate for
the OCO task in hand. Crucially, the additive regret overhead for
learning the best expert is not of the usual order $O(\sqrt{T})$, which
would dwarf all desired benefits, but only costs a negligible $O(\log
\log T)$.

The experts are instances of exponential weights on the continuous
parameters $\w$ with a quadratic surrogate loss function, which in
particular causes the exponential weights distributions to be
multivariate Gaussian. The resulting updates are closely related to the
ONS algorithm on the original losses, with the twist that here each
expert receives the controller's gradients instead of its own, so only a
single gradient (for the controller) needs to be calculated per round.
We start and stop experts on the fly using a dynamic grid of
exponentially spaced $\eta$-values, which guarantees that at most
$\ceil{\origlog_2 T}$ experts are active at any given time. Since
$\ceil{\origlog_2 T} \leq 30$ as long as $T \leq 10^9$, this seems
computationally acceptable. If not, then the number of experts can be
further reduced at the cost of slightly worse constants in the bound by
spacing the $\eta$ in the grid further apart.

The version of MetaGrad described so far maintains a full covariance
matrix $\F_T = \sum_{t=1}^T \grad_t \grad_t^\top$ of size \mbox{$d \times d$}, where $d$ is the parameter
dimension. This requires $O(d^2)$ computation steps per round
to update, which is prohibitive for large $d$. We therefore also present
two extensions that require less computation: one based on sketching and
one that works coordinatewise. The sketching extension applies the
matrix sketching approach of \citet{luo2017arxiv} to approximate $\F_T$
by a sketch of its top $m-1$ eigenvectors, and requires $O(md)$
amortised update time per round. As shown in
Theorem~\ref{thm:mainsketchbound}, the price we pay for the improved
run-time is that \eqref{eqn:roughmainbound} is replaced by
\[
\Rtrick_T^\u
~=~
O\del*{
\sqrt{
  \del*{V_T^\u + \Omega_{m-1}}
  \,
  d \ln(T/d)
}
+ d \ln(T/d)
},
\]
which includes an extra term $\Omega_{m-1} = \sum_{i=m}^d \lambda_i$ to
account for the remaining eigenvalues in $\F_T$ that are not captured by
the sketch. Thus the hyperparameter $m$ provides a trade-off between
regret and run-time.

Our second extension was inspired by the diagonal version of AdaGrad
\citep{adagrad,McMahanStreeter2010} and runs a separate copy of full
MetaGrad per coordinate, which takes $O(d)$ computation per round, just
like vanilla OGD and AdaGrad. To avoid interactions between coordinates,
we restrict attention to rectangular domains. Whether this restriction
can be lifted is not clear, as discussed in
Section~\ref{sec:metagradcoordanalysis}. The main regret bound for the
coordinatewise extension is obtained by summing the regret bound
\eqref{eqn:roughmainbound} over the coordinates:
\begin{equation}\label{eqn:coordinateintrobound}
\Rtrick_T^\u
~=~
O\del*{
\sum_{i=1}^d \sqrt{
  V_{T,i}^{u_i}\,
  \ln(T)
}
+ d \ln(T)
},
\end{equation}
where $V_{T, i}^{u_i} = \sum_{t=1}^T (u_i - w_{t, i})^2 g_{t, i}^2$ is
the coordinatewise variance. This is established by
Theorem~\ref{th:diag} and Corollary~\ref{cor:coordroughthm}, which also
show that the coordinate extension simultaneously guarantees regret of
order
\begin{equation}\label{eqn:coordintroadagradbound}
\Rtrick_T^\u
~=~
O\del*{
  \sum_{i=1}^d \|g_{1:T,i}\|_2 \sqrt{\log \log T} 
  + \sqrt{d} \log \log T
}
~=~
O\del*{
  \sqrt{d T \log \log T}
},
\end{equation}
where $g_{1:T,i} := (g_{i,1},\ldots,g_{i,T})$. While the full matrix
version of MetaGrad and its sketching approximation naturally favor
parameters $\u$ with small $\ell_2$-norm, the coordinatewise extension
is appropriate for the $\ell_\infty$-norm (i.e., dense parameter
vectors). Since the coordinate version does not keep track of a full
covariance matrix, we cannot expect it to exploit the Bernstein
condition for stochastic gradients in all cases.
Section~\ref{sec:coordBernstein} therefore introduces a more stringent
coordinate Bernstein condition, under which
\eqref{eqn:coordinateintrobound} does always imply fast rates, and
Theorem~\ref{thm:generaltocoordbernstein} gives sufficient conditions
under which the general Bernstein condition implies the coordinate
Bernstein condition. It is appealing that the coordinatewise MetaGrad
extension simultaneously satisfies \eqref{eqn:coordintroadagradbound},
because (up to the $\sqrt{\log \log T}$ factor) this recovers the
diagonal AdaGrad bound of $O(\sum_{i=1}^d \|g_{1:T,i}\|_2)$, which can
take advantage of sparse gradients \citep{adagrad}.

An important practical consideration for OCO algorithms is whether they
can adapt to the Lipschitz-constant of the losses $f_t$, i.e.\ the
maximum norm of the gradients. For instance, this is an important
feature of AdaGrad \citep{adagrad,McMahanStreeter2010}. The MetaGrad
algorithm is also Lipschitz-adaptive in this way. Our approach is a refinement of
the techniques of \citet{lipschitz.metagrad}: whereas their procedure
may occasionally restart the whole MetaGrad algorithm, we only restart
the controller but not the experts. Wherever possible, we further
measure the size of the gradients by the \mbox{(semi-)norm} $\max_{\w
\in \U} |(\w - \w_t)^\top \grad_t|$ instead of the larger $\max_{\w \in
\U} \|\w - \w_t\|_2 \|\grad_t\|_2$. The difference is crucial in
Section~\ref{sec:sketching}, where we consider a time-varying domain
introduced by \citet{luo2017arxiv} in the context of sketching: this
domain is bounded only in the direction of the gradients, so our norms
are under control, but has infinite diameter in all orthogonal
directions.

We conclude the paper with an empirical evaluation in which we compare
our new algorithms (the Full, Sketching and Coordinatewise versions of
MetaGrad) with AdaGrad and OGD on 17 real-world LIBSVM regression and
classification data sets. Our experiments show that the full-matrix
version of MetaGrad beats previous methods in all but one of our
experiments and delivers competitive performance throughout. Moreover,
we see that the sketching extension provides a controlled trade-off
between regret and run-time, while the fastest, coordinatewise version
of MetaGrad still works surprisingly well in the majority of
experiments.

\subsection{Related Work}

If we disregard computational efficiency and omit Lipschitz-adaptivity,
then the guarantee from \eqref{eqn:roughmainbound} can be achieved by finely
discretizing the domain $\U$ and running the Squint algorithm for
prediction with experts, with each discretization point as an expert
\citep{squint}. MetaGrad may therefore also be seen as a computationally
efficient extension of Squint to the OCO setting.

\citet{luo2017arxiv} show a lower bound on the regret of $\Theta(\sqrt{d
T})$ for the time-varying domain mentioned above, and obtain a nearly
matching upper bound of $O(\sqrt{dT} \log T)$ using a variant of ONS.
Our Theorem~\ref{thm:mainbound}, which implies
\eqref{eqn:roughmainbound} when the radius of the domain is bounded, is
actually more general and also covers the time-varying domain. For this
domain it improves on the upper bound of \citeauthor{luo2017arxiv} by
replacing the dependence on $T$ by $V_T^\u$ and by moving the log-factor
into the square root. Section~\ref{sec:composition} provides a detailed
comparison.

As already mentioned, \citet{WangLuZhang2020} extend MetaGrad to adapt
to strongly convex functions. \citet{zhang19:_dual_adapt} further
provide an extension for the case that the optimal parameters $\u$ vary
over time, as measured in terms of the adaptive regret. See also the
closely related extension of Squint for adaptive regret by
\citet{Neuteboom2020}.

Our focus in this work is on adapting to sequences of functions $f_t$
that are easier than general convex functions, but we require an
estimate $\sigma$ of the $\ell_2$-norm of the optimum $\u$ as a
hyperparameter. In contrast, a different line of work designs methods
that can adapt to the norm of $\u$ over all of $\reals^d$, but without
providing adaptivity to the functions $f_t$
\citep{McMahanStreeter2012,Orabona2014,CutkoskyOrabona2018}. It was
thought for some time that these two directions could not be reconciled,
because the impossibility result of \citet{DBLP:conf/colt/CutkoskyB17}
blocks simultaneous adaptivity to both the size of the gradients of the
functions $f_t$ and the norm of $\u$. The perspective has recently
shifted, however, following discoveries of ways to partially circumvent
this lower bound
\citep{KempkaKW19,cutkosky2019artificial,lipcomp.adaptive}.

Another notion of adaptivity is explored in a series of works obtaining
tighter bounds for linear functions $f_t$ that vary little between
rounds, as measured either by their deviation from the mean function or
by successive differences
\citep{hazan2010extracting,GradualVariationInCosts2012,SteinhardtLiang14}.
Such bounds imply super fast rates for optimizing a fixed linear
function, but reduce to slow $O(\sqrt{T})$ rates in the other cases of
easy functions that we consider. Finally, the way MetaGrad's experts
maintain a Gaussian distribution on parameters $\u$ is similar in spirit
to AROW and related confidence weighted methods, as analyzed by
\citep{CrammerEtAl2009AROW} in the mistake bound model.

\subsection{Outline}

We start with the main definitions in the next section. Then
Section~\ref{sec:fastRateExamples} contains an extensive set of examples
where the guarantee from \eqref{eqn:roughmainbound} leads to fast rates,
Section~\ref{sec:metagradfull} presents the Full Matrix version of the
MetaGrad algorithm, and Section~\ref{sec:metagradextensions} describes
the faster sketching and coordinatewise extensions.
Section~\ref{sec:analysisfull} provides the analysis leading to
Theorem~\ref{thm:mainbound} for the Full Matrix version of MetaGrad,
which is a more detailed statement of \eqref{eqn:roughmainbound} with
several quantities replaced by data-dependent versions and with exact
constants. Section~\ref{sec:analysisextensions} extends this analysis to
the two other versions of MetaGrad. Then, in
Section~\ref{sec:experiments}, we compare all versions of MetaGrad to OGD
and to AdaGrad in experiments with several benchmark classification and
regression data sets. We conclude with a discussion of possible further
extensions of MetaGrad in Section~\ref{sec:conclusion}. Finally, most
details of the proofs are postponed to the appendix.

\section{Setup}\label{sec:setup}

We consider algorithms for OCO, which operate according to the protocol
displayed in Protocol~\ref{alg:OCOprotocol}. In each round, the
environment reveals a closed convex domain $\U_t \subset \reals^d$,
which we assume contains the origin $\0$ (if not, it needs to be
translated). In the introduction, we assumed that $\U_t = \U$ was fixed
beforehand, but for the remainder of the paper we allow it to vary
between rounds, which is needed in the context of the sketching version
of MetaGrad (Section~\ref{sec:sketching}). Let $\w_t \in \U_t$ be the
iterate produced by the algorithm in round $t$, let $f_t : \U_t \to
\reals$ be the convex loss function produced by the environment and let
$\grad_t = \nabla f_t(\w_t)$ be the (sub)gradient, which is the feedback
given to the algorithm.\footnote{If $f_t$ is not differentiable at
$\w_t$, any choice of subgradient $\grad_t \in \partial f_t(\w_t)$ is
allowed. Since $f_t$ is convex, there always exists at least one
subgradient when $\w_t$ is in the interior of its domain. Existence of
subgradients on the boundary of $\U_t$ is guaranteed, for instance, if
there exists a finite convex extension of $f_t$ to~$\reals^d$.} The
\emph{regret} over $T$ rounds $R_T^\u$, its linearization $\Rtrick_T^\u$
and our measure of variance $V_T^\u$ are defined as
\begin{align*}
  R_T^\u &= \sum_{t=1}^T \del*{f_t(\w_t) - f_t(\u)},
  & \Rtrick_T^\u &= \sum_{t=1}^T(\w_t - \u)^\top \grad_t,\\
  V_T^\u &= \sum_{t=1}^T \del*{(\u - \w_t)^\top \grad_t}^2
  && \text{with respect to any $\u \in \Intersect_{t=1}^T \U_t$.}
\end{align*}

\begin{algorithm2e}[t]
\renewcommand{\algorithmicrequire}{\textbf{Input:}}
\renewcommand{\algorithmiccomment}[1]{\hfill $\triangleright$~\textit{#1}}
\begin{algorithmic}[1]
\FOR{$t=1,2,\ldots$}
  \STATE Environment reveals convex domain $\U_t \subseteq \reals^d$
  containing the origin $\0$
  \STATE Learner plays $\w_t \in \U_t$
  \STATE Environment chooses a convex loss function $f_t : \U_t \to \reals$
  \STATE Learner incurs loss $f_t(\w_t)$ and observes (sub)gradient $\grad_t = \nabla f_t(\w_t)$
\ENDFOR
\end{algorithmic}
\SetAlgorithmName{Protocol}{Protocol}{List of Protocols}
\caption{Online Convex Optimization with First-order Information}\label{alg:OCOprotocol}
\end{algorithm2e}
\setcounter{algocf}{0}  %

\noindent By convexity of $f_t$, we always have $f_t(\w_t) - f_t(\u)
\leq (\w_t - \u)^\top \grad_t$, which implies the first inequality in
\eqref{eqn:roughmainbound}: $R_T^\u \leq \Rtrick_T^\u$. Finally,
wherever possible we measure the size of the gradient $\grad_t$ in the
following (semi-)norm:
\[
  \|\grad\|_t = \max_{\w \in \U_t} |(\w - \w_t)^\top \grad|,
\]
which takes into account the shape of the domain, and is centered at the
learner's predictions $\w_t$.
This is a norm in the typical case that $\U_t$ has full dimension $d$,
and it is still a semi-norm in general. We note that this norm
is smaller than the usual upper bounds based on H\"older's inequality:
$\|\grad\|_t \leq \|\grad\|_* \max_{\w \in \U_t} \|\w - \w_t\|$ for any
dual norms $\|\cdot\|$ and $\|\cdot\|_*$. The difference becomes
essential in Section~\ref{sec:sketching}, where we consider a domain
$\U_t$ that has an infinite radius $\max_{\w \in \U_t} \|\w - \w_t\|$
in any norm $\|\cdot\|$, but for which $\|\grad_t\|_t$ is still
bounded. MetaGrad depends on (upper bounds on) the sizes of the
gradients per round $b_t$, as well as their running maximum $B_t$:
\begin{align}\label{eqn:rangebound}
  b_t &\geq \|\grad_t\|_t,
  &
  B_t &= \max_{s \leq t} b_s,
\end{align}
with the convention that $B_0 = 0$. We would normally take the best
upper bound $b_t = \|\grad_t\|_t$, except if this is difficult to
compute. In such cases, we may, for example, let $b_t = \|\grad_t\|_2
\max_{\u,\w \in \U_t} \|\u - \w\|_2$. We assume throughout that $B_T >
0$; otherwise the regret is trivially bounded by zero.

\smallskip\noindent
We denote by $\ceil{z}_+ = \max\{\ceil{z},1\}$ the smallest integer that
is at least~$z$ and at least $1$.

\section{Fast Rates Examples}\label{sec:fastRateExamples}

In this section, we motivate our interest in the adaptive bound
\eqref{eqn:roughmainbound} by giving a series of examples in which it
provides fast rates. For simplicity, we will in this section assume that
the domain is fixed: $\U_t = \U$, with bounded radius $D_2 \geq \max_{\u \in
\U} \|\u\|_2$, and that all gradients have length at most
$G_2 \geq \|\g_t\|_2$. The fast rates are all derived from two general
sufficient conditions: one based on the directional derivative of the
functions $f_t$ and one for stochastic gradients that satisfy the
\emph{Bernstein condition}, which is the standard condition for fast
rates in off-line statistical learning. In Appendix~\ref{app:simulations} we provide simple simulations illustrating these conditions, which are exploited by MetaGrad but not by AdaGrad. Proofs are also postponed to
Appendix~\ref{app:MoreFastRateExamplesAndProofs}.

\subsection{Directional Derivative Condition}

In order to control the regret with respect to some point $\u$, the
first condition requires a quadratic lower bound on the curvature of the
functions $f_t$ in the direction of $\u$:
\begin{theorem}\label{thm:curvedfunctions}
Suppose, for a given $\u \in \U$, there exist constants $a,b > 0$ such
that the functions $f_t$ all satisfy
\begin{equation}\label{eqn:curvedfunctions}
  f_t(\u) \geq f_t(\w) + a (\u - \w)^\top \nabla f_t(\w) + b \del*{(\u - \w)^\top \nabla f_t(\w)}^2
  \qquad \text{for all $\w \in \U$.}
\end{equation}
Then any method with regret bound \eqref{eqn:roughmainbound} incurs
logarithmic regret, $R_T^\u = O(d \ln T)$, with respect to $\u$.
\end{theorem}

The case $a = 1$ of this condition was introduced by \citep{ons}, who
show that it is satisfied for all $\u \in \U$ by exp-concave and
strongly convex functions. These are both requirements on the curvature
of $f_t$ that are stronger than mere convexity: $\alpha$-exp-concavity
of $f$ for $\alpha > 0$ means that $e^{-\alpha f}$ is concave or,
equivalently, that $\nabla^2 f \succeq \alpha \nabla f \nabla f^\top$;
$\alpha$-strong convexity means that $\nabla^2 f \succeq \alpha \I$.
We see that $\alpha$-strong convexity implies $(\alpha/\|\nabla
f\|_2^2)$-exp-concavity. The rate $O(d \log T)$ is also what we would
expect by summing the asymptotic offline rate obtained by ridge
regression on the squared loss \citep[Section~5.2]{SrebroEtAl2010},
which is exp-concave. Our extension to general $a > 0$ is technically a minor
step, but it makes the condition much more liberal, because it may then
also be satisfied by functions that do \emph{not} have any curvature.
For example, suppose that $f_t = f$ is a fixed convex function that does
not change with $t$. Then, when $\u^* = \argmin_\u f(\u)$ is the offline
minimizer, we have $(\u^* - \w)^\top \nabla f(\w) \in \intcc{- 2D_2
G_2,0}$, so that
\begin{equation*}
  f(\u^*) - f(\w)
    \geq (\u^* - \w)^\top \nabla f(\w)
    \geq 2 (\u^* - \w)^\top \nabla f(\w) + \frac{1}{2 D_2 G_2} \del*{(\u^* - \w)^\top
    \nabla f(\w)}^2,
\end{equation*}
where the first inequality uses only convexity of $f$. Thus condition
\eqref{eqn:curvedfunctions} is satisfied by \emph{any fixed convex
function}, even if it does not have any curvature at all, with $a =
2$ and $b=1/(2 D_2 G_2)$.

At first sight this may appear to contradict the lower bound of order
$1/\sqrt{T}$ for convergence of the iterates by \cite{Nesterov2004} \citep[see also][]{tibshirani14}, which implies a lower bound of order
$\sqrt{T}$ on the regret. Yet there is no contradiction, as Nesterov's
example requires large dimension $d \ge T$, in which case $O(d \ln T)$
is vacuous. In Nesterov's example, MetaGrad still gets the $\sqrt{T}$
rate up to a $\ln \ln T$ factor, however, because it satisfies
\eqref{eqn:gdBigORegret}.

\subsection{Bernstein Stochastic Gradients}

The possibility of getting fast rates even without any curvature is
intriguing, because it goes beyond the usual strong convexity or
exp-concavity conditions. In the online setting, the case of fixed
functions $f_t = f$ seems rather restricted, however, and may in fact be
handled by offline optimization methods. We therefore seek to loosen
this requirement by replacing it by a stochastic condition on the
distribution of the functions $f_t$. The relation between variance
bounds like \eqref{eqn:roughmainbound} and fast rates in the stochastic
setting is studied in depth by \citep{bernfast}, who obtain
fast rate results both in expectation and in probability. Here we
provide a direct proof only for the expected regret, which allows a
simplified analysis.

Suppose the functions $f_t$ are independent and identically
distributed (i.i.d.), with common distribution $\pr$. Then we say that
the gradients satisfy the \emph{$(B,\beta)$-Bernstein condition} with
respect to the stochastic optimum $\u^* = \argmin_{\u \in \U} \E_{f \sim
\pr}[f(\u)]$ if
\begin{equation}\label{eqn:bernstein}
(\w - \u^*)^\top
\ex_f \sbr*{
  \nabla f(\w) \nabla f(\w)^\top
}
(\w - \u^*)
~\le~
B
\big((\w - \u^*)^\top  \ex_f \sbr*{\nabla f(\w)}\big)^\beta
\qquad \text{for all $\w \in \U$.}
\end{equation}
This is an instance of the well-known Bernstein condition from offline
statistical learning
\citep{BartlettMendelson2006,VanErven2015FastRates}, applied to the
linearized excess loss $(\w - \u^*)^\top \nabla f(\w)$.
As shown in Appendix~\ref{sec:bnst}, imposing the condition for the
linearized excess loss is a weaker requirement than imposing it for the
original excess loss $f(\w) - f(\u^*)$.
\begin{theorem}\label{thm:Bernstein}
If the gradients satisfy the $(B,\beta)$-Bernstein condition for $B > 0$
and $\beta \in (0,1]$ with respect to $\u^* = \argmin_{\u \in \U} \E_{f
\sim \pr}[f(\u)]$, then any method with regret bound
\eqref{eqn:roughmainbound} incurs expected regret
\[
  \E\sbr{R_T^{\u^*}} = O\del*{\del*{B d \ln T}^{1/(2-\beta)}
  T^{(1-\beta)/(2-\beta)} + d\ln T}.
\]
\end{theorem}
For $\beta=1$, the rate becomes $O(d
\ln T)$, just like for fixed functions, and for smaller $\beta$ it is in
between logarithmic and $O(\sqrt{d T})$.
For instance, the hinge loss on the unit ball with i.i.d.\ data
satisfies the Bernstein condition with $\beta = 1$, which implies an
$O(d \log T)$ rate, albeit with a $B$ that depends on the distribution
of the data. (See Appendix~\ref{app:hingeLossExample}.) In stochastic
optimization for support vector machines, the hinge loss is combined
with an additional $\ell_2$-regularization term. It is sometimes argued
that this term also gives fast rates, because it makes the loss strongly
convex, but the amount of regularization used in practice is typically
too small to get any significant improvements. The present example shows
that, even without adding regularization to the loss, it is possible to
get logarithmic regret.

\section{Full Matrix Version of the MetaGrad Algorithm}
\label{sec:metagradfull}

In this section, we explain the full matrix version of the MetaGrad algorithm:
\metagradfull{}. Computationally more efficient extensions follow
in Section~\ref{sec:metagradextensions}.
\metagradfull{} will be defined by means of the following
\emph{surrogate loss} $\surr_t^\eta(\u)$:
\begin{equation}\label{eq:surrogate}
\surr_t^\eta(\u)
~\df~
\eta(\u - \w_t)^\top \grad_t
+ \big(\eta (\u - \w_t)^\top \grad_t\big)^2
.
\end{equation}
This surrogate loss consists of a linear and a quadratic part, whose
relative importance is controlled by a learning rate parameter $\eta >
0$. The sum of the quadratic parts is what appears in the regret bound
\eqref{eqn:roughmainbound}. They may be viewed as causing a
``time-varying regularizer'' \citep{Orabona2015} or ``temporal
adaptation of the proximal function'' \citep{adagrad}. 

\metagradfull{} is a two-level hierarchical construction: at the top is
a main controller, shown in Algorithm~\ref{alg:controller}, which
manages multiple $\eta$-experts, shown in
Algorithm~\ref{alg:MetaGradExpert}. Each $\eta$-expert produces
predictions for the surrogate loss $\surr_t^\eta$ with its own value of
$\eta$, and the controller is responsible for learning the best $\eta$
by starting and stopping multiple $\eta$-experts on demand, and
aggregating their predictions.

\begin{algorithm2e}[htb]
\setlength{\commentWidth}{7cm}
\renewcommand{\algorithmiccomment}[1]{\spacycomment{#1}}
\renewcommand{\algorithmicrequire}{\textbf{Input:}}
  \caption{MetaGrad Full: Controller}
  \label{alg:controller}
  \begin{algorithmic}[1]
    \FOR{$t=1,2,\ldots$}
    \STATE Receive domain $\U_t$
    \STATE \multiline{Start and stop $\eta$-experts to manage active set $\activeset_t$
    (Equation~\ref{eqn:activeexperts}). Give newly started  $\eta$-experts weight $p_t(\eta)=1$. }
    \IF{Nobody active: $\activeset_t = \emptyset$}
    \STATE Predict $\w_t = \zeros$ 
    \COMMENT {Make a default prediction ~~~~~~~}
    \ELSE
    \STATE Have active $\eta$-experts project onto $\U_t$
    \STATE
    Collect prediction $\w_t^\eta$ for every active $\eta$-expert
    \STATE \label{line:tilted_ewa}
    Predict \[
      \w_t ~=~ \frac{\sum_{\eta \in \activeset_t} p_t(\eta) \eta \w_t^\eta}{\sum_{\eta \in \activeset_t} p_t(\eta) \eta}
    \]
    \ENDIF
    \STATE Receive gradient $\grad_t = \nabla f_t(\w_t)$ and range bound
    $b_t$ (Equation~\ref{eqn:rangebound})
    \STATE Update every active $\eta$-expert with unclipped surrogate
    loss $\surr_t^\eta$
    \IF{No reset needed after round $t$ (Equation~\ref{eqn:resetcondition})}
    \STATE \label{line:expw}
    Update based on
    the clipped surrogate losses (Equation~\ref{eq:clipsurr}):\linebreak
\mbox{}\hspace{2em}
    $p_{t+1}(\eta) = \frac{p_t(\eta) \exp\del*{-
    \clipsurr_t^\eta(\w_t^\eta)}}{\sum_{\eta
    \in \activeset_t} p_t(\eta) \exp\del*{-\clipsurr_t^\eta(\w_t^\eta)}} \del{\sum_{\eta \in \activeset_t} p_t(\eta)}
    \qquad \text{for all $\eta \in \activeset_t$.}$
    \ELSE
    \STATE Set $p_{t+1}(\eta) = 1$ for all $\eta \in \activeset_t$
    \COMMENT{Reset ~~~~~~~}
    \ENDIF
  \ENDFOR
\end{algorithmic}
\end{algorithm2e}

\subsection{Controller}

Online learning of the best learning rate $\eta$ is notoriously
difficult because the regret is non-monotonic over rounds and may have
multiple local minima as a function of~$\eta$ (see
\citep{learning.learning.rate} for a study in the expert setting). The
standard technique is therefore to derive a monotonic upper bound on the
regret and tune the learning rate optimally \emph{for the bound}. In
contrast, our approach, inspired by the approach for combinatorial games
of \citet[Section~4]{squint}, is to weigh the different $\eta$ depending
on their empirical performance using exponential weights with sleeping
experts (line~\ref{line:expw}), except that in the predictions the
weights of the $\eta$-experts are \emph{tilted} by their learning rates
(line~\ref{line:tilted_ewa}), having the effect of giving a larger
weight to larger~$\eta$. Thus we never tune the controller's weights on
learning rates based on any bounds, but always directly in terms of
their empirical performance.

To be able to adapt to the norms of the gradients, the controller
maintains a finite grid $\activeset_t$ of active learning rates $\eta$,
which is dynamically adjusted over time. We will take exponentially
spaced learning rates from the infinite grid
\[
  \mathcal G ~\df~ \setc{2^i}{i \in \mathbb Z},
\]
and the following learning rates are active in round $t$:
\begin{equation}\label{eqn:activeexperts}
  \activeset_t ~\df~ 
  \begin{cases}
  \emptyset & \text{while $B_{t-1} =0$,}\\
  \mathcal G \cap
  \intoc*{
    \frac{1}{2 \del*{\sum_{s=1}^{t-1} b_s \frac{B_{s-1}}{B_s} + B_{t-1}}},
    \frac{1}{2 B_{t-1}}
  } & \text{afterwards.}
  \end{cases}
\end{equation}
This means that $a^\eta$, the first round in which an $\eta$-expert is
active, is
\begin{equation}\label{eqn:wakeuptime}
  a^\eta = \min \setc*{t \in \{1,2,\ldots\}}{
      \eta > \frac{1}{2 \del*{\sum_{s=1}^{t-1} b_s \frac{B_{s-1}}{B_s} + B_{t-1}}}
    }.
\end{equation}
Using that $b_s \frac{B_{s-1}}{B_s} \leq B_{t-1}$, it can be seen that
the number of active learning rates never exceeds $|\activeset_t| \leq
\ceil{\origlog_2 T}$. In the first two rounds, or if there
is a sudden enormous gradient such that $B_{t-1}$ dwarfs
$\sum_{s=1}^{t-1} b_s B_{s-1}/B_s$, it may also happen that
$\activeset_t$ is empty, which signals that all previous rounds were
negligible compared to the last round. In such cases the controller
decides it has not yet learned anything, and makes a default prediction:
$\w_t = \zeros$.

There are two further mechanisms to deal with extreme changes in the
size of the gradients. The first mechanism is that extremely large
gradients may trigger a \emph{reset} of the controller's weights on $\eta$-experts. This
splits the controller's learning process into epochs. When running in an
epoch starting at time $\tau+1$, a reset and new epoch will be triggered
after the first round $t$ such that
\begin{equation}\label{eqn:resetcondition}
  B_t > B_\tau \sum_{s=1}^t \frac{b_s}{B_s}.
\end{equation}
As the sum on the right-hand side will typically grow linearly in $t$,
we only expect a reset to occur when the effective size of the gradients
grows by more than a factor~$t$ compared to the largest size seen before
the start of the epoch. This should normally be very rare except perhaps
for a few initial rounds when $t$ is still small.

The second mechanism to protect against extreme gradients is that the
controller measures performance of the experts by a \emph{clipped}
version of their corresponding surrogate losses: 
\begin{equation}\label{eq:clipsurr}
  \clipsurr_t^\eta(\u) ~\df~
    \eta (\u-\w_t)^\top \clipgrad_t
    + \big(\eta (\u-\w_t)^\top \clipgrad_t\big)^2,
\end{equation}
which are based on the clipped gradients
\[
  \clipgrad_t ~\df~ \frac{B_{t-1}}{B_t} \grad_t.
\]
This is a trick first used by \citet{cutkosky2019artificial}, which
makes the effective sizes of the gradients predictable one round in
advance: $\max_{\u \in \U_t} |\u^\top \clipgrad_t| \leq B_{t-1}$.

\begin{algorithm2e}[htbp]
\setlength{\commentWidth}{7.5cm}
\renewcommand{\algorithmiccomment}[1]{\spacycomment{#1}}
\renewcommand{\algorithmicrequire}{\textbf{Input:}}
\begin{algorithmic}[1]
\REQUIRE \multiline{Learning rate $\eta > 0$, estimate $\sigma > 0$ of comparator norm
$\|\u\|_2$, first active round~$a \equiv a^\eta$}
\STATE Initialize $\check{\w}^\eta_a = \zeros$ and $\bLambda^\eta_a = \tfrac{1}{\sigma^2} \I$
\COMMENT{Invariant: $\bLambda^\eta_{t+1} = \frac{1}{\sigma^2} \I + 2 \eta^2
\sum_{s=a}^t \grad_s \grad_s^\top$}
\STATE Initialize $\Sigma^\eta_a = \sigma^2 \I$
\COMMENT{Invariant: $\Sigma_t^\eta = (\bLambda_t^\eta)^{-1}$}
\FOR{$t=a,a+1,\ldots$}
\STATE Project $\w_t^\eta = \argmin_{\u \in \U_t} (\u - \check{\w}^\eta_t)^\top
\bLambda_t^\eta (\u - \check{\w}^\eta_t)$
\STATE Predict $\w_t^\eta$
\STATE Observe gradient $\grad_t = \nabla f_t(\w_t)$ \COMMENT{Gradient
at \emph{controller} prediction $\w_t$}
\STATE
Update:\linebreak
\mbox{}\hspace{1em}
$
\begin{aligned}[t]
\Sigma^\eta_{t+1} &=
  \Sigma^\eta_t - 
  \frac{2 \eta^2(\Sigma_t^\eta \grad_t) (\grad_t^\top \Sigma_t^\eta)}{1 + 2
    \eta^2 \grad_t^\top \Sigma_t^\eta \grad_t}
  \qquad \simplecomment{Sherman-Morrison}
\\
\bLambda^\eta_{t+1} &= \bLambda^\eta_t + 2 \eta^2 \grad_t \grad_t^\top
\\
\check{\w}_{t+1}^\eta
  &=  \w^\eta_t
    - \del*{1 + 2 \eta (\w_t^\eta - \w_t)^\top \grad_t} \eta
      \Sigma^\eta_{t+1} \grad_t \label{line:md}
\end{aligned}$
\ENDFOR
\end{algorithmic}
\caption{MetaGrad Full: $\eta$-Expert}\label{alg:MetaGradExpert}
\end{algorithm2e}

\subsection{$\eta$-Experts}

Each $\eta$-expert is active for a single contiguous sequence of rounds
for which $\eta \in \activeset_t$. Upon activation, its job is to issue
predictions $\w_t^\eta \in \U_t$ for the (unclipped) surrogate loss
$\surr_t^\eta$ that achieve small regret compared to any $\u \in
\Intersect_{t : \eta \in \activeset_t} \U_t$. This is a standard online
convex optimization task with a quadratic loss function and time-varying
domain. We use continuous exponential weights
with a Gaussian prior, which is a standard approach for quadratic losses
\citep{Vovk2001}, because the corresponding posterior exponential
weights distribution is also Gaussian with mean $\w_t^\eta$ and
covariance matrix $\Sigma_t^\eta = \del*{\frac{1}{\sigma^2} \I + 2 \eta^2
\sum_{s=a}^t \grad_s \grad_s^\top}^{-1}$.
Algorithm~\ref{alg:MetaGradExpert} presents the update equations in a
computationally efficient form. To avoid inverting $\Sigma_t^\eta$, it
maintains its inverse $\bLambda_t^\eta = (\Sigma_t^\eta)^{-1}$ separately.
For a recent overview of continuous
exponential weights see \citep{hoeven2018many}. It can be seen that our
$\eta$-expert algorithm is nearly identical to Online Newton Step (ONS)
\citep{ons}, which is not surprising because ONS is minimizing a
quadratic loss that is nearly identical to our $\surr_t^\eta$. The
differences are that each $\eta$-expert receives the controller's
gradient $\grad_t = \nabla f_t(\w_t)$ instead of its own $\nabla
f_t(\w_t^\eta)$, and that an additional factor $\del*{1 + 2 \eta
(\w_t^\eta - \w_t)^\top \grad_t}$ in line~\ref{line:md} adjusts for the
difference between the $\eta$-expert's parameters $\w_t^\eta$ and the
controller's parameters $\w_t$. MetaGrad is therefore a bona fide
first-order algorithm that only accesses $f_t$ through $\grad_t$.  We
also note that we have chosen the Greedy projections version that
iteratively updates and projects. One might
alternatively consider the Lazy Projection version
\citep[as in][]{Zinkevich2004,Nesterov2009,Xiao2010} that forgets past
projections when updating on new data. Since projections are typically
computationally expensive, we have opted for the Greedy projection
version, which we expect to project less often, since a projected point
seems less likely to update to a point outside of the domain than an
unprojected point.

\subsection{Practical Considerations}
\label{sec:metagradfullpractical}

Although \metagradfull{} is adaptive to the maximum effective size of
the gradients $B_T$, its performance degrades when $B_T$ becomes too
large. In applications, it is therefore important that the domain $\U_t$
is small enough along the direction of $\grad_t$ to keep the effective
gradient size $b_t$ under control.

It is further required to choose the hyperparameter $\sigma$, which is
an estimate of the $\ell_2$-norm of the comparator $\u$.
Theorem~\ref{thm:mainbound} quantifies the trade-off between
underestimating and overestimating this parameter. As discussed below
the theorem, underestimating $\sigma$ always harms the rate. But, for
low-dimensional settings, overestimating $\sigma$ only incurs a
logarithmic penalty, so it is much less expensive to use a too large
value than to use a too small value. For high-dimensional settings the
dependence on $\sigma$ is similar to the usual dependence of Online
Gradient Descent on a guess for $\|\u\|_2$, so the rate deteriorates
linearly when taking $\sigma$ too large.

Finally, we note that there is no gain in pre-processing the data by
scaling all gradients by a fixed constant factor, since the regret bound
in Theorem~\ref{thm:mainbound} already scales linearly with the size of
the gradients. In fact, the \metagradfull{} algorithm itself is almost
invariant under such rescaling, except for the grid $\setc{2^i}{i \in
\mathbb Z}$ in the definition of $\activeset_t$. If one wants to make
the algorithm fully invariant under rescaling, the grid may be replaced
by $\setc{2^i/B_\tau}{i \in \mathbb Z}$, where $\tau$ is the first round
that $B_\tau > 0$. Or, equivalently, one may replace all gradients by
$\grad_t/B_\tau$ for $t \geq \tau$. Since we do not expect any
noticeable difference in performance from this modification, we have
left it out. 

\subsubsection{Run Time}

The run time of \metagradfull{} is dominated by computations for the
$\eta$-experts. Ignoring the projection step, an $\eta$-expert takes
$O(d^2)$ time to update. If there are at most $k$ active $\eta$-experts
in any round, this makes the overall computational effort $O(k d^2)$,
both in time per round and in memory. Since $|\activeset_t| \leq
\ceil{\origlog_2 T}$, it is guaranteed that $k \leq 30$ as long as $T
\leq 10^9$. We note that all $\eta$-experts share the same gradient
$\grad_t$, which is only computed once. We remark that a potential
speed-up is possible by running the $\eta$-experts in parallel. If the
factor $k$ is still considered too large, it is possible to reduce the
size of $|\activeset_t|$ by spacing the learning rates by a factor
larger than $2$, at the cost of a worse constant in the regret bound.

In addition, each $\eta$-expert may incur the cost of a projection,
which depends on the shape of the domain $\U_t$. To get a sense for the
projection cost, we consider the Euclidean ball as a typical example. If
the matrix $\Sigma_t^\eta$ were diagonal, we could project to any
desired precision using a few iterations of Newton's method. Since each
such iteration takes $O(d)$ time, this would be affordable. But for the
non-diagonal $\Sigma_t^\eta$ that occur in the algorithm, we first need
to reduce to the diagonal case by a basis transformation, which takes
$O(d^3)$ to compute using a singular value decomposition. We therefore
see that the projection dwarfs the other run time by an order of
magnitude. This has motivated \citet{luo2017arxiv} to define a different
domain (see Section~\ref{sec:sketching}), for which projections can be
computed in closed form with $O(d)$ computation steps. In this case, the
computation for the projections is negligible and the total
computational complexity is $O(d^2)$ per round. We refer to
\citet{adagrad} for examples of how to compute projections for various
other domains $\U_t$.

\section{Faster Extension Algorithms}
\label{sec:metagradextensions}

As discussed above, \metagradfull{} requires at least $O(d^2)$
computation per round, which makes it slow in high dimensions. We
therefore present two extensions to speed up the algorithm. The first is
a straightforward adaption of the sketching approach of
\citet{luo2017arxiv}, which we apply to approximate the matrix
$\Sigma_t^\eta$ used in each $\eta$-expert. This reduces the computation per
round to $O(kd)$, where $k$ is a hyperparameter that determines the
sketch size. The second extension is to run a separate copy of the
algorithm per dimension, which was inspired by the diagonal version of
AdaGrad \citep{adagrad}. This requires $O(d)$ computation per round.

\subsection{Sketched MetaGrad with Closed-form Projections}
\label{sec:sketching}

\begin{algorithm2e}[htbp]
\setlength{\commentWidth}{7cm}
\renewcommand{\algorithmiccomment}[1]{\spacycomment{#1}}
\renewcommand{\algorithmicrequire}{\textbf{Input:}}
\begin{algorithmic}[1]
\REQUIRE \multiline{Learning rate $\eta > 0$, estimate $\sigma > 0$ of comparator norm
$\|\u\|_2$, first active round~$a \equiv a^\eta$}
\STATE Initialize $\check{\w}^\eta_a = \zeros$
\STATE \multiline{Get $\S_{a-1}^\eta$ and $\H_{a-1}^\eta$ from
initialisation of Frequent Directions Sketching
\newline Algorithm~\ref{alg:fdsketch}}
\FOR{$t=a,a+1,\ldots$}
\STATE Observe feature vector $\x_t$
\STATE Obtain $\w_t^\eta$ by projection \eqref{eq:projection}
\STATE Issue prediction $\w_t^\eta$
\STATE Observe gradient $\grad_t = \nabla f_t(\w_t)$ \COMMENT{Gradient
  at \emph{controller} prediction $\w_t$ ~~~~~}
\STATE Send $\grad_t$ to Frequent Directions Sketching
Algorithm~\ref{alg:fdsketch} and receive $\S_t^\eta$ and $\H_t^\eta$ 
\STATE Update $\check{\w}^\eta_{t+1}$ as per \eqref{eq:sketched.update}
\ENDFOR
\end{algorithmic}
\caption{Sketched $\eta$-Expert}\label{alg:MetaGrad.Sketched.Expert}
\end{algorithm2e}

\noindent
In this section, we are mixing matrices of different dimensions. The
identity matrix $\I_d \in \reals^d$ and the all-zeros matrix $\0_{a
\times b} \in \reals^{a \times b}$ are therefore annotated with
subscripts to make their dimensions explicit.

\citet{luo2017arxiv} develop several sketching approaches for Online
Newton Step, which transfer directly to our $\eta$-experts. They combine
these with a computationally efficient choice of the domain that applies
to loss functions of the form $f_t(\w) = h_t(\w^\top \x_t)$, where the
input vectors $\x_t \in \reals^d$ are assumed to be known at the start
of round $t$, but the convex functions $h_t : \reals \to \reals$ become available only after the prediction has been made. They then choose the domain to be
\begin{equation}\label{eqn:luodomain}
  \U_t = \{\w : |\w^\top \x_t| \leq C\}
  \qquad \text{for a fixed constant $C$.}
\end{equation}
Let $a^\eta$ be the round in which the $\eta$-expert is first activated
and define $\G_t^\eta = (\grad_{a^\eta}, \hdots, \grad_t)^\top \in
\reals^{(t-a^\eta+1) \times d}$, such that $\Sigma_{t+1}^\eta =
(\frac{1}{\sigma^2} \I_d + 2 \eta^2 (\G_t^\eta)^\top \G_t^\eta)^{-1}$.
The idea of sketching is to replace $\Sigma_{t+1}^\eta \in \mathbb R^{d \times d}$ by an
approximation
\[
  \approxSigma_{t+1}^\eta
    = \del*{\tfrac{1}{\sigma^2} \I_d + 2 \eta^2 (\S^\eta_t)^\top \S^\eta_t}^{-1},
\]
where $\S^\eta_t \in \reals^{k \times d}$ for a given \emph{sketch size}
$k$ that can be much smaller than $d$, so that $(\S^\eta_t)^\top \S^\eta_t$ has rank at most $k$.
Abbreviating
\begin{equation}\label{eq:sketching.ghat}
\approxgrad_t^\eta = \del*{1 + 2 \eta (\w_t^\eta - \w_t)^\top
\grad_t} \eta \grad_t,
\end{equation}
we then need to compute
\begin{align*}
\w_t^\eta 
  ~&=~ \argmin_{\u \in \U_t} ~~ (\u - \check{\w}^\eta_t)^\top
(\approxSigma_t^\eta)^{-1} (\u - \check{\w}^\eta_t)
      \tag{projection}\\
\check{\w}_{t+1}^\eta
  ~&=~  \w^\eta_t - \approxSigma^\eta_{t+1} \approxgrad_t^\eta.
      \tag{update}
\end{align*}
The key to an efficient implementation of these steps is to rewrite
$\approxSigma_{t+1}^\eta$ using the Woodbury identity
\citep{golub2012matrix}:
\[
  \approxSigma_{t+1}^\eta
    = \sigma^2 (\I_d - 2 \eta^2 (\S^\eta_t)^\top (\tfrac{1}{\sigma^2} \I_k + 2\eta^2 \S^\eta_t (\S^\eta_t)^\top)^{-1}
    \S^\eta_t)
    = \sigma^2 (\I_d - 2 \eta^2 (\S^\eta_t)^\top \H_t^\eta \S^\eta_t),
\]
where we have introduced the abbreviation
\begin{equation}\label{eq:H}
  \H_t^\eta = (\tfrac{1}{\sigma^2}
  \I_k + 2\eta^2 \S^\eta_t (\S^\eta_t)^\top)^{-1}
  .
\end{equation}
Let $\shrink_C(y) = \sign(y)\max\{|y| - C, 0\}$. By Lemma~1
of \citet{luo2017arxiv}, the projection step then becomes 
\begin{equation}\label{eq:projection}
  \w_t^\eta
    = \check{\w}^\eta_t
      - \frac{\shrink_C(\x_t^\top \check{\w}^\eta_t)}
             {(\x_t^\top \x_t - 2\eta^2 \x_t^\top (\S^\eta_{t-1})^\top \H_{t-1}^\eta \S^\eta_{t-1} \x_t)}
      (\x_t - 2\eta^2 (\S^\eta_{t-1})^\top \H_{t-1}^\eta \S^\eta_{t-1} \x_t),
\end{equation}
and the update step can be written (with $\approxgrad_t^\eta$ as in Equation~\ref{eq:sketching.ghat}) as
\begin{equation}\label{eq:sketched.update}
  \check{\w}_{t+1}^\eta
    = \w^\eta_t - \sigma^2 (\approxgrad_t^\eta - 2\eta^2 (\S^\eta_t)^\top \H_t^\eta \S^\eta_t \approxgrad_t^\eta).
\end{equation}
Assuming that $\S^\eta_t$ and $\H_t^\eta$ can be efficiently maintained, the
operations involving $\S^\eta_t \x_t$ or $\S^\eta_t \approxgrad_t^\eta$ require $O(kd)$
computation time and matrix-vector products with $\H_t^\eta$ can be performed
in $O(k^2)$ time. As noted by \citet{luo2017arxiv}, both of these are
only a factor~$k$ more than the $O(d)$ time required by first-order
methods. They describe two sketching techniques to maintain $\S^\eta_t$ and
$\H_t^\eta$, each requiring $O(kd)$ storage and $O(kd)$ amortised computation time per round. The first technique is based on Frequent
Directions (FD) sketching; the other one on Oja's algorithm. We adopt
the FD approach, which comes with a guaranteed bound on the regret.
\citet{luo2017arxiv} further develop an extension of FD for sparse
gradients, and yet another option in the literature is the Robust Frequent
Directions sketching method of \citet{nothaipengluo2017robust}.

\subsubsection{Frequent Directions Sketching}

Some sketching approaches are randomized, but Frequent Directions
sketching \citep{GhashamiEtAl2016} is a deterministic method. The
simplest version \citep[Algorithm~2]{luo2017arxiv} performs a singular
value decomposition (SVD) of $\S^\eta_t$ every round at the cost of $O(k^2
d)$ computation time, but there also exists a refined epoch-based
version which only performs an SVD once per epoch. Each epoch takes $m$
rounds and $k=2m$, leading to an amortised runtime of $O(kd)$ per round.
We describe here the epoch version, adapted from Algorithm~6 of
\citet{luo2017arxiv} and summarized in Algorithm~\ref{alg:fdsketch}.

\begin{algorithm2e}[htbp]
\renewcommand{\algorithmicrequire}{\textbf{Input:}}
\renewcommand{\algorithmiccomment}[1]{\hfill \ensuremath{\triangleright}~\textit{#1}}
  \caption{Frequent Directions Sketching}
  \label{alg:fdsketch}
  \begin{algorithmic}[1]
\REQUIRE \multiline{Sketch rank~$m$, first active round~$a \equiv a^\eta$}
    \STATE Initialize $\S^\eta_{a-1} = \0_{2m \times d}$, and $\H_{a-1}^\eta = \sigma^2
    \I_{2m}$.
    \FOR{$t=a,a+1,\ldots$}
      \STATE Receive $\grad_t$
      \STATE Let $\tau = (t-a) \bmod (m+1)$ and write $\grad_t^\top$ to row $(m+\tau)$ of $\S^\eta_{t-1}$ to obtain $\tilde \S$
      \IF{$\tau < m$}
        \STATE Set $\S^\eta_t = \tilde \S$
        \STATE \multiline{Let $\e \in \reals^{2m}$ be the basis vector in direction
        $m+\tau$
        \newline and $\q = 2\eta^2 (\tilde \S \grad_t - \frac{\grad_t^\top
        \grad_t}{2} \e)$}
        \STATE Update $\H_t^\eta = \tilde \H - \frac{\tilde \H \e
        \q^\top \tilde \H}{1+ \q^\top \tilde \H \e}$, where $\tilde \H =
        \H_{t-1}^\eta - \frac{\H_{t-1}^\eta \q \e^\top
        \H_{t-1}^\eta}{1+\e^\top \H_{t-1}^\eta \q}$
      \ELSE
        \STATE \multiline{From the SVD of $\tilde \S$, compute the top-$m$
        singular values $\sigma_1 \geq \cdots \geq \sigma_m$ and
        corresponding right-singular vectors as 
        $\V
        \in \reals^{d \times m}$}
        \STATE Set $\S^\eta_t = \begin{pmatrix}
          \diag(\sigma_1^2-\sigma_m^2,\ldots,\sigma_m^2-\sigma_m^2)^{1/2} \V^\top
          \\
          \0_{m \times d}
        \end{pmatrix}
        $
        \STATE Set $\H_t^\eta = \diag(\frac{1}{\sigma^{-2} + 2\eta^2
        (\sigma_1^2-\sigma_m^2)},\ldots,\frac{1}{\sigma^{-2} + 2\eta^2 (\sigma_m^2-\sigma_m^2)},
        \frac{1}{\sigma^{-2}}, \ldots, \frac{1}{\sigma^{-2}})$
      \ENDIF
    \ENDFOR
  \end{algorithmic}
\end{algorithm2e}

Recall that $(\S^\eta_t)^\top \S^\eta_t$ is an approximation of
$(\G_t^\eta)^\top \G_t^\eta$. At
the start of each epoch, we have the invariant that only the first $m-1$
rows of $\S^\eta_t$ contribute to this approximation and the remaining $m+1$
rows are filled with zeros. During the $\tau$-th round in any epoch we
first write the incoming gradient $\grad_t^\top$ to row $m+\tau$ of
$\S^\eta_{t-1}$ to obtain an intermediate result $\tilde \S$. If we are not
yet in the last round of the epoch (i.e.\ $\tau < m$), then we simply
set $\S^\eta_t = \tilde \S$, and use \eqref{eq:H} to see that
\[
  (\H_t^\eta)^{-1} = (\H_{t-1}^\eta)^{-1} + \q \e^\top + \e \q^\top,
\]
where $\e \in \reals^{2m}$ is the basis vector in direction $m+\tau$ and
$\q = 2\eta^2 (\tilde \S \grad_t - \frac{\grad_t^\top \grad_t}{2}
\e)$. It
follows that we can compute $\H_t^\eta$ from $\H_{t-1}^\eta$ using two
rank-one updates with the Sherman-Morrison formula:
\[
  \H_t^\eta
    = \tilde \H -
        \frac{\tilde \H \e \q^\top \tilde \H}{1+ \q^\top \tilde \H \e},
        \text{ where }
  \tilde \H 
    = \H_{t-1}^\eta - \frac{\H_{t-1}^\eta \q \e^\top
    \H_{t-1}^\eta}{1+\e^\top \H_{t-1}^\eta \q}.
\]
Otherwise, if we are in the last round of the epoch (i.e.\ $\tau = m$),
the invariant is restored by eigen decomposing $\tilde \S^\top \tilde
\S$ into $\W \bLambda \W^\top$, where $\bLambda =
\diag(\lambda_1,\ldots,\lambda_{2m})$ contains the potentially non-zero
eigenvalues in non-decreasing order $\lambda_1 \geq \cdots \geq
\lambda_{2m}$ and the columns of $\W \in \reals^{d \times 2m}$ contain
the corresponding eigenvectors. Then we set $\S^\eta_t = \diag(\lambda_1 -
\lambda_m, \ldots, \lambda_m - \lambda_m, 0, \ldots, 0)^{1/2} \W^\top$.
Since the rows of $\S^\eta_t$ are now orthogonal,
\begin{align*}
  \H_t^\eta
    &= (\tfrac{1}{\sigma^2} \I_{2m} + 2\eta^2 \S^\eta_t (\S^\eta_t)^\top)^{-1}\\
    &= \diag\Big(\frac{1}{\sigma^{-2} + 2\eta^2(\lambda_1 - \lambda_m)}, \ldots,
    \frac{1}{\sigma^{-2} + 2\eta^2(\lambda_m - \lambda_m)}, \frac{1}{\sigma^{-2}},
    \ldots, \frac{1}{\sigma^{-2}}\Big)
\end{align*}
is a diagonal matrix.

\subsubsection{Implementation Details}

When implementing the FD procedure, we can calculate the eigen
decomposition of $\tilde \S^\top \tilde \S$ via an SVD of $\tilde \S$,
which can be performed in $O(m^2 d)$ computation steps. The eigenvalues
$\lambda_i$ then correspond to the squared singular values $\sigma_i^2$
of $\tilde \S$, and $\W$ contains the corresponding right-singular
vectors. In fact, we only need the top-$m$ singular values and the
corresponding $m$ right-singular vectors $\V \in \reals^{d \times m}$ to
compute $\S^\eta_t = \diag(\lambda_1 - \lambda_m, \ldots, \lambda_m -
\lambda_m, 0, \ldots, 0)^{1/2} \W^\top = \diag(\sigma_1^2 - \sigma_m^2,
\ldots, \sigma_m^2 - \sigma_m^2)^{1/2} \V^\top$.

\subsubsection{Practical Considerations}

Sketching introduces an extra hyperparameter $k = 2m$, which controls
the sketch size. The sketch keeps track of $m-1$ dimensions, so in
theory we expect that larger $k$ provides a better approximation of the
full version of MetaGrad, at the cost of more computation. We indeed
observe this in practice in the experiments in
Section~\ref{sec:experiments}.

\subsection{Coordinate MetaGrad}\label{sec:coordgrad}

\citet{adagrad} introduce a full and a diagonal version of their AdaGrad
algorithm. The diagonal version, which is the version that is widely
used in applications, may be interpreted as running a copy of online
gradient descent \citep{Zinkevich2003} for each dimension separately,
with a separate data-dependent tuning of the step size per dimension.
This approach of running a separate copy per dimension can be applied to
any online learning algorithm, and works out as follows.

We output a joint prediction $\w_t = (w_{t,1},\ldots,w_{t,d})^\top$,
where each $w_{t,i}$ is the output of the copy of the algorithm for
dimension $i$. Each of these copies gets as inputs the $1$-dimensional
losses $f_{t,i}(w) = w g_{t,i}$, where $g_{t,i}$ is the $i$-th component
of the joint gradient $\grad_t = \nabla f_t(\w_t)$. This works because
the linearized regret decomposes per dimension:
\[
  \sum_{t=1}^T (\w_t - \u)^\top \grad_t
    = \sum_{i=1}^d \sum_{t=1}^T \del{f_{t,i}(w_{t,i}) - f_{t,i}(u_i)},
\]
so our joint linearized regret is simply the sum of the linearized
regrets per dimension.

One limitation of this approach, if we apply it as is, is that the
domain cannot introduce dependencies between the dimensions, so we are
limited to rectangular domains:
\[
   \U^\text{rect}_t = \{ \w \in \reals^d \mid -D_{t,i} \leq w_i \leq
   D_{t,i} \text{ for } i = 1,\ldots,d\},
\]
with our only freedom consisting of choosing the side lengths $D_{t,i}$.

\subsubsection{Practical Considerations}
\label{sec:metagradcoordpractical}

The bounds $b_t$ on the gradients now become a separate bound per
dimension:
\[
  b_{t,i}
    \df \max_{w_i \in [-D_{t,i},D_{t,i}]} |(w_i -
    w_{t,i})g_{t,i}|
    = (D_{t,i} + |w_{t,i}|)|g_{t,i}|,
  \qquad
  B_{t,i} = \max_{s \leq t} b_{s, i}. 
\]
Running a copy of MetaGrad per dimension potentially introduces a
separate hyperparameter $\sigma_i$ per dimension $i$. Like
\citet{adagrad}, we reduce the complexity of hyperparameter tuning by
letting $\sigma_i = \sigma$ be the same for all dimensions. In line with
the discussion in Section~\ref{sec:metagradfullpractical}, the
recommended setting for $\sigma$ then becomes (an overestimate of) the
$\ell_\infty$-norm of the comparator $\u$. If no specific domain is
required and the components of the gradients are approximately
standardized, it is also generally sufficient to set the dimensions of
the rectangular domain to $D_{t,i} = D_\infty$ for a fixed parameter
$D_\infty$.

\section{Analysis of the Full Matrix Version of MetaGrad}\label{sec:analysisfull}

Recall that MetaGrad runs multiple instances of a baseline
``$\eta$-expert'' algorithm, each with a different candidate tuning of the
learning rate $\eta$. A controller then aggregates the predictions of
these $\eta$-experts
and manages their lifetimes to always have the required
tuning present.
The \metagradfull{} $\eta$-experts are Exponentially Weighted Average
forecasters starting from a Gaussian prior and taking in our quadratic
surrogate losses.
In turn, the controller is a specialists (aka sleeping experts) algorithm to deal with the starting and retiring of $\eta$-experts. When measured in the surrogate loss, the controller ensures a uniform regret bound w.r.t.\ each $\eta$-expert. Yet in the original loss, which is not scaled by $\eta$, this results in a non-uniform regret guarantee, obtaining especially small regret when the best learning rate turns out to be high.
Finally, our approach for adapting to the Lipschitz constant is speculative. Starting at zero, we monitor the implied Lipschitz constant of the incoming gradients. If it is increasing slowly, the controller is able to accommodate the overshoots in a lower-order term. If it makes a large jump, then the controller may need to reset. We do so by resetting the controller weights without changing the state of the affected $\eta$-experts.

\subsection{Controller}

Let us introduce
the concept of expiration to capture when $\eta$-experts become inactive
and are never used again:
\begin{definition}\label{def:expired}
  We say that $\eta \in \mathcal G$ is \emph{expired} after $T$ rounds (or, equivalently, after round $T$) if $\eta > \frac{1}{2 B_{T-1}}$.
\end{definition}
Note that expiration can be checked \emph{before} the round happens (it is ``predictable''). All learning rates used by Algorithm~\ref{alg:controller} by means of the active set $\activeset_t$ \eqref{eqn:activeexperts} are not expired. Also note the ``lifecycle'' of any fixed learning rate $\eta$. It starts inactive unexpired. Then it becomes active unexpired. And finally it expires, after which it loses all relevance.

For the controller, we prove that its behavior approximates that of any $\eta$-expert not expired, when measured in the $\eta$ surrogate loss \eqref{eq:surrogate}.

\begin{lemma}[Controller Surrogate Regret Bound]\label{lem:controllerbd}
  For any learning rate $\eta \in \mathcal G$ not expired after $T$ rounds and any comparator $\u \in \bigcap_{t=1}^T \U_t$, \metagradfull{} ensures
\[
  R_T^\eta(\u)
  ~\le~
  \underbrace{
    \frac{1}{2}
    +  2 \eta B_T
    }_\text{tiny}
    + 2 \underbrace{
      \ln \ceil*{2 \origlog_2 \del*{\sum_{t=1}^{T-1}
          \frac{b_t }{B_t} + 1}}_+
    }_\text{specialist regret for epoch, $O(\ln \ln T)$}
  + \underbrace{
      \sum_{t = a^\eta}^T \del*{
        \surr_t^\eta(\w_t^\eta) - \surr_t^\eta(\u)
      }
    }_\text{$\surr^\eta$-regret of $\eta$-expert w.r.t.\ $\u$}
    ,
  \]
  where we interpret the last sum as $0$ if $a^\eta > T$.
\end{lemma}

The proof is in Appendix~\ref{sec:controllerbdpf}. It follows the MetaGrad analysis of \citet{lipschitz.metagrad}, including the range clipping technique due to \citet{cutkosky2019artificial}, and the reset technique of \citet{lipschitz.metagrad}, which in particular ensures that whenever a reset occurs, the accumulated regret up until the \emph{previous} reset is tiny. As such, we only have to pay for the controller regret for the last two epochs.

We further streamline the approach by using a standard specialists
(sleeping experts) algorithm on a discrete grid of $\eta$-experts with
$\eta \in \mathcal G$ as our controller algorithm. Of note here is our
use of a uniform prior on $\mathcal G$, which is improper in the sense
that it does not sum to one. Improperness does not cause any problems,
because the prior is automatically renormalized on the sets of active
learning rates $\activeset_1,\activeset_2,\ldots$ We also employ a slightly tightened
measure $b_t$ of the effective loss range.

To make further progress, we need to make use of the details of the $\eta$-experts.

\subsection{Full $\eta$-Experts}

Next we establish an $O(d \log T)$ regret bound in terms of the surrogate
loss for each \metagradfull{} $\eta$-expert. The $\eta$-experts
implement the exponentially weighted average forecaster for the
quadratic losses $\surr_t^\eta$ starting from a Gaussian prior.
Alternatively, they may be viewed as instances of mirror descent with a
time-varying quadratic regularizer. The exponentially weighted average
forecaster was previously used for a different quadratic loss arising in
linear regression by \cite{Vovk2001}. Mirror descent for the general
quadratic case goes back (at least) to \cite{ons}. Although they do
not separate the analysis for general quadratic losses from the
reduction of exp-concave losses to quadratics, the ideas
are clearly present. The explicit analysis by \citet{metagrad} includes
an unnecessary range restriction, which was subsequently removed by
\citet{hoeven2018many}. As pointed out by \citet{luo2017arxiv}, the
extension to time-varying domains is trivial.

\begin{lemma}[Surrogate regret bound]\label{lem:surrogateregret}
  Consider the \metagradfull{} $\eta$-expert in Algorithm~\ref{alg:MetaGradExpert} with learning rate $\eta \le \frac{1}{2 B_T}$ starting from time $a^\eta$. Its surrogate regret after round $T \ge a^\eta$ w.r.t.\ any comparator $\u \in \bigcap_{t=a^\eta}^T \U_t$ is bounded by
  \begin{equation*}
    \sum_{t=a^\eta}^T \del*{
      \surr_t^\eta(\w_t^\eta)
      - \surr_t^\eta(\u)
    }
    ~\le~
    \frac{1}{2 \sigma^2} \norm{\u}^2_2
    + \ln \det \del*{\I + 2 \eta^2 \sigma^2  \sum_{t=a^\eta}^T \grad_t\grad_t^\top}
    .
  \end{equation*}
\end{lemma}
We note that the condition on $\eta$ in the lemma is slightly stricter than not being expired (Definition~\ref{def:expired}), which only requires $\eta \le \frac{1}{2 B_{T-1}}$. The reason is that the $\eta$-expert operates off the \emph{unclipped} surrogate loss and gradients.

\begin{proof}
  The $\eta$-expert algorithm implements the exponentially weighted average forecaster with $\surr_t^\eta$ as the quadratic loss, unit learning rate, and with greedy projections (of the mean) onto $\U_t$. By \cite[Proof of Theorem 2]{ons}, we obtain that
\[
  \sum_{t=a^\eta}^T \del*{
    \surr_t^\eta(\w_t^\eta)
    - \surr_t^\eta(\u)
  }
~\le~
\frac{\norm{\u}_2^2}{2 \sigma^2}
+
\frac{1}{2} \sum_{t=a^\eta}^T  \grad_t'^\top \Sigma^\eta_{t+1} \grad_t',
\]
where $\grad_t' = \eta\del*{1+2 \eta \tuple*{\w_t-\w_t^\eta, \grad_t}} \grad_t$ and where we recall that $(\Sigma^\eta_{t+1})^{-1} = \frac{1}{\sigma^2} \I + 2 \eta^2 \sum_{s=a^\eta}^t \grad_s \grad_s^\top$. Expanding, we obtain
\[
  \grad_t'^\top \Sigma^\eta_{t+1} \grad_t'
  ~=~
  \frac{1}{2}\del*{1+2 \eta \tuple*{\w_t-\w_t^\eta, \grad_t}}^2 \cdot 2
  \eta^2 \grad_t^\top \del*{\frac{1}{\sigma^2} \I + 2 \eta^2
  \sum_{s=a^\eta}^t \grad_s \grad_s^\top}^{-1}  \grad_t.
\]
Now we may use that
\begin{equation}\label{eqn:thefactor}
  \frac{1}{2} \del*{1+2 \eta \tuple*{\w_t-\w_t^\eta, \grad_t}}^2 \le \frac{1}{2} (1+2 \eta b_t)^2 \le \frac{1}{2} (1+1)^2 = 2
\end{equation}
by the assumed upper bound on $\eta$. Moreover, abbreviating $\A =
\frac{1}{\sigma^2} \I + 2 \eta^2 \sum_{s=a^\eta}^t \grad_s \grad_s^\top$
and $\B = \frac{1}{\sigma^2} \I + 2 \eta^2 \sum_{s=a^\eta}^{t-1} \grad_s
\grad_s^\top$, concavity of the log determinant implies that
\begin{align*}
  & 2 \eta^2 \grad_t^\top \del*{\frac{1}{\sigma^2} \I + 2 \eta^2 \sum_{s=a^\eta}^t \grad_s \grad_s^\top}^{-1}  \grad_t
   = \tr\del*{\A^{-1}\del*{\A - \B}}
  \le \log \frac{\det(\A)}{\det(\B)}\\
  & ~=~
  \ln \det \del*{\frac{1}{\sigma^2} \I + 2 \eta^2 \sum_{s=a^\eta}^t \grad_s \grad_s^\top}
  - \ln \det \del*{\frac{1}{\sigma^2} \I + 2 \eta^2 \sum_{s=a^\eta}^{t-1} \grad_s \grad_s^\top}
.
\end{align*}
(Lemma~12 of \citet{ons} provides a detailed proof of this inequality.)
Summing over rounds and telescoping, we find
\[
  \frac{1}{2} \sum_{t=a^\eta}^T  \grad_t'^\top \Sigma^\eta_{t+1} \grad_t'
  ~\le~
  \ln \det \del[\Big]{\I + 2 \eta^2 \sigma^2 \sum_{t=a^\eta}^T \grad_t \grad_t^\top}
\]
and obtain the result.
\end{proof}

\subsection{Composition (bounding the actual regret)}
\label{sec:composition}

To complete the analysis of \metagradfull{}, %
we put the regret bounds for the controller and $\eta$-experts together.
We then optimize $\eta$ over the grid $\mathcal G$ to get our main
result. For the purpose of this section, let us define the
\emph{gradient covariance matrix} 
and
\emph{essential horizon} by
\begin{equation}\label{eq:covm.and.horz}
  \F_T ~\df~ \sum_{t=1}^T \grad_t\grad_t^\top
  \qquad
  \text{and}
  \qquad
  Q_T ~\df~ \sum_{t=1}^{T-1} \frac{b_t }{B_t} + 1
  .
\end{equation}

\begin{theorem}[Grid point regret]\label{thm:untuned.regret}
  \metagradfull{} guarantees that the linearized regret w.r.t.\ any
  comparator $\u \in \bigcap_{t=1}^T \U_t$ is at most
  \[
    \Rtrick_T^\u
    ~\le~
    \eta V_T^\u
    + \frac{
      \ln \det \del*{\I + 2 \eta^2 \sigma^2  \F_T}
      + \frac{1}{2 \sigma^2} \norm{\u}^2_2
      + 2 \ln \ceil*{2 \origlog_2 Q_T}_+
      + \frac{1}{2}
    }{\eta}
    + 2 B_T,
  \]
  simultaneously for all $\eta \in \mathcal G$ such that $\eta \le
  \frac{1}{2 B_T}$.
\end{theorem}

\begin{proof}
  Combining the controller and $\eta$-expert surrogate regret bounds
  from Lemma~\ref{lem:controllerbd} and Lemma~\ref{lem:surrogateregret}, we obtain
  \begin{align*}
    \sum_{t=1}^T \del*{
      \surr_t^\eta(\w_t)
      - \surr_t^\eta(\u)
      }
    &~\le~
    \frac{1}{2}
    + 2 \eta B_T
    + 2
    \ln \ceil*{2 \origlog_2 \del*{\sum_{t=1}^{T-1}
      \frac{b_t }{B_t} + 1}}_+
    \\
    &\quad + \frac{1}{2 \sigma^2} \norm{\u}^2_2
    + \ln \det \del*{\I + 2 \eta^2 \sigma^2  \sum_{t=1}^T \grad_t\grad_t^\top}
    .
  \end{align*}
  The definition of the surrogate loss \eqref{eq:surrogate} gives
  $
    \surr_t^\eta(\w_t)
    - \surr_t^\eta(\u)
    =
    \eta (\w_t- \u)^\top \grad_t
    - \big(\eta (\u - \w_t)^\top \grad_t\big)^2
    $
    and the theorem follows by reorganising and dividing by $\eta$.
\end{proof}

The final step is to properly select the learning rate $\eta \in \mathcal G$ in the regret bound Theorem~\ref{thm:untuned.regret}. This leads to our main result. The proof is in Appendix~\ref{appx:composition}.
\begin{theorem}[MetaGrad Full Regret Bound]\label{thm:mainbound}
  For all $\u \in \bigcap_{t=1}^T
  \U_t$ the linearized regret of \metagradfull{} is simultaneously bounded by
  \begin{equation}\label{eqn:mainbound}
    \Rtrick_T^\u
    ~\le~
    \frac{5}{2} \sqrt{V_T^\u (\tfrac{1}{2 \sigma^2} \norm{\u}^2_2 + Z_T)} + 5
    B_T (\tfrac{1}{2 \sigma^2} \norm{\u}^2_2 + Z_T) + 2 B_T,
  \end{equation}
  where 
    $Z_T =
    \rk(\F_T) \ln \del*{1 + \frac{\sigma^2\sum_{t=1}^T \|\grad_t\|_2^2}{2
    B_T^2 \rk(\F_T)}}
    + 2 \ln \ceil*{2 \origlog_2 T}_+
    + \frac{1}{2}$,
  and by
  \[
    \Rtrick_T^\u
    ~\le~
    \frac{5}{2} \sqrt{\Big(V_T^\u
    + 2 \sigma^2 \sum_{t=1}^T \|\grad_t\|_2^2
    \Big)
    \Big(
    \tfrac{1}{2 \sigma^2} \norm{\u}^2_2
    + Z'_T
    \Big)
    } + 5 B_T \Big(
    \tfrac{1}{2 \sigma^2} \norm{\u}^2_2
    + Z'_T
    \Big)
    + 2 B_T,
  \]
  where $Z'_T = 2 \ln \ceil*{2 \origlog_2 T}_+ + \frac{1}{2}$.
\end{theorem}
Here the rank $\rk(\F_T) \leq d$ plays the role of an effective
dimension. If the eigenvalues of $\F_T$ satisfy a decay condition, then
a more refined bound on $Z_T$ is possible, as can be seen from the
proof. The recommended tuning is to set $\sigma$ to (an upper bound on)
$\|\u\|_2$. For this case, we obtain the following corollary, which is
proved in Appendix~\ref{app:fulldiagcorollaries}:
\begin{corollary}\label{cor:roughthm}
Suppose the domain $\U_t = \U$ is fixed with finite radius
$D_2:=\max_{\u \in \U} \|\u\|_2$, and we tune $\sigma=D_2$. Then, if all
gradients are uniformly bounded by $\|\grad_t\|_2 \leq G_2$, the
linearized regret of \metagradfull{} with respect to any $\u \in \U$ is
bounded by
\begin{equation}\label{eqn:roughbound}
\Rtrick_T^\u
~=~
O\del*{
\sqrt{
  V_T^\u\,
  d \ln\frac{D_2 G_2 T}{d}
}
+ D_2 G_2 d \ln\del*{\frac{D_2 G_2 T}{d}}
},
\end{equation}
and it is simultaneously bounded by
\begin{equation*}
\Rtrick_T^\u
~=~
O\del*{
  D_2 G_2 \sqrt{T \log \log T}
}.
\end{equation*}
\end{corollary}

\subsubsection{Sensitivity to $\sigma$-Tuning}

The recommended tuning for $\sigma$ is to set it to (an upper bound on)
$\|\u\|_2$. In the first result of Theorem~\ref{thm:mainbound}, which
covers the regime that $\rk(\F_T)$ is relatively small compared to $T$,
the effect of overestimating $\|\u\|_2$ is minor, because $Z_T$ depends
only logarithmically on $\sigma$. In the second result of the Theorem,
however, which covers the high-dimensional setting, the effect of $\sigma$ is
similar to the usual dependence of Online Gradient Descent on a guess
for $\|\u\|_2$ \citep{Zinkevich2003,ShalevShwartz2012} and taking
$\sigma$ much larger affects the rate linearly. In both regimes,
underestimating $\|\u\|_2$ when tuning $\sigma$ may degrade
performance.

\subsubsection{An Undesirable Reparametrization}

If underestimating $\|\u\|_2$ is a major concern, then it is possible to
reparametrize in terms of a new tuning parameter $\alpha > 0$ by setting
$\sigma = \frac{1}{\sqrt{\alpha \eta}}$, as done by
\citet{luo2017arxiv}. This means that each $\eta$-expert is now using a
different choice for $\sigma$, but all our results up to
Theorem~\ref{thm:untuned.regret} still go through. Optimizing $\eta$
then leads to the following variant of Theorem~\ref{thm:mainbound},
proved in Appendix~\ref{appx:composition}:
\begin{theorem}[The Road Not Taken]\label{thm:roadnottaken}
  Suppose we tune each $\eta$-expert in \metagradfull{} with $\sigma =
  1/\sqrt{\alpha \eta}$ for a given $\alpha > 0$. Then for all $\u \in
  \bigcap_{t=1}^T \U_t$ its linearized regret is simultaneously bounded
  by
  \[
    \Rtrick_T^\u ~\le~ \frac{5}{2} \sqrt{V_T^\u Z_T } + 5 B_T Z_T
    + \frac{\alpha}{2} \norm{\u}^2_2 + 2 B_T,
  \]
  where 
    $Z_T =
      \rk(\F_T) \ln (1 + \frac{\sum_{t=1}^T \|\grad_t\|_2^2}{B_T \alpha \rk(\F_T)})
      + 2 \ln \ceil*{2 \origlog_2 T}_+
      + \frac{1}{2}$,
  and by
  \[
    \Rtrick_T^\u
    ~\le~
    \frac{5}{2} \sqrt{V_T^\u Z'_T} + 5 B_T Z'_T
    + \frac{2}{\alpha} \sum_{t=1}^T \|\grad_t\|_2^2
    + \frac{\alpha}{2} \norm{\u}^2_2
    + 2 B_T,
  \]
  where $Z'_T = 2 \ln \ceil*{2 \origlog_2 T}_+ + \frac{1}{2}$.
\end{theorem}
This result is of a similar flavor as Theorem~\ref{thm:mainbound} if we
set $\alpha = 1/\|\u\|_2^2$ in the first inequality and $\alpha =
2 \sqrt{\sum_{t=1}^T \|\grad_t\|_2^2}/\|\u\|_2$ in the second
inequality. A potential gain is that tuning $\alpha$ may be easier in
case of the first inequality: the term $\frac{\alpha}{2} \|\u\|_2^2$ may
not be dominant even if our choice of $\alpha$ is significantly off from
the optimal tuning. But we pay significantly for this convenience,
because there no longer exists a single choice for $\alpha$ that works
both for the first and the second inequality simultaneously, which is
why we do not advocate this reparametrization in terms of $\alpha$.

\subsubsection{The Computationally Efficient Domain from
Section~\ref{sec:sketching}}

A further important case to consider is when $\U_t$ is the computationally
efficient domain from \eqref{eqn:luodomain}, for which the diameter is
not bounded. This domain presumes that losses take the form $f_t(\w) =
h_t(\w^\top \x_t)$ for a convex function $h_t$. Under the Lipschitz
assumption that $|h'_t(z)| \leq L$ for all $|z| \leq C$,
\citet{luo2017arxiv} show a lower bound of $\Theta(\sqrt{dT})$ on the
worst-case regret. They further obtain a (nearly) matching upper bound
of 
\[
  R_T^\u = O\del*{\sqrt{d T} \log \tfrac{\sum_{t=1}^T \|\grad_t\|_2^2}{\alpha} + \alpha \|\u\|_2^2}
  \qquad \text{for all $\u \in \bigcap_{t=1}^T \U_t$}
\]
with a variant of ONS, where $\alpha > 0$ is a tuning parameter similar
to the $\alpha$ in Theorem~\ref{thm:roadnottaken}. The first results of
Theorems~\ref{thm:mainbound} and \ref{thm:roadnottaken} improve on this
in that they improve the dependence on $T$ to $V_T^\u \leq L^2 C^2 T$,
they only depend on the effective dimension via $\rk(\F_T) \leq d$ and
the logarithmic factor is moved inside the square root.

\section{Analysis of the Faster Extension Algorithms}
\label{sec:analysisextensions}
In this section analyse the sketched and coordinate-wise versions of MetaGrad.

\subsection{Sketching: Analysis}
We will refer to the Frequent Directions sketching version of MetaGrad
as \metagradsketch{}. Its analysis with sketch size $k=2m$ proceeds like
the analysis of the full matrix version, except that we obtain a
different bound for the $\eta$-expert regret. This bound depends on the
spectral decay of $\F_T = \sum_{t=1}^T \grad_t \grad_t^\top$. Let
$\lambda_i$ be the $i$-th eigenvalue of $\F_T$ and define $\Omega_q =
\sum_{i=q+1}^d \lambda_i$. Then the surrogate regret of the
$\eta$-expert algorithm with FD sketching is bounded as follows:
\begin{lemma}\label{lem:FreqD}
  Consider the sketching version of the MetaGrad $\eta$-expert algorithm
  with sketch size parameter $m$, learning rate $\eta \le \frac{1}{2
  B_T}$, and starting from time $a^\eta$. Its surrogate regret after
  round $T \ge a^\eta$ w.r.t.\ any comparator $\u \in
  \bigcap_{t=a^\eta}^T \U_t$ is bounded by 
  \begin{equation*}
  \begin{split}
  \sum_{t=a^\eta}^T \del*{\surr_t^\eta(\w_t^\eta) -  \surr_t^\eta(\u)}
    \leq \frac{1}{2D^2}||\u||_2^2
        + \log(\det(\I + 2 \eta^2 \sigma^2 (\S^\eta_T)^\top \S^\eta_T))
        + \frac{2 \eta^2 \sigma^2 m \Omega_q}{m-q}
  \end{split}
  \end{equation*}
  for any $q = 0,\ldots, m-1$.
\end{lemma}
Compared to Lemma~\ref{lem:surrogateregret}, we see that
$\sum_{t=a^\eta}^T \grad_t \grad_t^\top = (\G_T^\eta)^\top \G_T^\eta$ in the logarithmic term has been
replaced by its sketching approximation $(\S^\eta_T)^\top \S^\eta_T$. We therefore
pay logarithmically for the top $m$ directions, which are captured by
the sketch. What we lose is the rightmost term of order $O(\eta^2
\Omega_q)$, which corresponds to the remaining $d-q$ directions that are
not captured.

The proof of Lemma~\ref{lem:FreqD} is a straightforward adaptation of
the proof of Theorem~3 by \citet{luo2017arxiv}. For the details, we
refer to Chapter~4 of \citet{deswarte2018linear}, with three minor
remarks: the first is that Deswarte imposes a slightly stricter upper bound on
$\eta$, which allows him to bound $\frac{1}{2} \del*{1+2 \eta
\tuple*{\w_t-\w_t^\eta, \grad_t}}^2 \le 1$, whereas we get an upper
bound of~$2$ from \eqref{eqn:thefactor} and therefore obtain a final
result that is a factor of $2$ larger. The second remark is that our
$\eta$-expert algorithm is started in round $a^\eta$ instead of
round~$1$, leading to a bound involving $\Omega_q^\eta = \sum_{i=q+1}^d
\lambda_i^\eta$, where $\lambda_i^\eta$ is the $i$-th eigenvalue of
$\sum_{t=a^\eta}^T \grad_t \grad_t^\top$. For simplicity, we immediately use Weyl's inequality to bound $\Omega_q^\eta
\leq \Omega_q$, because the difference is minor. Finally, we have
described the fast version of FD sketching, which corresponds to
Algorithm~6 of \citet{luo2017arxiv} instead of the simpler slow version
in their Algorithm~2. They and Deswarte consider the slow version in
their analysis, but this makes no difference for the proof because the
fast algorithm satisfies the same guarantees \citep{GhashamiEtAl2016}.
Analogously with Theorem~\ref{thm:untuned.regret}, we find:
\begin{theorem}[Sketching Grid Point Regret]\label{thm:untuned.sketchregret} Let $\eta
\in \mathcal G$ be such that $\eta \le \frac{1}{2 B_T}$. Then
\metagradsketch{} with sketch size parameter $m$ guarantees that the
linearized regret w.r.t.\ any comparator $\u \in \bigcap_{t=1}^T \U_t$
is at most
  \begin{align*}
    \Rtrick_T^\u
    &~\le~
    \eta V_T^\u
    + \frac{2 \eta \sigma^2 m \Omega_q}{m-q}
      + 2 B_T
    \\
    &
      \qquad + \frac{
      \ln \det \del*{\I + 2 \eta^2 \sigma^2 (\S^\eta_T)^\top \S^\eta_T}
      + \frac{1}{2 \sigma^2} \norm{\u}^2_2
      + 2 \ln \ceil*{2 \origlog_2 Q_T}_+
      + \frac{1}{2}
      }{\eta}
  \end{align*}
  for any $q = 0,\ldots, m-1$. Recall that $Q_T$ is defined in \eqref{eq:covm.and.horz}.
\end{theorem}

As shown in Appendix~\ref{appx:composition}, optimizing $\eta$ and bounding $(\S^\eta_T)^\top \S^\eta_T$ appropriately leads to
the following final result:
\begin{theorem}[MetaGrad Sketching Regret Bound]\label{thm:mainsketchbound}
  For all $\u \in \bigcap_{t=1}^T \U_t$ the linearized regret of
  \metagradsketch{} with sketch size parameter $m$ is simultaneously
  bounded by
  \begin{align*}
    \Rtrick_T^\u
    ~&\le~
    \frac{5}{2} \sqrt{\del[\Big]{V_T^\u + \frac{2 \sigma^2 m \Omega_q}{m-q}}
    \del[\Big]{\tfrac{1}{2 \sigma^2} \norm{\u}^2_2 + Z_T}} + 5
    B_T (\tfrac{1}{2 \sigma^2} \norm{\u}^2_2 + Z_T) + 2 B_T,
  \intertext{where 
    $Z_T =
    2m \ln \del*{1 + \frac{\sigma^2 \sum_{t=1}^T \|\grad_t\|_2^2}{4
    B_T^2 m}}
    + 2 \ln \ceil*{2 \origlog_2 T}_+
    + \frac{1}{2} = O(m \ln T)$,
  and by}
    \Rtrick_T^\u
    ~&\le~
    \frac{5}{2} \sqrt{\Big(V_T^\u
    + 2 \sigma^2 \sum_{t=1}^T \|\grad_t\|_2^2
    + \frac{2 \sigma^2 m \Omega_q}{m-q}
    \Big)
    \Big(
    \tfrac{1}{2 \sigma^2} \norm{\u}^2_2
    + Z'_T
    \Big)
    }\\
    &\quad + 5 B_T \Big(
    \tfrac{1}{2 \sigma^2} \norm{\u}^2_2
    + Z'_T
    \Big)
    + 2 B_T
  \end{align*}
  for any $q = 0,\ldots, m-1$, where $Z'_T = 2 \ln \ceil*{2 \origlog_2 T}_+
  + \frac{1}{2}$.
\end{theorem}
Compared to Theorem~\ref{thm:mainbound},
we see the additional term involving~$\Omega_q$, which corresponds to
the directions not captured by the sketch. We also see that $\rk(\F_T)
\leq d$ got replaced by $2m$ in the definition of $Z_T$. This comes from
the analogous upper bound $\rk((\S^\eta_T)^\top \S^\eta_T) \leq 2m$.

\subsection{Coordinate MetaGrad: Analysis}
\label{sec:metagradcoordanalysis}

The analysis of the coordinate version of MetaGrad, which we call 
\metagraddiag{}, is straightforward as we can simply apply the regret
bound of \metagradfull{} to each dimension and add up the bounds:
\begin{theorem}\label{th:diag}
Let $V_{T, i}^{u_i} = \sum_{t=1}^T (u_i - w_{t, i})^2 g_{t, i}^2$. For any $\u \in \bigcap_{t=1}^T
  \U_t^\textnormal{rect}$, the linearized regret of \metagraddiag{} is simultaneously bounded by 
\begin{equation}\label{eqn:coordV}
\Rtrick_T^\u \leq \sum_{i = 1}^d  \Big\{\frac{5}{2} \sqrt{V_{T,
i}^{u_i}(\tfrac{1}{2 \sigma^2} u_i^2 + Z_{T, i})} + 5 B_{T, i} (\tfrac{1}{2
\sigma^2} u_i^2+ Z_{T, i}) + 2 B_{T, i}\Big\},
\end{equation}
where 
    $Z_{T, i} ~=~
    \ln \del*{1 + \frac{\sigma^2\sum_{t=1}^T g_{t, i}^2}{8 B_{T, i}^2}}
    + 2 \ln \ceil*{2 \origlog_2 T}
    + \frac{1}{2}$,
  and by
\begin{equation}\label{eqn:coordhighdim}
    \Rtrick_T^\u
    ~\le~
     \sum_{i = 1}^d \Big\{ \frac{5}{2} \sqrt{\Big(V_{T, i}^{u_i}
    + 2 \sigma^2 \sum_{t=1}^T g_{t, i}^2
    \Big)
    \Big(
    \tfrac{1}{2 \sigma^2} u_i^2
    + Z'_T
    \Big)
    } + 5 B_{T, i} \Big(
    \tfrac{1}{2 \sigma^2} u_i^2
    + Z'_T
    \Big)
    + 2 B_{T, i}\Big\},
\end{equation}
where $Z'_T = 2 \ln \ceil*{2 \origlog_2 T} + \frac{1}{2}$.
\end{theorem}
The recommended tuning is to set $\sigma$ to (an upper bound on)
$\|\u\|_\infty$. For this case, we obtain the following corollary, which
is proved in Appendix~\ref{app:fulldiagcorollaries}:
\begin{corollary}\label{cor:coordroughthm}
Suppose the domain is a fixed rectangle: $\U_t = \U^\textnormal{rect}$,
and we tune $\sigma=D_\infty:=\max_{\u \in \U^\textnormal{rect}}
\|\u\|_\infty$ based on the size of the domain. Let $g_{1:T,i} :=
(g_{i,1},\ldots,g_{i,T})$. Then the linearized regret of \metagraddiag{}
with respect to any $\u \in \U^\textnormal{rect}$ is bounded by
\begin{equation}\label{eqn:roughcoordbound}
\Rtrick_T^\u
~=~
O\del*{
  \sum_{i=1}^d
\sqrt{
  V_{T,i}^{u_i}\,
  \ln(D_\infty G_\infty T)
}
+ D_\infty G_\infty d \ln(D_\infty G_\infty T)
},
\end{equation}
provided that $\|\grad_t\|_\infty \leq G_\infty$ for all $t$, and it is
simultaneously bounded by
\begin{equation}\label{eqn:roughhighdimcoordbound}
\Rtrick_T^\u
~=~
O\del*{
  D_\infty \sum_{i=1}^d \|g_{1:T,i}\|_2 \sqrt{\log \log T} 
  + D_\infty G_2 \sqrt{d} \log \log T
}
~=~
O\del*{
  D_\infty G_2 \sqrt{d T \log \log T}
},
\end{equation}
provided that $\|\grad_t\|_2 \leq G_2$ for all $t$.
\end{corollary}
The first result of the corollary, \eqref{eqn:roughcoordbound}, is
sufficient to obtain fast rates under a coordinate version of the
Bernstein condition, which is discussed below.
The second result, \eqref{eqn:roughhighdimcoordbound}, shows that we
simultaneously recover the
regret bound of order $\tilde O\del*{\sum_{i=1}^d D_\infty \sum_{i=1}^d
\|g_{1:T,i}\|_2}$ that is the main feature of the diagonal version of
AdaGrad. As pointed out by \citet{adagrad}, the norms $\|g_{1:T,i}\|_2$
can be significantly smaller than $T$ when the gradients are sparse, and
the dependence on $D_\infty$ is appropriate when the optimal parameters
$\u$ form a dense vector. When the gradients are not sparse the bound
degrades to $\tilde O\del*{D_\infty \sqrt{dT}}$, which is optimal over
rectangular domains under an $\ell_2$ or even $\ell_1$ bound on the gradients:
if we encounter $T/d$ axis-aligned gradients per dimension, then each
dimension can contribute $\Omega(\sqrt{T/d})$ to the regret, which gives
$\Omega(\sqrt{dT})$ regret in total.

\subsubsection{Open Problem: Restricted Domains}

Most online learning methods can deal with arbitrary convex domains
using projections, but we have presented Coordinate MetaGrad only for
rectangular domains. Can it be extended to other domains, preferably
without incurring any significant computational overhead? One approach
we have tried is to apply the black-box reduction of
\citet{CutkoskyOrabona2018}, which would run \metagraddiag{} with fake
gradients $\tilde \grad_t$ to obtain iterates $\tilde{\w}_t$ from a
rectangular domain, which are then projected onto the true domain $\U_t$
to obtain final iterates $\w_t$. Formally, this reduction goes through,
but it leads to a regret bound in which the terms $V_{T,i}^{u_i}$ are
replaced by unsatisfactory surrogates $\tilde V_{T,i}^{u_i}$ that are
measured in terms of the fake gradients $\tilde g_{t,i}$ and the wrong,
unprojected parameters $\tilde{w}_{t,i}$. This can partially be remedied,
because the reduction guarantees that
\begin{equation}\label{eqn:faketorealgrad}
  \|\tilde \grad_t\|_* \leq \|\grad_t\|_*,
\end{equation}
where $\|\cdot\|_*$ is the dual norm of the norm $\|\cdot\|$ that is
used to project $\tilde{\w}_t$ onto the domain. If $f_t(\w) = h_t(\w^\top
\x_t)$ for a convex function $h_t$ and $\x_t \in \reals^d$ is available
before we choose $\w_t$, then $\nabla f_t(\w) = h'_t(\w^\top \x_t) \x_t$
and we can project with the norm $\|\w\|_{\x_t} = \sum_{i=1}^d |x_{t,i}|
|w_i|$, which leads to the dual norm $\|\grad\|_{\x_t,*} = \max_i
\frac{|g_i|}{|x_{t,i}|}$. Plugging this into \eqref{eqn:faketorealgrad}
and simplifying, we then find that 
\[
  |\tilde g_{t,i}| \leq |g_{t,i}|   \qquad \text{for $i=1,\ldots,d$,}
\]
which implies that
\[
  \tilde V_{T,i}^{u_i}
    = \sum_{t=1}^T (u_i - \tilde{w}_{t,i})^2 \tilde g_{t,i}^2
    \leq \sum_{t=1}^T (u_i - \tilde{w}_{t,i})^2 g_{t,i}^2.
\]
We can thus get rid of the dependence on the fake gradients, but the
dependence on the wrong iterates $\tilde{\w}_t$ remains, so in the end
the black-box reduction only gets us half of the way. This is
unsatisfying, because the wrong iterates $\tilde{\w}_t$ do not lead to
fast rates under the coordinate Bernstein condition described below. It is
an open problem whether there exists another (computationally efficient)
approach that fully preserves the original regret bounds from
Corollary~\ref{cor:coordroughthm} for non-rectangular domains, and
therefore does achieve these fast rates. In light of this open problem,
it is interesting to remark that the black-box reduction can be made to
work for MetaGrad Full, as exploited by
\citet[Theorem~10]{lipschitz.metagrad}.

\subsubsection{Coordinate Bernstein Condition}
\label{sec:coordBernstein}

Since the coordinate version of MetaGrad does not keep track of a full
covariance matrix $\Sigma_T$, we cannot expect it to exploit the
Bernstein condition in all cases. An appropriate modification is the
following \emph{coordinate $(B,\beta)$-Bernstein condition}:
\begin{equation}\label{eqn:coordbernstein}
 \sum_{i=1}^d (w_i - u_i^*)^2
\ex_f \sbr*{
  [\nabla f(\w)]_i^2
}
~\le~
B
\big((\w - \u^*)^\top  \ex_f \sbr*{\nabla f(\w)}\big)^\beta
\qquad \text{for all $\w \in \U$,}
\end{equation}
where we have again assumed that the domain $\U_t = \U$ does not vary
between rounds, and that the losses $f_t$ are independent, identically
distributed. The following theorem, proved in
Appendix~\ref{app:MoreFastRateExamplesAndProofs}, is analogous to
Theorem~\ref{thm:Bernstein}: it shows that, under the coordinate
Bernstein condition, the coordinate version of MetaGrad achieves fast
rates:
\begin{theorem}\label{thm:coordBernstein}
If the gradients satisfy the coordinate $(B,\beta)$-Bernstein condition for $B > 0$
and $\beta \in (0,1]$ with respect to $\u^* = \argmin_{\u \in \U} \E_{f
\sim \pr}[f(\u)]$, then any method with regret bound
\eqref{eqn:roughcoordbound} incurs expected regret
\[
  \E\sbr{R_T^{\u^*}} = O\del*{\del*{B d \ln T}^{1/(2-\beta)}
  T^{(1-\beta)/(2-\beta)} + d\ln T}.
\]
\end{theorem}
So when can we expect the coordinate Bernstein condition to hold? If the
covariances between the coordinates of the gradients are zero, then the
ordinary Bernstein condition reduces to the coordinate Bernstein
condition with the same $B$ and~$\beta$, but this is a very strong
assumption that seems of limited practical use. The following result
captures how this assumption may be significantly relaxed while still
obtaining the same dependence on~$\beta$, albeit at the cost of a worse
constant~$B$. It considers the case that the losses are of the form
$f_t(\w) = h_t(\w^\top \x_t)$ for a convex function $h_t$, with
$(\x_1,h_1),\ldots,(\x_T,h_T)$ independent, identically distributed. Let
$h_t'(z)$ denote the (sub)derivative of $h_t$ at $z$.
\begin{theorem}\label{thm:generaltocoordbernstein}
  Suppose that $0 < L_- \leq |h_t'(\w^\top \x)| \leq L_+$ for all $\w
  \in \U$ and that $\E[\x \x^\top] \succeq C \diag(\E[\x \x^\top])$ for
  some $C > 0$. Then
  \[
    \ex_f \sbr*{
      \nabla f(\w) \nabla f(\w)^\top
    }
    \succeq
      \frac{C L_-^2}{L_+^2}
      \ex_f \sbr*{
      \diag([\nabla f(\w)]_1^2, \ldots, [\nabla f(\w)]_d^2)
    }
  \]
  for all $\w \in \U$, and consequently the $(B,\beta)$-Bernstein
  condition \eqref{eqn:bernstein} implies the coordinate
  $(B',\beta)$-Bernstein condition \eqref{eqn:coordbernstein} with a
  constant $B' = \frac{L_+^2}{C L_-^2} B$ instead of $B$.
\end{theorem}
The condition $\E[\x \x^\top] \succeq C \diag(\E[\x \x^\top])$
expresses that the (uncentered) covariances between features should be weak.\footnote{The highest $C \ge 0$ satisfying the condition is given by the smallest eigenvalue of the (uncentered) correlation matrix, i.e. $\lambda_{\min}\del*{\diag(\E[\x \x^\top])^{-1/2} \E[\x \x^\top] \diag(\E[\x \x^\top])^{-1/2}}$.} In
particular, it is satisfied with $C=1$ if all pairs of features have
covariance zero. The conditions on $h_t$ are satisfied by the logistic
loss $h_t(z) = \log(1+e^{-y_t z})$ when the margins $y_t \w^\top \x_t$
are uniformly bounded. For the hinge loss $h_t(z) = \max\{1-y_t z,0\}$
they are also satisfied if the margins are strictly less than $1$, but
we get $L_- = 0$ for $y_t \w^\top \x_t > 1$.
\begin{proof} \textbf{(Theorem~\ref{thm:generaltocoordbernstein})}
The main inequality is established as follows:
\begin{multline*}
\ex_f \sbr*{
  \nabla f(\w) \nabla f(\w)^\top
}
=
\ex_f \sbr*{
  h_t'(\w^\top \x)^2 \x \x^\top
}
\succeq
\ex_f \sbr*{
  L_-^2 \x \x^\top
}
\succeq 
C L_-^2 \diag(\ex_f \sbr*{
  \x \x^\top
})\\
\succeq 
\frac{C L_-^2}{L_+^2} \diag(\ex_f \sbr*{
  h_t'(\w^\top \x)^2 \x \x^\top
})
=
\frac{C L_-^2}{L_+^2}
\ex_f \sbr*{
\diag([\nabla f(\w)]_1^2, \ldots, [\nabla f(\w)]_d^2)
}.
\end{multline*}
Multiplying both sides of the inequality by $(\w-\u^*)^\top$ on the left
and $(\w-\u^*)$ on the right, we see that the left-hand side of
\eqref{eqn:bernstein} dominates the left-hand side of
\eqref{eqn:coordbernstein} up to a factor of $\frac{C L_-^2}{L_+^2}$,
which is responsible for the difference between $B$ and $B'$.
\end{proof}

\section{Experiments}\label{sec:experiments}

The goal of this experiments section is to quantify the performance of
the proposed MetaGrad variants in comparison with existing algorithms
for Online Convex Optimization. The corresponding Python code is available
from GitHub \citep{metagradjournalcode}. We set things up as follows.

\subsection{Setup}
We describe the data we used, the specific prediction task we considered, and the algorithms we evaluated. We also discuss the choice of domain and hyper-parameters.

\subsubsection{Data} We evaluate on real-world regression and binary
classification data sets from the standard
\href{https://www.csie.ntu.edu.tw/~cjlin/libsvmtools/datasets/}{LIBSVM
library} \citep{chang2011libsvm}. A summary of the data sets can be
found in Table~\ref{tbl:sumdat} in Appendix~\ref{app:expresults}. We
have included data sets of dimension up to 300, so that MetaGrad Full is
tractable. We exclude the \texttt{mushrooms} data set, because it is
linearly separable. This makes the best offline parameters $\u^*$
non-unique for the hinge loss and have infinite norm for the logistic
loss, which is incompatible with our parameter tuning below. The
resulting seventeen data sets have sample sizes $T$ ranging from 252 to
581\,012 and dimensions $d$ ranging from 6 to 300. When available we used
the normalised version of the data set with features in $[-1,1]$.

\subsubsection{Task}
  Each data set is a sequence of labelled examples $(\x_t, y_t)$ for $t
  = 1, \ldots, T$. We define our task to be sequential prediction of the
  labels $y_t$ from the features $\x_t$ using a linear model $\hat y_t =
  \w_t^\top \x_t$. We use a linear model with intercept, which we
  implement by appending a constant $1$ to each feature vector. For the
  classification data with $y_t \in \set{-1,1}$, we consider both the hinge loss $f_t(\w) =
\max\{0,1-y_t \w^\top \x_t\}$ and the logistic loss $f_t(\w) =
\log\del*{1 + e^{-y_t \w^\top \x_t}}$. For the regression data with $y_t
\in \mathbb R$, we consider the absolute loss $f_t(\w) = \abs*{y_t
- \w^\top \x_t}$ and the squared loss $f_t(\w) = \del*{y_t - \w^\top \x_t}^2$.

\subsubsection{Methods}
We compare 9 methods: two popular versions of Online Gradient Descent,
the diagonal version of AdaGrad, and six versions of MetaGrad. We
include the Online Gradient Descent scheme $\w_{t+1} =
\argmin_{\w \in \U} \w^\top \grad_t + \frac{1}{2 \eta_t}
\norm{\w-\w_t}^2 $ with time-decreasing learning rate $\eta_{t} =
\frac{\sigma}{\sqrt{t} \, \max_{s \le t} \norm{\grad_s}_2}$ (abbreviated
as OGDt) and with the gradient-norm-adaptive learning rate $\eta_{t} =
\frac{\sigma}{\sqrt{\sum_{s=1}^t \|\grad_s\|_2^2}}$ (abbreviated as
OGDnorm). In both cases, $\sigma$ is a hyperparameter. Diagonal AdaGrad
can be viewed as OGDnorm applied to each coordinate separately, and
predicts with iterates $\w_{t+1} = \argmin_{\w \in \U} \w^\top \grad_t +
\sum_i \frac{1}{2 \eta_{t,i}} (w_i- w_{t,i})^2$ with separate learning
rates per dimension $\eta_{t,i} = \frac{\sigma}{\sqrt{\sum_{s=1}^t g_{s,
i}^2}}$. Note that for both AdaGrad and Gradient Descent we use the
standard mirror descent version, as opposed to the
FTRL/primal-dual version.
The six versions of MetaGrad are the full version presented in
Section~\ref{sec:metagradfull} (abbreviated as MGFull), the coordinate
version presented in Section~\ref{sec:coordgrad} (abbreviated as MGCo),
and the Frequent Directions sketching version presented in
Section~\ref{sec:sketching} for $m = 2$, $m = \min \{11, d+1\}$, $m =
\min \{26, d+1\}$, and $m = \min \{51, d+1\}$ (abbreviated as MGF$m$).
Note that in each case the number of directions maintained is $m-1$, so
this corresponds to effective dimensions $1, 10, 25$ and $50$, or $d$
when $d$ is smaller.

\subsubsection{Domain}
Each of our algorithms requires a choice of domain $\U$.
While algorithm and domain are independent in principle, in practice
computational convenience is paramount, and only the convenient default
domain choice is prevalent for each algorithm. In this sense one may
think of the choice of algorithm as importing (additional)
regularisation through its associated domain. For the two versions of
Gradient Descent we use the $\ell_2$-norm ball $\U = \{\w: \|\w\|_2 \leq
D_2\}$. For the two diagonal algorithms, AdaGrad and MGCo, we use the
$\ell_\infty$-norm ball $\U = \{\w: \|\w\|_\infty \leq D_\infty\}$. For
the other versions of MetaGrad we use the time-varying domain $\U_t = \{\w: |\w^\top \x_t| \leq C\}$ from
\eqref{eqn:luodomain}, recalling its major benefit that projections on
$\U_t$ can be computed efficiently (see the discussion in
Section~\ref{sec:sketching}). Having fixed the domain shape, we still
have to fix the domain sizes $D_2$, $D_\infty$ and $C$. We note
that this is part of the art of employing machine learning in practice.
To simulate the availability of weak prior knowledge about the
appropriate domain bound, we first compute the unconstrained optimizer of the
cumulative loss $\u^* = \argmin_{\u \in
\mathbb R^d} \sumT f_t(\u)$. We then set the size bounds to fit the comparator up to
the small factor $3$, i.e.\ setting $D_2 = 3 \|\u^*\|_2$, $D_\infty = 3
\|\u^*\|_\infty$ and $C = 3 \max_t |\x_t^\top\u^* |$. We overprovision
the domain by the factor $3 > 1$ to prevent possibly beneficiary effects
that may kick in when the comparator lies on the boundary of the domain
\citep{ftl.boundary}.

\subsubsection{Hyperparameter Tuning}

We now discuss tuning the hyperparameter $\sigma$ that is present in all
methods. To keep the playing field level and convey the same tuning
advantage to all algorithms, we provide the optimal theoretical tuning
of $\sigma$ for all methods, even though $\norm{\u^*}$ is unknown in
practice. Theoretical guidance on this optimal setting comes in two
types, depending on the makeup of the algorithm, and in particular
control of a telescoping term in its regret bound (see e.g.\
\citealt{Zinkevich2003} or \citealt[Corollary~11]{adagrad}). For all
versions of MetaGrad, the optimal setting $\sigma = \norm{\u^*}$ is the
norm of the comparator itself (see
Sections~\ref{sec:metagradfullpractical} and
\ref{sec:metagradcoordpractical}). For algorithms in the mirror descent
family (including OGD and AdaGrad), the optimal theoretical setting is
$\sigma = \max_{w \in \U} \norm{\w - \u^*}/\sqrt{2}$ (see e.g.\
\citealt[Corollary~11]{adagrad}, \citealt[Theorem~2]{OrabonaPal2018}).
For our norm-ball  domain of radius $3 \norm{\u^*}$ this yields $\sigma
= \sqrt{8} \norm{\u^*}$. We also study the effect of tuning $\sigma$
with hindsight in Appendix~\ref{app:hypertune}.

\subsection{Experimental Results}

Table~\ref{tab:regret table} in Appendix~\ref{app:expresults} lists the
regrets of all $9$ algorithms on all $17$ data sets measured in $2$ loss functions each.

\subsubsection{Best Overall Algorithm}

\begin{table}
  \centering
\begin{tabular}{lrrr}
  \hline
Algorithm & \# best & \# better than OGDt & MedianRatio \\
  \hline
AdaGrad & 0 & 0 & 3.54 \\
  OGDnorm & 0 & 4 & 1.41 \\
  OGDt & 1 & 34 & 1.00 \\
  MGCo & 12 & 33 & 0.32 \\
  MGF2 & 2 & 31 & 0.31 \\
  MGF11 & 14 & 31 & 0.27 \\
  MGF26 & 15 & 33 & 0.27 \\
  MGF51 & 17 & 33 & 0.25 \\
  MGFull & 21 & 33 & 0.25 \\
   \hline
\end{tabular}
\caption{Comparison of algorithms with OGDt. The MedianRatio column
contains the median ratio of the regret of each algorithm over that of
OGDt. Columns ``\# best'' and ``\# better than OGDt'' count cases where
the algorithm is at most one regret unit above the best algorithm or OGDt,
respectively.
}\label{tbl:sumresults}
\end{table}

\begin{figure}[htbp]
    \centering
    \includegraphics[width = \textwidth]{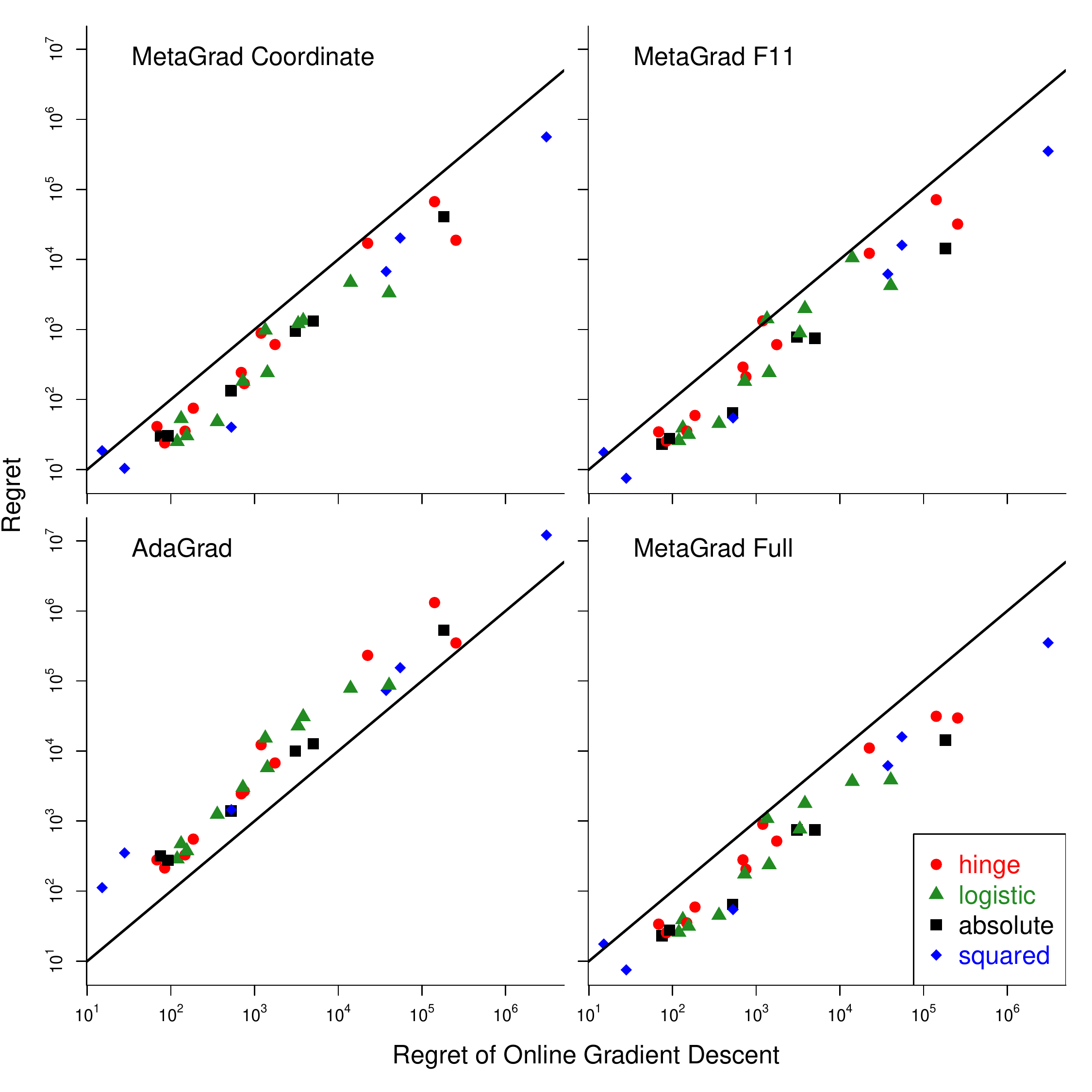}
    \caption{Comparison of the logarithm of the regret of three versions of MetaGrad and AdaGrad and the logarithm of the regret of OGDt. Each marker represents a combination of data set and loss function.}
    \label{fig:MGvsGD}
\end{figure}

We first try to answer the question which algorithm is best. To this
end, we present a summary of the results in Table~\ref{tbl:sumresults}.
The table shows how often each algorithm achieves the lowest loss on our
17 data sets as measured in either of the two relevant loss functions
(for some data sets ties result in several algorithms being counted as best).
It further lists how often each algorithm beats OGDt, and what the median
ratio of regrets is between each algorithm and OGDt. We choose OGDt as the baseline, as it is empirically best
among our competitors AdaGrad, OGDnorm and OGDt. We see that versions
of MetaGrad, especially MetaGrad Full, are often the best algorithm
overall (there is only 1 of 34 cases where no MetaGrad version is the
winner). Moreover, the experiments corroborate the intuition that the
performance of the MetaGrad sketching versions improves as more
dimensions are retained. The table suggests the recommendation to use
sketched MetaGrad in practice, retaining as many dimensions as are computationally affordable.

\subsubsection{Predicting the Winner}
The next question is if we can observe any patterns in the circumstances
for which each algorithm shines. To gain insight here, we compare in
Figure~\ref{fig:MGvsGD} the regrets of MGCo, MGF11, AdaGrad and MGFull
each with OGDt. There does not seem to be
a visually discernible pattern allowing us to predict the advantage over
OGDt (vertical position) from the OGDt regret on the data set (horizontal
position) or the loss function (color). We looked into the blue diamond (i.e.\ square loss) at the far left in each plot, on which OGDt dominates. This data set is \texttt{mg}. One feature that stands out for this set is that its optimal unconstrained coefficient vector assigns a large coefficient to the constant $1$ feature we added to implement the intercept, while using much smaller (three orders of magnitude) coefficients for all other features. This results in an essentially sparse optimal weight vector. It is not clear to us how OGDt is especially able to exploit this. For sure its regret bound does not give a hint.

Comparing MetaGrad Full to MGF11 in Figure~\ref{fig:MGvsGD}, we see that most data set markers are very similarly positioned, except for two (one green, one red) that moved slightly upwards. This is the \texttt{covtype} data set. For this set we indeed see the loss rise dramatically when we sketch to smaller sizes or use the coordinatewise approximation. Apparently, for this data set all features are important, and the intrinsic orientation is not the coordinate basis.
Overall, the coordinatewise version of MetaGrad is close to the performance of the Full version of MetaGrad (the median regret ratio is $1.09$), which suggests that on most data sets the correlations between the features are of little importance.

\subsubsection{Surprises}

To our surprise, AdaGrad has the worst performance of all algorithms.
Upon closer review of the literature we observe that the hyperparameter
($\sigma$) of AdaGrad is often optimized in hindsight based on the data.
In Appendix~\ref{app:hypertune}, 
we tested by how much we could improve the performance of all methods by
tuning $\sigma$ in hindsight. We find there that the benefits are
especially large for AdaGrad. For instance, on
\texttt{w8a} with the logistic loss, we find that we can tune
AdaGrad such that the regret improves from $86921$ to $1147$,
improving upon the theoretically recommended tuning of AdaGrad by an
astounding factor $76$, as well as beating the best theoretically
tuned algorithm
(MGCo) by a factor $2.9$. See Appendix~\ref{app:hypertune} for a
further study of this phenomenon. However, such post-hoc tuning is
not available in sequential decision-making applications. We take this
as a clear motivation to study the effects of tuning, and develop
algorithms that tune themselves.

A second surprise is that OGDt beats OGDnorm (median regret ratio $1.41$).
Worst-case regret bounds indicate that the reverse should occur. There
exist tighter ``luckiness'' analyses for stochastic cases in the
literature \citep{GaillardStoltzVanErven2014,bernfast,JMLR:v20:18-869}, but
these are not sharp enough to explain the difference between OGDnorm and
OGDt. Moreover, these analyses require conditions for which it seems
implausible that a large majority of data sets should satisfy them to the same degree, so these analyses cannot explain why the dominance of OGDt is so consistent.

\paragraph{Conclusion}

Overall, we see that MetaGrad outperforms AdaGrad and Online Gradient Descent consistently across a range of real-world data sets. We conclude that MetaGrad is the best choice for sequential scenarios where safety requirements dictate tuning for theoretical guarantees, as we have studied here.

\subsection{Additional Experiments with Hypertuning}
\label{sec:additionalexperiments}

In Appendix~\ref{app:hypertune} we provide additional experiments to
investigate the best performance that can be achieved in principle by
each of the $9$ algorithms, by tuning the hyperparameter $\sigma$ in
hindsight for each data set. For all methods, $\sigma$ directly
influences the effective learning rate $\eta$, so this hypertuning
provides a loophole to optimize $\eta$ for the data. Although such
post-hoc tuning is impossible in a fully online setting, these
experiments give insight into whether the theoretical tunings are good
advice in practice for non-adversarial data. In general, all methods
gain from hypertuning, but some more than others.
A striking difference with the previous experiments is that AdaGrad,
OGDnorm and MetaGrad all achieve very similar performance when they use
hypertuning. It therefore seems to matter more if the hyperparameters
are optimally tuned to the data set, than which algorithm is chosen.
This suggests that the empirical superiority of MetaGrad in
the primary experiments may be attributed to its ability to better
adapt to the optimal learning rate $\eta$. 

\section{Conclusion and Possible Extensions}
\label{sec:conclusion}

We provide a new adaptive method, MetaGrad, which is robust
to general convex losses but simultaneously can take advantage of
special structure in the losses, like curvature in the loss function or if the data come
from a fixed distribution. The main new technique is to consider
multiple learning rates in parallel: each learning rate $\eta$ has its
own surrogate loss \eqref{eq:surrogate} and there is a single controller
method that aggregates the predictions of $\eta$-experts corresponding
to the different surrogate losses.

An important feature of the controller is that its contribution to the
final regret is only the log of the number of experts, and since the
number of experts is $O(\log T)$ this leads to an additional $O(\log
\log T)$ term that is typically dominated by other terms in the bound.
It is therefore also cheap to add more experts for possibly different
surrogate losses. To make the proof go through, a sufficient requirement
on any such surrogates is that they replace the term $\big(\eta (\u -
\w_t)^\top \grad_t\big)^2$ in \eqref{eq:surrogate} by an upper bound.
This possibility is exploited by \citet{WangLuZhang2020}, who add extra
experts with surrogates that contain $\big(\eta G \|\u -
\w_t\|_2\big)^2$ instead, where $G$ is a known upper bound on
$\|\grad_t\|$.\footnote{To make a Lipschitz-adaptive version of their
approach, we might replace the constant $G$ by the quantity
$\|\grad_t\|$ that it upper bounds.} Since
these surrogates are quadratic in all directions, and not just in the
direction of $\grad_t$, they are better suited for strongly convex
losses, which then leads to an even more adaptive
extension of MetaGrad that also gets the optimal rate $O(\log T)$ for
strongly convex losses, without any dependence on $d$. The price of this
extension is that it doubles the number of experts, which adds a
negligible constant $\log 2$ to the regret, but doubles the run-time of
the algorithm.

If we are willing to increase the number of experts, then another
appealing extension would be to adapt to the $\sigma$ hyperparameter.
This is possible by adding multiple copies of each $\eta$-expert with
different values for $\sigma$. If $\sigma$ ranges over a set of
candidate values $\mathcal{S} := \{\sigma_1,\ldots,\sigma_p\}$, then our
overhead compared to the best choice of $\sigma$ from $\mathcal{S}$ is
an additional small constant $\log p$ in the regret, but our run-time
multiplies by~$p$, so we would still need to keep $p$ relatively small.
For example, we might take an exponentially spaced grid $\mathcal{S} :=
\{2^i \in [\sigma_\textnormal{min},\sigma_\textnormal{max}] \mid i \in
\mathbb Z\}$, so that $p \leq \ceil{\origlog_2
\sigma_\textnormal{max}/\sigma_\textnormal{min}}$.

Another way to extend MetaGrad is to replace the exponential weights
update in the controller by a different experts algorithm.
\citet{zhang19:_dual_adapt} use this to extend MetaGrad for the case
that the optimal parameters $\u$ vary over time, as measured in terms of
the adaptive regret. See also \citet{Neuteboom2020}, who provides a
similar extension of the closely related Squint algorithm for adaptive
regret.

As a final possible extension, we mention the sliding window variant
of Full Matrix AdaGrad \citep{agarwal2018case}. The same sliding window
idea could be used to base the covariance matrix $\Sigma_t^\eta$ in our
Algorithm~\ref{alg:MetaGradExpert} only on the $k$ most recent gradients.
This has both computational advantages, because $\Sigma_t^\eta$ then
becomes a matrix of fixed rank $d+k$, and it could be beneficial for
non-convex optimization when older covariance information needs to be
discarded.

\paragraph{Acknowledgements}

The authors would like to thank Rapha\"el Deswarte and Zakaria Mhammedi
for collaborations on earlier versions of MetaGrad
\citep{deswarte2018linear,lipschitz.metagrad} and the anonymous
reviewers for insightful comments on the manuscript. This research was
supported by the Netherlands Organization for Scientific Research (NWO):
grant numbers VI.Vidi.192.095 for Van Erven and TOP2EW.15.211 for Van
der Hoeven.

\appendix

\section{Extra Material Related to Fast Rate Conditions}
\label{app:MoreFastRateExamplesAndProofs}

In this section we gather extra material related to the fast rate
examples from Sections~\ref{sec:fastRateExamples} and
\ref{sec:metagradcoordanalysis}. We first provide simulations. Then we
present the proofs of Theorems~\ref{thm:curvedfunctions},
\ref{thm:Bernstein} and \ref{thm:coordBernstein}. And finally we give an
example in which the unregularized hinge loss satisfies the Bernstein
condition.

\subsection{Simulations: Logarithmic Regret without Curvature}
\label{app:simulations}

We provide two simple simulation examples to illustrate the sufficient
conditions for Theorems~\ref{thm:curvedfunctions} and
\ref{thm:Bernstein}, and to show that the resulting fast rates are not
automatically obtained by previous methods for general functions. Both
our examples are one-dimensional, and have a stable optimum (that good algorithms will converge
to); yet the functions are based on absolute values, which are neither
strongly convex nor smooth, so the gradient norms do not vanish near the
optimum. As our baseline we include AdaGrad \citep{adagrad}, because it
is commonly used in practice
\citep{WordRepresentations,NeuralNetworkReview} and because, in the
one-dimensional case, it coincides with OGD with an adaptive tuning of the
learning rate that is applicable to general convex functions. See the
description of AdaGrad/OGDnorm in Section~\ref{sec:experiments} for a full
description.

In the first example, we consider offline convex optimization of the
fixed function $f_t(w) \equiv f(w) = \abs{w-\frac{1}{4}}$, which
satisfies the directional derivative condition \eqref{eqn:curvedfunctions} because it is convex. In the
second example, we look at stochastic optimization with convex functions
$f_t(w) = \abs{w - x_t}$, where the outcomes $x_t = \pm \half$ are
chosen i.i.d.\ with probabilities $0.4$ and $0.6$. These probabilities
satisfy \eqref{eqn:bernstein} with $\beta = 1$. Their values are by no
means essential, as long we avoid the worst case where the probabilities
are equal. In both examples, the domain is $\U = [-1,1]$. We tune
AdaGrad with hyperparameter $\sigma = \max_{w,u \in \U} |w-u|/\sqrt{2} =
\sqrt{2}$ and MetaGrad with $\sigma = \max_{u \in \U} |u| = 1$.

Figure~\ref{fig:killer} graphs the results. We see that in both cases
the regret of AdaGrad follows its $O(\sqrt{T})$ bound, while MetaGrad
achieves an $O(\ln T)$ rate, as predicted by
Theorems~\ref{thm:curvedfunctions} and~\ref{thm:Bernstein}. This shows
that MetaGrad achieves a type of adaptivity that is not achieved by
AdaGrad.

\begin{figure}[t]
\centering
\subfigure[{Offline: $f_t(u) = \abs{u-1/4}$} ]{
\includegraphics[width=.45\textwidth]{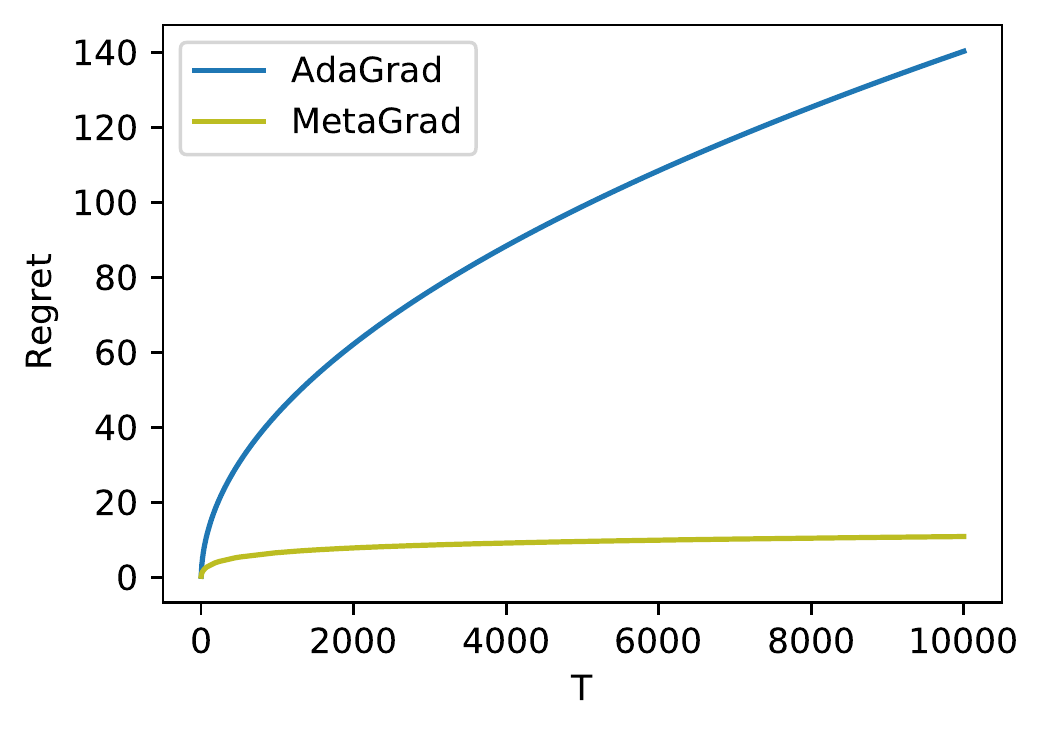}
}
\subfigure[{Stochastic Online: $f_t(u) = \abs{u- x_t}$ where $x_t = \pm
\half$ i.i.d.\ with probabilities $0.4$ and $0.6$.}]{
\includegraphics[width=.45\textwidth]{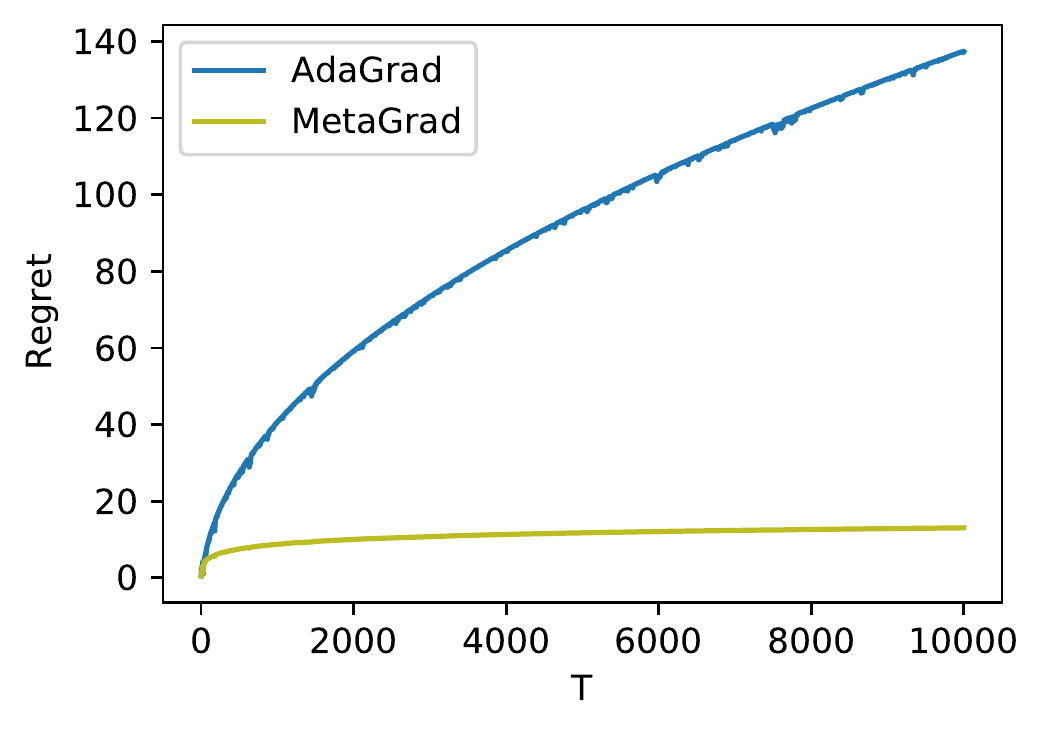}
}
\caption{Examples of fast rates on functions without curvature. MetaGrad
incurs logarithmic regret $O(\log T)$, while AdaGrad incurs
$O(\sqrt{T})$ regret, matching its bound.}
\label{fig:killer}
\end{figure}

\subsection{Proof of Theorem~\ref{thm:curvedfunctions}}

\begin{proof}
By \eqref{eqn:curvedfunctions}, applied with $\w = \w_t$, and
\eqref{eqn:roughmainbound}, there exists a $C > 0$ (depending on~$a$)
such that, for all sufficiently large $T$,
\begin{align*}
  R_T^\u
    &\leq a \Rtrick_T^\u - b V_T^\u
    \leq C\sqrt{ V_T^\u\, d \ln T } + C d \ln T - b V_T^\u\\
    &\leq \frac{\gamma}{2} C V_T^\u + \del*{\frac{1}{2\gamma}+1} C d \ln
    T - b V_T^\u
    \qquad \text{for all $\gamma > 0$,}
\end{align*}
where the last inequality is based on $\sqrt{xy} = \min_{\gamma > 0}
\frac{\gamma}{2} x + \frac{y}{2\gamma}$ for all $x,y > 0$. The result
follows upon taking $\gamma = \frac{2b}{C}$.
\end{proof}

\subsection{Proofs of Theorems~\ref{thm:Bernstein} and
\ref{thm:coordBernstein}}

\begin{proof} \textbf{(Theorem~\ref{thm:Bernstein})}
By \eqref{eqn:roughmainbound} there exists a constant $C > 0$ such that,
for all sufficiently large~$T$,
\[
  \E\sbr*{\Rtrick_T^{\u^*}}
    \leq C \E\sbr*{\sqrt{ V_T^{\u^*}\, d \ln T }} + C d \ln T.
\]
Abbreviating $\rtrick_t^\u = (\w_t - \u)^\top \grad_t$, we see that
$\Rtrick_T^{\u^*} = \sum_{t=1}^T \rtrick_t^{\u^*}$, $V_T^{\u^*} =
\sum_{t=1}^T (\rtrick_t^{\u^*})^2$ and the
Bernstein condition with $\w = \w_t$ becomes
\[
  \E[(\rtrick_t^{\u^*})^2 \mid \w_t]
    \leq B \E[\rtrick_t^{\u^*} \mid \w_t]^\beta.
\]
Combining the above with multiple applications of Jensen's inequality,
the expected linearized regret is at most
\begin{align}
  \E\sbr*{\Rtrick_T^{\u^*}}
&\leq C \sqrt{ \E\sbr*{V_T^{\u^*}}\, d \ln T} + C d \ln T \notag\\ 
&\leq C \sqrt{ B \sum_{t=1}^T
\E_{\w_t}\sbr*{\del*{\E\sbr*{\rtrick_t^{\u^*}|\w_t}}^\beta} \,
d \ln T} + C d \ln T  \notag\\
&\leq C \sqrt{ B \sum_{t=1}^T \del*{\E\sbr*{\rtrick_t^{\u^*}}}^\beta \, d \ln T} + C d \ln T.
\label{eqn:bernsteinstart}
\end{align}
In the following, we will repeatedly use the fact that
\begin{equation}\label{eqn:linearize}
  x^\alpha y^{1-\alpha}
   = c_\alpha \inf_{\gamma > 0} \del*{\frac{x}{\gamma} +
   \gamma^\frac{\alpha}{1-\alpha}y}
  \qquad \text{for any $x,y \geq 0$ and $\alpha \in (0,1)$,}
\end{equation}
where $c_\alpha = (1-\alpha)^{1-\alpha} \alpha^\alpha$. Applying this
first with $\alpha = 1/2$, $x = B d \ln T$ and $y = \sum_{t=1}^T
\del*{\E\sbr{\rtrick_t^{\u^*}}}^\beta$, we obtain
\begin{equation*}
  \sqrt{ B \sum_{t=1}^T \del*{\E\sbr{\rtrick_t^{\u^*}}}^\beta \, d \ln T}
  \leq 
  c_{1/2}\gamma_1 \sum_{t=1}^T \del*{\E\sbr{\rtrick_t^{\u^*}}}^\beta +
  \frac{c_{1/2}}{\gamma_1} B d \ln T
  \qquad \text{for any $\gamma_1 > 0$.}
\end{equation*}
If $\beta = 1$, then $\sum_{t=1}^T \del*{\E\sbr{\rtrick_t^{\u^*}}}^\beta
= \E\sbr{\Rtrick_T^{\u^*}}$ and the result follows by taking $\gamma_1 =
\frac{1}{2C c_{1/2}}$. Alternatively, if $\beta < 1$, then we apply
\eqref{eqn:linearize} a second time, with $\alpha = \beta$, $x =
\E\sbr{\rtrick_t^{\u^*}}$ and $y=1$, to find that, for any $\gamma_2>0$,
\begin{align*}
  \sqrt{ B \sum_{t=1}^T \del*{\E\sbr{\rtrick_t^{\u^*}}}^\beta \, d \ln T}
  &\leq 
  c_\beta c_{1/2} \gamma_1 \sum_{t=1}^T
  \del*{\frac{\E\sbr{\rtrick_t^{\u^*}}}{\gamma_2} +
  \gamma_2^{\beta/(1-\beta)}} +
  \frac{c_{1/2}}{\gamma_1} B d \ln T\\
  &= 
  \frac{c_\beta c_{1/2} \gamma_1}{\gamma_2} \E\sbr{\Rtrick_T^{\u^*}}
  + c_\beta c_{1/2} \gamma_1 \gamma_2^{\beta/(1-\beta)} T
  + \frac{c_{1/2}}{\gamma_1} B d \ln T.
\end{align*}
Taking $\gamma_1 = \frac{\gamma_2}{2 c_\beta c_{1/2} C}$, this yields
\begin{equation*}
  \E\sbr{\Rtrick_T^{\u^*}}
    \leq
  \gamma_2^{1/(1-\beta)} T
  + \frac{4 C^2 c_{1/2}^2 c_\beta B d \ln T}{\gamma_2} + 2 Cd\ln T.
\end{equation*}
We may optimize over $\gamma_2$ by a third application of
\eqref{eqn:linearize}, now with $x = 4 C^2 c_{1/2}^2 c_\beta B d \ln T$,
$y = T$ and $\alpha = 1/(2-\beta)$, such that $\alpha/(1-\alpha) =
1/(1-\beta)$:
\begin{align*}
  \E\sbr{\Rtrick_T^{\u^*}}
    &\leq \frac{1}{c_{1/(2-\beta)}}
    \del*{4 C^2 c_{1/2}^2 c_\beta B d \ln T}^{1/(2-\beta)}
    T^{(1-\beta)/(2-\beta)}
    + 2 Cd\ln T\\
    &= O\del*{\del*{B d \ln T}^{1/(2-\beta)} T^{(1-\beta)/(2-\beta)}
    + d\ln T},
\end{align*}
which completes the proof.
\end{proof}

\begin{proof} \textbf{(Theorem~\ref{thm:coordBernstein})}
  We will show that \eqref{eqn:bernsteinstart} from the proof of
  Theorem~\ref{thm:Bernstein} also holds under the conditions of
  Theorem~\ref{thm:coordBernstein}. The rest of the proof then proceeds
  in the same way. To this end, we use that \eqref{eqn:roughcoordbound}
  implies the existence of a constant $C>0$ such that, for all
  sufficiently large $T$,
  \[
    \Rtrick_T^{\u^*}
      \leq  C \sum_{i=1}^d \sqrt{V_{T,i}^{u_i}\, \ln(T)}
            + C d \ln(T).
  \]
  Multiple applications of Jensen's inequality, together with the
  coordinate Bernstein condition, then imply that
  \begin{align*}
    \E\sbr*{\Rtrick_T^{\u^*}}
      &\leq C \E\sbr*{\sum_{i=1}^d \sqrt{V_{T,i}^{u^*_i}\,\ln(T)}}
            + C d \ln(T)
      = C d \E\sbr*{\sum_{i=1}^d \frac{1}{d} \sqrt{V_{T,i}^{u^*_i}\,\ln(T)}}
            + C d \ln(T)\\
      &\leq C d \sqrt{\E\sbr*{\sum_{i=1}^d \frac{1}{d}
      \del*{V_{T,i}^{u^*_i}\,\ln(T)}}}
            + C d \ln(T)\\
      &= C \sqrt{\sum_{t=1}^T \E_{\w_t}\sbr*{\sum_{i=1}^d (w_{t,i} - u^*_i)^2
      \E[g_{t,i}^2\mid \w_t]}\,d \ln(T)}
            + C d \ln(T)\\
      &\leq C \sqrt{B \sum_{t=1}^T \E_{\w_t}\sbr*{\del*{\E[\rtrick_t^{\u^*}\mid
      \w_t]}^\beta}\,d \ln(T)}
            + C d \ln(T)\\
      &\leq C \sqrt{B \sum_{t=1}^T \del*{\E[\rtrick_t^{\u^*}]}^\beta\,d \ln(T)}
            + C d \ln(T).
  \end{align*}
  This establishes the same inequality as in \eqref{eqn:bernsteinstart},
  and the remainder of the proof is the same as for
  Theorem~\ref{thm:Bernstein}.
\end{proof}

\subsection{Unregularized Hinge Loss Example}
\label{app:hingeLossExample}

As shown by \citet{bernfast}, the Bernstein
condition is satisfied in the following classification task:
\begin{lemma}[Unregularized Hinge Loss
Example]\label{lem:hingeLossExample}
  Suppose that $(\X_1,Y_1),(\X_2,Y_2),\ldots$ are i.i.d.\ with $Y_t$
  taking values in $\{-1,+1\}$, and let $f_t(\u) = \max\{0,1 - Y_t
  \u^\top \X_t\}$ be the \emph{hinge loss}. Assume that both $\U$ and
  the domain for $\X_t$ are the $d$-dimensional unit ball. Then the
  $(B,\beta)$-Bernstein condition is satisfied with $\beta = 1$ and $B =
  \frac{2\lambdamax}{\|\vmu\|_2}$, where $\lambdamax$ is the maximum
  eigenvalue of $\E\sbr*{\X \X^\top}$ and $\vmu = \E\sbr{Y\X}$, provided
  that $\|\vmu\|_2 > 0$.

  In particular, if $\X_t$ is uniformly distributed on the sphere and
  $Y_t = \sign(\ip{\bar{\u}}{\X_t})$ is the noiseless classification of
  $\X_t$ according to the hyperplane with normal vector $\bar{\u}$, then $B
  \leq \frac{c}{\sqrt{d}}$ for some absolute constant $c > 0$.
\end{lemma}
Thus the version of the Bernstein condition that implies an $O(d \log T)$
rate is always satisfied for the hinge loss on the unit ball, except
when $\|\vmu\|_2 = 0$, which is very natural to exclude, because it
implies that the expected hinge loss is $1$ (its maximal value) for all
$\u$, so there is nothing to learn. It is common to add
$\ell_2$-regularization to the hinge loss to make it strongly convex,
but this example shows that that is not necessary to get logarithmic
regret.

For completeness, we repeat the proof of
Lemma~\ref{lem:hingeLossExample} from \citet{bernfast}:
\begin{proof} \textbf{(Lemma~\ref{lem:hingeLossExample})}
  Since, by assumption, $\u$ and $\X$ have length at most $1$, the hinge
  loss simplifies to $f(\u) = 1 - Y \u^\top \X$ with gradient $\nabla
  f(\u) = -Y \X$. This implies that
  \begin{equation}
    \u^* \df \argmin_\u \E\sbr*{f(\u)} = \frac{\vmu}{\|\vmu\|},
  \end{equation}
  and
  \begin{align*}
    (\w - \u^*)^\top \ex &\sbr*{\nabla f(\w) \nabla f(\w)^\top}(\w - \u^*)
      = (\w - \u^*)^\top \ex \sbr*{\X \X^\top}(\w - \u^*)\\
      &\leq \lambdamax (\w - \u^*)^\top (\w - \u^*)
      \leq 2\lambdamax (1 - \w^\top \u^*)\\
      &= \frac{2\lambdamax}{\|\vmu\|} (\w -\u^*)^\top (-\vmu)
      = \frac{2\lambdamax}{\|\vmu\|} (\w -\u^*)^\top \E\sbr*{\nabla
      f(\w)},
  \end{align*}
  which proves the first part of the lemma

  For the second part, we first observe that $\lambdamax = 1/d$. Then,
  to compute $\|\vmu\|$, assume without loss of generality that
  $\|\bar{\u}\| = 1$, in which case $\bar{\u} = \u^*$. Now symmetry of
  the distribution of $\X$ conditional on $\X^\top \u^*$ gives
  \begin{equation*}
    \E\sbr*{Y \X \mid \X^\top \u^*}
      = \sign(\X^\top \u^*) \E\sbr*{\X \mid \X^\top \u^*}
      = \sign(\X^\top \u^*) \X^\top \u^* \u^*
      = |\X^\top \u^*| \u^*.
  \end{equation*}
  By rotational symmetry, we may further assume without loss of
  generality that $\u^* = \e_1$ is the first unit vector in the standard
  basis, and therefore
  \begin{equation*}
   \|\vmu\| = \|\E\sbr*{|\X^\top \u^*|} \u^*\| =  \E\sbr*{|X_1|}.
  \end{equation*}
  If $\Z = (Z_1,\ldots,Z_d)$ is multivariate Gaussian $\normal(0,I)$.
  Then $\X = \Z/\|\Z\|$ is uniformly distributed on the sphere, so
  \begin{align*}
    \E\sbr{|X_1|}
      = \E\sbr*{\frac{|Z_1|}{\|\Z\|}}
      \geq \frac{1}{4 \sqrt{d}} \Pr\del*{|Z_1| \geq \half \wedge \|\Z\|
      \leq 2\sqrt{d}}.
  \end{align*}
  Since $\Pr\del*{|Z_1| < \half} \leq 0.4$ and $\Pr\del*{\|\Z\| \geq
  2\sqrt{d}} \leq \frac{1}{4d} \E\sbr*{\|\Z\|^2} = \frac{1}{4}$, we
  have
  \begin{equation*}
    \Pr\del*{|Z_1| \geq \half \wedge \|\Z\|
        \leq 2\sqrt{d}}
    \geq 1 - 0.4 - \frac{1}{4} = 0.35,
  \end{equation*}
  from which the conclusion of the second part follows with $c =
  8/0.35$.
\end{proof}

\subsection{Bernstein for Linearized Excess Loss}\label{sec:bnst}
Let $f : \U \to \reals$ be a convex function drawn from distribution $\pr$ with stochastic optimum $\u^* = \argmin_{\u \in \U} \E_{f \sim
\pr}[f(\u)]$. For any $\w \in \U$, we now show that the Bernstein
condition for the excess loss $X \df f(\w)-f(\u^*)$ implies the
Bernstein condition with the same exponent $\beta$ for the linearized
excess loss $Y \df (\w-\u^*)^\top \nabla f(\w)$. These variables satisfy
$Y \ge X$ by convexity of $f$ and $Y \le C \df 2 D_2 G_2$. 

\begin{lemma}
For $\beta \in (0,1]$, let $X$ be a $(B,\beta)$-Bernstein random variable:
\[
\ex[X^2] \le B \ex[X]^\beta.
\]
Then any bounded random variable $Y \le C$ with $Y \ge X$ pointwise satisfies the $(B',\beta)$-Bernstein condition
\[
\ex[Y^2] \le B' \ex[Y]^\beta
\]
for $B' = \max \set*{
B,
\frac{2}{\beta}
C^{2-\beta}
}
$.
\end{lemma}

\begin{proof}
For $\beta \in (0,1)$ we will use the fact that
\[
  z^\beta
  ~=~ 
  c_\beta \inf_{\gamma > 0} \del*{\frac{z}{\gamma} +
   \gamma^\frac{\beta}{1-\beta}}
  \qquad \text{for any $z \geq 0$,}
\]
with $c_\beta =   (1-\beta)^{1-\beta} \beta^\beta$.
For $\gamma = \del*{\frac{1-\beta}{\beta} \ex[Y]}^{1-\beta}$ we therefore have
\begin{align}
\notag
\ex[X^2] - B' \ex[X]^\beta
&~\ge~
\ex[X^2] - B' c_\beta \del*{\frac{\ex [X]}{\gamma}  + \gamma^{\frac{\beta}{1-\beta}}}
~\ge~
\ex[Y^2] - B' c_\beta \del*{\frac{\ex [Y]}{\gamma}  + \gamma^{\frac{\beta}{1-\beta}}}
\\
\label{eq:tolimitme}
&~=~
\ex[Y^2] - B' \ex[Y]^\beta
,
\end{align}
where the second inequality holds because $x^2 - c_\beta B' x/\gamma$ is a decreasing function of $x \le C$ for $\gamma  \le \frac{c_\beta B'}{2 C}$, which is satisfied by the choice of $B'$.
This proves the lemma for $\beta \in (0,1)$. The claim for $\beta=1$ follows by taking the limit $\beta \to 1$ in \eqref{eq:tolimitme}.
\end{proof}

\section{Controller Regret Bound (Proof of Lemma~\ref{lem:controllerbd})}\label{sec:controllerbdpf}
We prove the lemma in two parts.

\subsection{Decomposing the Surrogate Regret}
Fix a comparator point $\u \in \bigcap_{t=1}^T \U_t$. We will first bound the surrogate regret
\[
  R_T^\eta(\u)
  ~\df~
  \sum_{t=1}^T \del*{
    \surr_t^\eta(\w_t) - \surr_t^\eta(\u)
  }
\]
for any $\eta \in \mathcal G$ not expired after $T$ rounds (see Definition~\ref{def:expired}). Note that by definition \eqref{eq:surrogate}, the surrogate loss $\surr_t^\eta(\w_t)$ of the controller is always zero, but we believe writing it helps interpretation. We will then use this surrogate regret bound to control the (non-surrogate) regret.

\bigskip\noindent
For the first half of this section, we fix a final time $T$, and a
grid-point $\eta \in \mathcal G$ that is still not expired after time
$T$ (see Definition~\ref{def:expired}). We redefine $a^\eta$ from
\eqref{eqn:wakeuptime} as follows:
\begin{definition}\label{def:tau.eta}
  We define the wakeup time of learning rate $\eta \in \mathcal G$ by
  \[
    a^\eta ~\df~ \inf \setc*{t \le T}{
      \eta > \frac{1}{2 \del*{\sum_{s=1}^{t-1} b_s \frac{B_{s-1}}{B_s} + B_{t-1}}}
    } \glb (T+1)
    .
  \]
\end{definition}
The difference with \eqref{eqn:wakeuptime} is that we now manually set to $T+1$ the
wakeup time of an $\eta$ that does not wake up during the first $T$
rounds. We do this so that $[1,a^\eta-1]$ and $[a^\eta, T]$ always
partition rounds $[1,T]$.

Our strategy will be to split the regret in three parts, which we will analyse separately.
\begin{proposition}\label{prop:threeway}
  We have
  \[
    R_T^\eta(\u)
    ~=~
    \underbrace{
      \sum_{t=1}^{a^\eta-1} \del*{
        \surr_t^\eta(\w_t) - \surr_t^\eta(\u)
      }
    }_\text{$\surr^\eta$-regret of controller w.r.t.\ $\u$}
    +
    \underbrace{
      \sum_{t = a^\eta}^T \del*{
        \surr_t^\eta(\w_t) - \surr_t^\eta(\w_t^\eta)
      }
    }_\text{$\surr^\eta$-regret of controller w.r.t.\ $\eta$-expert}
    +
    \underbrace{
      \sum_{t = a^\eta}^T \del*{
        \surr_t^\eta(\w_t^\eta) - \surr_t^\eta(\u)
      }
    }_\text{$\surr^\eta$-regret of $\eta$-expert w.r.t.\ $\u$}.
  \]
\end{proposition}
\begin{proof}
  The choice of $a^\eta$ makes all $\w_t^\eta$ defined. We can hence merge the sums.
\end{proof}
We think of the three sums as follows. The first sum is ``startup nuisance'', and it will turn out to be tiny. The second sum is controlled by the controller, and it only depends on its construction. The third sum is controlled by the $\eta$-experts, and it only depends on their construction.

We will now proceed to bound the three parts above. First, we reduce to
the clipped surrogate losses \eqref{eq:clipsurr} at almost negligible cumulative cost using the clipping technique of \cite{cutkosky2019artificial}.
\begin{lemma}[Clipping in the controller is
cheap]\label{lem:controller.cheap.clipping}
  \begin{align*}
    &
    \underbrace{
      \sum_{t=1}^{a^\eta-1} \del*{
        \surr_t^\eta(\w_t) - \surr_t^\eta(\u)
      }
    }_\text{$\surr^\eta$-regret of controller w.r.t.\ $\u$}
    +
    \underbrace{
      \sum_{t = a^\eta}^T \del*{
        \surr_t^\eta(\w_t) - \surr_t^\eta(\w_t^\eta)
      }
      }_\text{$\surr^\eta$-regret of controller w.r.t.\ $\eta$-expert}
    \\
    &~\le~
    \underbrace{
      \sum_{t=1}^{a^\eta-1} \del*{
        \clipsurr_t^\eta(\w_t) - \clipsurr_t^\eta(\u)
      }
    }_\text{$\clipsurr^\eta$-regret of controller w.r.t.\ $\u$}
    +
    \underbrace{
      \sum_{t = a^\eta}^T \del*{
        \clipsurr_t^\eta(\w_t) - \clipsurr_t^\eta(\w_t^\eta)
      }
    }_\text{$\clipsurr^\eta$-regret of controller w.r.t.\ $\eta$-expert}
    + \eta B_T
  \end{align*}
\end{lemma}
\begin{proof}
  For any $\u \in \U_t$ (which includes the case $\u = \w_t^\eta$), we may use the definition of the range bound \eqref{eqn:rangebound}, the surrogate loss \eqref{eq:surrogate} and its clipped version \eqref{eq:clipsurr} to find
  \begin{align*}
    &
    \del*{
    \surr_t^\eta(\w_t) - \surr_t^\eta(\u)
    }
    - \del*{
    \clipsurr_t^\eta(\w_t) - \clipsurr_t^\eta(\u)
      }
    \\
    &~=~
      \eta \frac{B_t - B_{t-1}}{B_t} (\w_t - \u)^\top \grad_t
      -
      \underbrace{
      \eta^2
      \frac{B_t^2 - B_{t-1}^2}{B_t^2}
      \big((\u-\w_t)^\top \grad_t\big)^2
      }_{\ge 0}
    \\
    &~\le~
      \eta \frac{B_t - B_{t-1}}{B_t} b_t
      ~\le~
      \eta \del*{B_t - B_{t-1}}
      .
  \end{align*}
Summing over rounds completes the proof.
\end{proof}
Next we deal with the clipped surrogate regret. We first handle the case of the early rounds before $a^\eta$. The key idea is that when $\eta$ has not yet woken up, it is very small. Since the surrogate loss scales with $\eta$, it is small as well, even in sum.

\begin{lemma}\label{lem:early.rounds}
  For any $\eta$ and any $\u \in \bigcap_{s=1}^{a^\eta-1} \mathcal W_s$
  \[
    \underbrace{
      \sum_{t=1}^{a^\eta-1} \del*{
        \clipsurr_t^\eta(\w_t) - \clipsurr_t^\eta(\u)
      }
    }_\text{$\clipsurr^\eta$-regret of controller w.r.t.\ $\u$}
    ~\le~
    \frac{1}{2}
    .
  \]
\end{lemma}
\begin{proof}
  By definition of the clipped surrogate loss $\clipsurr_t^\eta$ in \eqref{eq:clipsurr}, the range bound $b_t$ in \eqref{eqn:rangebound} and the wakeup time $a_t$ in Definition~\ref{def:tau.eta},
  \[
    \sum_{t=1}^{a^\eta-1} \clipsurr_t^\eta(\w_t) - \clipsurr_t^\eta(\u)
    ~\le~
    \sum_{t=1}^{a^\eta-1} \eta (\w_t - \u)^\top \clipgrad_t
    ~\le~
    \sum_{t : \eta \le
    \frac{1}{2 \del*{\sum_{s=1}^{t-1} b_s \frac{B_{s-1}}{B_s} + B_{t-1}}}} \eta b_t \frac{B_{t-1}}{B_t}
  ~\le~
  \frac{1}{2}
  .
\]
\end{proof}

In the next subsection we deal with the middle sum in Proposition~\ref{prop:threeway}. This part only depends on the construction of the controller. We deal with the final sum in the section after that.

\subsection{Controller surrogate regret bound}

The controller is a specialists algorithm, which sometimes resets. We call the time segments between resets epochs. In every epoch, the controller guarantees a certain specialists regret bound w.r.t.\ any $\eta$-expert in its grid.

The $\eta$-expert that we need can be active during several epochs. Our strategy, following \citet{lipschitz.metagrad}, will be the following. We incur the controller regret in the last and one-before-last epochs. We further separately prove, using the reset condition, that the total regret in all earlier epochs is tiny.
\begin{lemma}\label{lem:controller.epoch.bd}
  Consider an epoch starting at time $\tau+1$ and fix any later time $t$ in that same epoch. Fix any grid point $\eta \in \mathcal G$ not expired after $t$ rounds (meaning $\eta \le \frac{1}{2 B_{t-1}}$). Then the MetaGrad controller guarantees
  \[
    \underbrace{
      \sum_{s \in (\tau, t] :s \ge a^\eta}  \del*{
        \clipsurr_t^\eta(\w_t) - \clipsurr_t^\eta(\w_t^\eta)
      }
    }_\text{specialist $\clipsurr^\eta$-regret of controller w.r.t.\ $\eta$-expert on $(\tau, t]$}
    ~\le~
    \ln \ceil*{2 \origlog_2 \del*{\sum_{s=1}^{t-1}
        \frac{b_s }{B_s} + 1}}_+.
  \]
\end{lemma}

Note that it is not important what the $\eta$-experts do at this point,
the only feature that we use in the proof is that $\w_t^\eta \in \U_t$
for each active $\eta$. Also, note that the right-hand side is $O(\ln
\ln T)$. 

\begin{proof}
  We first observe that Algorithm~\ref{alg:controller}, as far as it maintains the weights $p_t(\eta)$ between resets, implements Specialists Exponential Weights \citep[called SBayes by][]{fssw-se-97}. In our particular case it is applied to specialists $\eta \in \mathcal G$, loss function $\eta \mapsto \surr_t^\eta(\w_t^\eta)$, active set $\activeset_t \subseteq \mathcal G$ and uniform (improper) prior on $\mathcal G$. The specialists regret bound \citep[Theorem~1]{fssw-se-97} directly yields\footnote{Our improper prior does not cause any trouble here, because renormalizing the prior, in hindsight, to the finite set of $\eta$-experts that were ever active preserves the algorithm's output and hence its regret bound.}
\[
  \sum_{s \in (\tau, t] :s \ge a^\eta} - \ln \ex_{p_t(\eta)} \sbr*{
    e^{- \clipsurr_t^\eta(\w_t^\eta)}
  }
  ~\le~
  \ln \card*{
    \bigcup_{s \in (\tau, t]} \mathcal A_s
  }
  +
  \sum_{s \in (\tau, t] :s \ge a^\eta}  \clipsurr_t^\eta(\w_t^\eta)
  .
\]
Algorithm~\ref{alg:controller} further chooses the controller iterate
\[
  \w_t ~=~ \frac{\ex_{p_t(\eta)} \sbr*{\eta \w_t^\eta}}{\ex_{p_t(\eta)} \sbr*{\eta }}
\]
which we claim ensures that
\[
  0 ~\le~
  - \ln \ex_{p_t(\eta)} \sbr*{
    e^{-\clipsurr_t^\eta(\w_t^\eta)}
  }
  .
\]
To see why, we use the definition \eqref{eq:clipsurr} of clipped loss and gradient to obtain $(\w_t-\w_t^\eta)^\top \clipgrad_t \ge -B_{t-1}$, and we further use that $p_t$ is supported on $\activeset_t$, which implies that $\eta \le \frac{1}{2 B_{t-1}}$. Together these license\footnote{Here we motivate our controller algorithm using the loss function $\eta \mapsto \clipsurr_t^\eta(\w_t^\eta)$. One can alternatively base it on the loss function $\eta \mapsto - \ln \del*{1 + \eta (\w_t - \w_t^\eta)^\top \clipgrad_t}$ \citep[These two versions are called Squint and iProd respectively by][]{squint}. As the second is always smaller (by the prod bound), using it would give a strictly tighter theorem here. We do not see a way to ultimately harvest this gain, as we would still need to invoke the prod bound at a later point in the analysis to express our regret bound in second-order form. We chose to present the ``Squint-style'' version here as we believe it is the more intuitive of the two.}
the ``prod bound'' ($e^{x-x^2} \le 1+x$ for $x \ge -\frac{1}{2}$) yielding
\[
- \ln \ex_{p_t(\eta)} \sbr*{
    e^{-\clipsurr_t^\eta(\w_t^\eta)}
  }
  ~\ge~
  - \ln \ex_{p_t(\eta)} \sbr*{
    1 + \eta (\w_t-\w_t^\eta)^\top \clipgrad_t
  }
  ~=~
  0.
\]
Inserting $\surr_t^\eta(\w_t) = 0$, this implies
\[
  \sum_{s \in (\tau, t] :s \ge a^\eta}  \del*{
    \clipsurr_t^\eta(\w_t)
    - \clipsurr_t^\eta(\w_t^\eta)
  }
  ~\le~
  \ln \card*{
    \bigcup_{s \in (\tau, t]} \mathcal A_s
  }
  .
\]
It remains to bound the maximum number of active grid-points during any
epoch. Recall from \eqref{eqn:activeexperts} that the active set at any time $t$ is
\[
  \activeset_t
  ~=~
  \intoc*{
    \frac{1}{2 \del*{\sum_{s=1}^{t-1} b_s \frac{B_{s-1}}{B_s} + B_{t-1}}},
    \frac{1}{2 B_{t-1}}
  }
  \cap \mathcal G
  .
\]
Both endpoints are decreasing with $t$. Since our epoch starts at time $\tau +1$, the maximal $\eta$ active in the epoch is
\[
  \max \setc[\Big]{\eta \in \mathcal G }{ \eta \leq \frac{1}{2 B_{\tau}}}.
\]
As we consider the part of the epoch up to time $t \geq \tau +1$, the
smallest $\eta$ active in the epoch is
\[
  \min \setc[\bigg]{\eta \in \mathcal G }{\eta \geq \frac{1}{2 \del*{\sum_{s=1}^{t-1} b_s
        \frac{B_{s-1}}{B_s} + B_{t-1}}}
  }
  .
\]
And since $\mathcal G$ is exponentially spaced with base $2$, the maximum number of $\eta$ that could possibly have been active is
\begin{align*}
  \ceil*{\origlog_2 \frac{\del*{\sum_{s=1}^{t-1} b_s \frac{B_{s-1}}{B_s} +
  B_{t-1}}}
              {B_{\tau}}}
  &\leq
  \ceil*{\origlog_2 \frac{B_{t-1}\del*{\sum_{s=1}^{t-1}
  \frac{b_s }{B_s} + 1}}
              {B_{\tau}}}\\
  &\leq
  \ceil*{\origlog_2 \del*{\del*{\sum_{s=1}^{t-1} \frac{b_s}{B_s}}\del*{\sum_{s=1}^{t-1}
  \frac{b_s }{B_s} + 1}}}\\
  &\leq
  \ceil*{2 \origlog_2 \del*{\sum_{s=1}^{t-1}
  \frac{b_s }{B_s} + 1}}_+,
\end{align*}
where the second inequality holds because of the reset condition
\eqref{eqn:resetcondition}. All together, we conclude that our prior
costs for the improper (uniform on $\mathcal G$) prior are upper bounded
by
\begin{equation}\label{eqn:priorcosts}
  \ln \card*{
    \bigcup_{s \in (\tau, t]} \mathcal A_s
  }
  ~\leq~
  \ln \ceil*{2 \origlog_2 \del*{\sum_{s=1}^{t-1}
  \frac{b_s }{B_s} + 1}}_+.
\end{equation}

\end{proof}

We now have a specialists regret bound that we can apply to each epoch.

\begin{lemma}[Total regret in far past is tiny]
  \label{lem:tiny.past}
  Consider two consecutive epochs, starting after $\tau_1 < \tau_2$, and let $\eta$ be not expired after $\tau_1$ rounds. Then
  \[
    \sum_{s \in [1,\tau_1], s \ge a^\eta}
    \del*{
      \clipsurr_s^\eta(\w_s)
      -\clipsurr_s^\eta(\w_s^\eta)
    }
    ~\le~
    \eta B_{\tau_2}.
  \]
\end{lemma}
\begin{proof}
\[
  -\sum_{s \in [1,\tau_1], s \ge a^\eta}
  \clipsurr_s^\eta(\w_s^\eta)
  ~\le~
  \eta
  \sum_{s=1}^{\tau_1} b_s \frac{B_{s-1}}{B_s}
  ~\le~
  \eta
  B_{\tau_1}
  \sum_{s=1}^{\tau_1} \frac{b_s}{B_s}
  ~\le~
  \eta
  B_{\tau_1}
  \sum_{s=1}^{\tau_2} \frac{b_s}{B_s}
  ~\le~
  \eta
  B_{\tau_2},
\]
where the last inequality is the reset condition \eqref{eqn:resetcondition} at time $\tau_2$.
\end{proof}
We are now ready to compose the previous two lemmas to obtain
the following result:
\begin{lemma}[Overall controller specialists regret
bound]\label{lem:overall.controller.regret}
  Let $\eta$ be not expired after $T$ rounds. Then
\begin{equation}\label{eqn:controllersleepingregret}
  \sum_{t= a^\eta}^T
  \del*{
    \clipsurr_t^\eta(\w_t)
    - \clipsurr_t^\eta(\w_t^\eta)
  }
  ~\le~
  \eta B_{T}
  +
  2\ln \ceil*{2 \origlog_2 \del*{\sum_{t=1}^{T-1}
    \frac{b_t }{B_t} + 1}}.
\end{equation}
\end{lemma}

\begin{proof}
  We make a case distinction based on the number of epochs started by the algorithm. First, let us check the general case of $\ge 3$ epochs (at
least two normal epochs after the startup epoch). 
We apply the controller
regret bound, Lemma~\ref{lem:controller.epoch.bd}, to the last two epochs each. Suppose
these start after $\tau_1$ and $\tau_2$. For any $\eta \in \mathcal G$
that has not expired after $T$ rounds, we find
\begin{multline*}
  - \sum_{t \in (\tau_1, \tau_2],t \ge a^\eta} \clipsurr_t^\eta(\w_t^\eta)
  - \sum_{t \in (\tau_2, T],t \ge a^\eta} \clipsurr_t^\eta(\w_t^\eta)\\
  ~\le~
  \ln \ceil*{2 \origlog_2 \del*{\sum_{s=1}^{\tau_2-1}
    \frac{b_s }{B_s} + 1}}
  +
  \ln \ceil*{2 \origlog_2 \del*{\sum_{t=1}^{T-1}
    \frac{b_t }{B_t} + 1}}
  .
\end{multline*}
The regret on all epochs except the last two is bounded by
Lemma~\ref{lem:tiny.past}. So together we obtain \eqref{eqn:controllersleepingregret}.
Alternatively, suppose there are 2 epochs. Then, since we get no clipped
regret in the 1st epoch (as $B_{t-1} = 0$ throughout it, and hence $\clipgrad_t = \zeros$ and $\clipsurr_t^\eta(\cdot)=0$), we apply the controller regret bound only in the
second epoch to get
\[
- \sum_{t\in [1, T],t \ge a^\eta} \clipsurr_t^\eta(\w_t^\eta)
  \leq
  \ln \ceil*{2 \origlog_2 \del*{\sum_{t=1}^{T-1}
    \frac{b_t }{B_t} + 1}},
\]
and \eqref{eqn:controllersleepingregret} also holds. Finally, if there
is only 1 epoch, then our clipped regret is 0, so \eqref{eqn:controllersleepingregret} also holds.
\end{proof}

The proof of Lemma~\ref{lem:controllerbd} is completed by plugging in the
upper bounds from Lemmas~\ref{lem:controller.cheap.clipping},
\ref{lem:early.rounds} and~\ref{lem:overall.controller.regret}
into Proposition~\ref{prop:threeway}. 

\section{Composition Proofs of Theorems~\ref{thm:mainbound},
\ref{thm:roadnottaken} and
\ref{thm:mainsketchbound}}\label{appx:composition}

We combine the proofs of Theorems~\ref{thm:mainbound} and
\ref{thm:mainsketchbound}, which are both special cases of the 
abstract result Theorem~\ref{thm:abstractoptimize} below. The proof of
Theorem~\ref{thm:roadnottaken} is very similar in spirit, but
sufficiently different that we postpone it to the end of the section.
\begin{theorem}\label{thm:abstractoptimize}
  Suppose there exist a number $V \geq 0$ and positive semi-definite
  matrices $\F^\eta$ (possibly dependent on $\eta$) such that
  $\rk(\F^\eta) \leq r$, $\tr(\F^\eta) \leq s$ and the linearized
  regret is at most 
  \[
    \Rtrick_T^\u
    ~\le~
    \eta V 
    + \frac{
      \ln \det \del*{\I + 2 \eta^2 \sigma^2 \F^\eta}
      + \frac{1}{2 \sigma^2} \norm{\u}^2_2
      + 2 \ln \ceil*{2 \origlog_2 T}_+
      + \frac{1}{2}
    }{\eta}
    + 2 B_T
  \]
  simultaneously for all $\eta \in \mathcal{G}$ such that $\eta \leq
  \frac{1}{2 B_T}$. Then the linearized regret is both bounded by
  \[
    \Rtrick_T^\u
    ~\le~
    \frac{5}{2} \sqrt{V (\tfrac{1}{2 \sigma^2} \norm{\u}^2_2 + Z_T)} + 5
    B_T (\tfrac{1}{2 \sigma^2} \norm{\u}^2_2 + Z_T) + 2 B_T,
  \]
  where 
    $Z_T =
    r \ln \del*{1 + \frac{\sigma^2 s}{2
    B_T^2 r}}
    + 2 \ln \ceil*{2 \origlog_2 T}_+
    + \frac{1}{2}$,
  and by
  \[
    \Rtrick_T^\u
    ~\le~
    \frac{5}{2} \sqrt{\Big(V + 2 \sigma^2 s \Big)
    \Big(
    \tfrac{1}{2 \sigma^2} \norm{\u}^2_2
    + Z'_T
    \Big)
    } + 5 B_T \Big(
    \tfrac{1}{2 \sigma^2} \norm{\u}^2_2
    + Z'_T
    \Big)
    + 2 B_T,
  \]
  where $Z'_T = 2 \ln \ceil*{2 \origlog_2 T}_+ + \frac{1}{2}$.
\end{theorem}
Theorem~\ref{thm:mainbound} corresponds to the case $V = V_T^\u$
and $\F^\eta = \F_T$, such that $\Tr(\F^\eta) = \sum_{t=1}^T \|\grad_t\|_2^2$;
Theorem~\ref{thm:mainsketchbound} is obtained with $V = V_T^\u +
\frac{2 \sigma^2 m \Omega_q}{m-q}$ and $\F^\eta = (\S^\eta_T)^\top \S^\eta_T$. To bound $\tr(\F^\eta)$ we use that $(\S^\eta_T)^\top \S^\eta_T \preceq (\G^\eta_T)^\top \G^\eta_T =
\sum_{t=a^\eta}^T \grad_t \grad_t^\top \preceq \F_T$, where the first
inequality holds because $\S^\eta_T$ is the Frequent Directions
approximation of $\G_T^\eta$ \citep{GhashamiEtAl2016}. It follows
that $\tr(\F^\eta) \le \tr(\F_T) = \sum_{t=1}^T \|\grad_t\|_2^2$. We may
further use that $\rk(\F^\eta) \leq 2m$, by the dimensionality of $\S_T^\eta$. The
precondition of Theorem~\ref{thm:abstractoptimize} is established by
Theorems~\ref{thm:untuned.regret} and \ref{thm:untuned.sketchregret},
respectively, and the observation that $Q_T \leq T$.

To prove Theorem~\ref{thm:abstractoptimize} we start with a general lemma about optimizing in $\eta$:
\begin{lemma}\label{lem:optimize.eta}
  For any $X,Y>0$, 
  \begin{equation*}
    \min_{\eta \in \mathcal{G}~:~\eta \leq \frac{1}{2 B_T}} \eta X + \frac{Y}{\eta} 
      ~ \leq ~
      \frac{5}{2} \sqrt{XY} + 5 B_T Y.
  \end{equation*}
\end{lemma}

\begin{proof}
Let us denote the unconstrained optimizer of the left-hand side by $\hat
\eta = \sqrt{Y/X}$. We distinguish two cases: first, when $\hat \eta \le
\frac{1}{2 B_T}$, we upper bound the left-hand side by choosing the
closest grid point $\eta \in \mathcal G$ below~$\hat \eta$ (which, in
the worst case, is at $\hat\eta/2$) to obtain
\[
  \min_{\eta \in \mathcal{G}~:~\eta \leq \frac{1}{2 B_T}} \eta X + \frac{Y}{\eta} 
  ~\le~
  \max_{\eta \in [\hat \eta/2, \hat \eta]} \eta X + \frac{Y}{\eta} 
  = \frac{5}{2} \sqrt{XY}.
\]
In the second case, if $\hat \eta > \frac{1}{2 B_T}$, we plug in the
highest available grid point (for which the worst case is $\frac{1}{4
B_T}$) to find
\[
  \min_{\eta \in \mathcal{G}~:~\eta \leq \frac{1}{2 B_T}} \eta X + \frac{Y}{\eta} 
  ~\le~
  \frac{1}{4 B_T} X + 4 B_T Y
  ~<~
  5 B_T Y,
\]
where the second inequality follows by the assumption that $\hat \eta >
\frac{1}{2 B_T}$. In both cases the conclusion of the lemma follows.
\end{proof}

\begin{proof} \textbf{(Theorem~\ref{thm:abstractoptimize})}
We start with the first claim of the theorem. By assumption, for any
$\eta \le \frac{1}{2 B_T}$ in the grid $\mathcal G$, we have
\[
  \Rtrick_T^\u
    ~\le~
    \eta V + \frac{A^\eta}{\eta} + 2 B_T
    ~\le~
    \eta V + \frac{A}{\eta} + 2 B_T,
\]
where 
\begin{align*}
  A^\eta &=
    \ln \det \del*{\I + 2 \eta^2 \sigma^2 \F^\eta}
    + \frac{1}{2 \sigma^2} \norm{\u}^2_2
    + 2 \ln \ceil*{2 \origlog_2 T}
    + \frac{1}{2}\\
  A &= 
    r \ln\Big(1 + \frac{\sigma^2 s}{2 B_T^2 r}\Big)
    + \frac{1}{2 \sigma^2} \norm{\u}^2_2
    + 2 \ln \ceil*{2 \origlog_2 T}
    + \frac{1}{2},
\end{align*}
and $A^\eta \leq A$ follows from $\eta \leq 1/(2B_T)$, the first inequality in
Lemma~\ref{lem:matrixinequalities} below and the fact that the expression $r \ln
\del*{1+\frac{s}{r}}$ is increasing in $r \ge 0$ for all $s \ge 0$.
Lemma~\ref{lem:optimize.eta} therefore implies that
\[
  \Rtrick_T^\u ~\le~ \frac{5}{2} \sqrt{V A} + 5 B_T A + 2 B_T,
\]
which establishes the first claim of the theorem.

For the second claim of the theorem, we upper bound $A^\eta$
differently,
using the second inequality in Lemma~\ref{lem:matrixinequalities}, to
obtain
\[
  \Rtrick_T^\u
  ~\le~
  \eta V 
  + 2 \eta \sigma^2 s
  + \frac{A'}{\eta}
  + 2 B_T,
  \qquad
  \text{where}
  \qquad
  A' ~=~
  \frac{1}{2 \sigma^2} \norm{\u}^2_2
  + 2 \ln \ceil*{2 \origlog_2 T}
  + \frac{1}{2}
  .
\]
Using Lemma~\ref{lem:optimize.eta}, the second claim follows, which
completes the proof of the theorem.
\end{proof}

\begin{lemma}\label{lem:matrixinequalities}
  For any positive semi-definite matrix $\M \in \reals^{d \times d}$
  \begin{equation*}
    \log \det (\I + \M)
    ~\leq~
    \rk(\M) \log \del*{1 + \frac{\Tr(\M)}{\rk(\M)}}
    ~\leq~
    \Tr(\M),
  \end{equation*}
  where the middle term is extended by continuity to equal zero at $\M = \zeros$.
\end{lemma}

\begin{proof}
  If $\M = \zeros$ then all three equal zero and we are done. Otherwise,
  let $\lambda_1,\ldots,\lambda_d$ be the eigenvalues of $\M$. Then
  $(1+\lambda_1),\ldots,(1+\lambda_d)$ are the eigenvalues of $\I + \M$,
  and Jensen's inequality implies
  \begin{multline*}
    \log \det (\I + \M)
      = \sum_{i=1}^d \log (1+\lambda_i)
      = \rk(\M) \sum_{i : \lambda_i \neq 0} \frac{1}{\rk(\M)} \log
      (1+\lambda_i)\\
      \leq \rk(\M) \log \del*{1 + \sum_{i : \lambda_i \neq 0}
      \frac{\lambda_i}{\rk(\M)}}
      = \rk(\M) \log \del*{1 + \frac{\Tr(\M)}{\rk(\M)}},
  \end{multline*}
  which proves the first inequality. The second inequality follows
  by $\ln (1+x) \le x$ for all $x \ge 0$.
\end{proof}

To conclude the section, it remains to prove Theorem~\ref{thm:roadnottaken}.
\begin{proof} \textbf{(Theorem~\ref{thm:roadnottaken})}
Starting from Theorem~\ref{thm:untuned.regret}, which still holds even
though $\sigma$ depends on~$\eta$, we plug in $\sigma = 1/\sqrt{\alpha \eta}$ to
obtain
\begin{multline*}
  \Rtrick_T^\u
  ~\le~
  \eta V_T^\u
  + \frac{A}{\eta}
  + \frac{\alpha}{2} \norm{\u}^2_2
  + 2 B_T,\\
  \text{where}
  \qquad
  A ~=~ \ln \det \del*{\I + \frac{1}{B_T \alpha}  \F_T}
    + 2 \ln \ceil*{2 \origlog_2 T}_+
    + \frac{1}{2},
\end{multline*}
for all $\eta \in \mathcal{G}$ such that $\eta \leq 1/(2 B_T)$.
Lemma~\ref{lem:optimize.eta} therefore implies that
\[
  \Rtrick_T^\u ~\le~ \frac{5}{2} \sqrt{V_T^\u A} + 5 B_T A +
  \frac{\alpha}{2} \norm{\u}^2_2 + 2 B_T,
\]
and the first claim of the theorem follows upon applying
the first inequality from Lemma~\ref{lem:matrixinequalities} with $\M
= \frac{1}{B_T \alpha} \F_T$ and observing that $\tr(\F_T) =
\sum_{t=1}^T \|\grad_t\|_2^2$.

For the second claim of the theorem, we again start from
Theorem~\ref{thm:untuned.regret} and now apply the second inequality
from Lemma~\ref{lem:matrixinequalities} for $\M = \frac{2 \eta}{\alpha}
\F_T$ to obtain
\[
  \Rtrick_T^\u
  ~\le~
  \eta V_T^\u
  + \frac{Z'_T}{\eta}
  + \frac{2}{\alpha} \tr(\F_T)
  + \frac{\alpha}{2} \norm{\u}^2_2
  + 2 B_T.
\]
Using Lemma~\ref{lem:optimize.eta} and $\tr(\F_T) = \sum_{t=1}^T
\|\grad_t\|_2^2$, the second claim follows, which completes the proof of
the theorem.
\end{proof}

\section{Proofs of Corollaries~\ref{cor:roughthm} and~\ref{cor:coordroughthm}}
\label{app:fulldiagcorollaries}

\begin{proof} \textbf{(Corollary~\ref{cor:roughthm})}
  If we ignore the corner case that $B_T^2$ in the definition of $Z_T$
  is exceedingly small, then \eqref{eqn:roughbound} follows from
  \eqref{eqn:mainbound} upon bounding $\|\g_t\|_2^2 \leq G_2^2$, $B_T
  \leq 2 D_2 G_2$, and observing that $Z_T$ is increasing in $\rk(\F_T)
  \leq d$. To see that \eqref{eqn:roughbound} holds in general, even for
  very small $B_T^2$, we need to verify that $V_T^{\u} Z_T= O(V_T^\u
  d \log(D_2 G_2 T/d))$ and $B_T Z_T = O(D_2 G_2 \log(D_2 G_2 T/d))$. To
  establish the
  first of these, we reason as follows:
  \begin{align*}
    V_T^\u \rk(\F_T) &\ln \del*{1 + \frac{\sigma^2\sum_{t=1}^T \|\grad_t\|_2^2}{8
      B_T^2 \rk(\F_T)}}\\
      &\leq V_T^\u d \log\del*{1 + \frac{D_2^2 G_2^2 T}{8
      B_T^2 d}}\\
      &= V_T^\u d \log\del*{\frac{D_2^2 G_2^2 T^3}{d}}
        + V_T^\u d \log\del*{\frac{d}{D_2^2 G_2^2 T^3} +
        \frac{1}{8 B_T^2 T^2}}\\
      &\leq V_T^\u d \log\del*{\frac{D_2^2 G_2^2 T^3}{d}}
        + V_T^\u \del*{\frac{d^2}{D_2^2 G_2^2 T^3} +
        \frac{d}{8 B_T^2 T^2}}\\
      &\leq V_T^\u d \log\del*{\frac{D_2^2 G_2^2 T^3}{d}}
        + \frac{2 d^2}{T^2} +
        \frac{d}{8 T}
      = O\del*{V_T^\u d \log \frac{D_2 G_2 T}{d}},
  \end{align*}
  where the last inequality follows from $V_T^\u \leq B_T^2 T
  \leq 2 D_2^2 G_2^2 T$. To establish the second case, we
  observe that
  \begin{align*}
    B_T 
    \rk(\F_T) &\ln \del*{1 + \frac{\sigma^2\sum_{t=1}^T \|\grad_t\|_2^2}{8
      B_T^2 \rk(\F_T)}}
      \leq 
        B_T d \ln \del*{1 + \frac{D_2^2 G_2^2 T}{8 B_T^2 d}}\\
      &= B_T d \log\del*{B_T^2 + \frac{D_2^2 G_2^2 T}{8 d}}
        - 2 d B_T \log B_T\\
      &\leq 2 D_2 G_2 d \log(4 D_2^2 G_2^2 + \frac{D_2^2 G_2^2 T}{8d})
        + \frac{2d}{e}
      = O\del*{D_2 G_2 d \log\del*{\frac{D_2 G_2 T}{d}}},
  \end{align*}
  where the last inequality uses that $x \log x \geq -1/e$. This
  completes the proof of \eqref{eqn:roughbound}.
\end{proof}

\begin{proof}\textbf{(Corollary~\ref{cor:coordroughthm})}
  To see that \eqref{eqn:roughcoordbound} follows from
  \eqref{eqn:coordV}, we need to verify that $V_{T, i}^{u_i} Z_{T,i}=
  O(V_{T,i}^{u_i} \log(D_\infty G_\infty T))$ and $B_{T,i} Z_{T,i} =
  O(D_\infty G_\infty \log(D_\infty G_\infty T))$. These follow as the
  one-dimensional special cases of the analogous quantities in the proof
  of Corollary~\ref{cor:roughthm}.

  The first part of \eqref{eqn:roughhighdimcoordbound} then follows from
  \eqref{eqn:coordhighdim} upon observing that $B_{T,i} \leq 2 D_\infty
  \max_t \|\grad_t\|_1 \leq 2 D_\infty G_2 \sqrt{d}$. The second part
  follows because
  \[
    \sum_{i=1}^d \|g_{1:T,i}\|_2
      = d \sum_{i=1}^d \frac{1}{d} \sqrt{\sum_{t=1}^T g_{t,i}^2}
      \leq d \sqrt{\sum_{i=1}^d \frac{1}{d} \sum_{t=1}^T g_{t,i}^2}
      = \sqrt{d \sum_{t=1}^T \|\grad_t\|_2^2}
      \leq G_2 \sqrt{d T}
  \]
  by Jensen's inequality.
\end{proof}

\clearpage
\section{Experimental Results}\label{app:expresults}

\begin{table}[ht]
\centering
\begin{tabular}{lrrlr}
  \hline
Data set & $T$ & $d$ & Outcome & $P(y = 1)$ \\
  \hline
a9a & 32561 & 123 & binary & 0.24 \\ 
  australian & 690 &  14 & binary & 0.44 \\ 
  breast-cancer & 683 &   9 & binary & 0.35 \\ 
  covtype & 581012 &  54 & binary & 0.49 \\ 
  diabetes & 768 &   8 & binary & 0.65 \\ 
  heart & 270 &  13 & binary & 0.44 \\ 
  ijcnn1 & 91701 &  22 & binary & 0.10 \\ 
  ionosphere & 351 &  34 & binary & 0.64 \\ 
  phishing & 11055 & 68 & binary & 0.56 \\
  splice & 1000 &  60 & binary & 0.52 \\ 
  w8a & 49479 & 300 & binary & 0.03 \\ 
  abalone & 4177 &   8 & real &  \\ 
  bodyfat & 252 &  14 & real &  \\ 
  cpusmall & 8192 &  12 & real &  \\ 
  housing & 506 &  13 & real &  \\ 
  mg & 1385 &   6 & real &  \\ 
  space\_ga & 3107 &   6 & real &  \\ 
   \hline
\end{tabular}
\caption{Summary of the data sets}\label{tbl:sumdat}
\end{table}

\clearpage
\begin{table}[ht]
\resizebox{\linewidth}{!}{
\begin{tabular}{llrrrrrrrrr}
  \hline
Data set & Loss & AdaGrad & GDnorm & OGDt & MGCo & MGF2 & MGF11 & MGF26 & MGF51 & MGFull \\
  \hline
a9a & hinge & 232414 & 37708 & 22472 & 17012 & 13754 & 12230 & 11671 & 11160 & \textbf{11045} \\
   & logistic & 30910 & 7176 & 3817 & \textbf{1340} & 2249 & 1990 & 1910 & 1813 & 1783 \\
  australian & hinge & 279 & 99 & 68 & 41 & 40 & \textbf{34} & \textbf{34} & \textbf{34} & \textbf{34} \\
   & logistic & 1250 & 492 & 359 & 48 & 52 & \textbf{45} & \textbf{45} & \textbf{45} & \textbf{45} \\
  breast-cancer & hinge & 214 & 106 & 84 & \textbf{24} & 26 & 25 & 25 & 25 & 25 \\
   & logistic & 288 & 147 & 119 & \textbf{25} & \textbf{26} & \textbf{26} & \textbf{26} & \textbf{26} & \textbf{26} \\
  covtype & hinge & 1317765 & 254930 & 141706 & 66797 & 83958 & 71218 & 62087 & 31368 & \textbf{31355} \\
   & logistic & 78430 & 33935 & 14042 & 4713 & 12214 & 10516 & 8941 & 3668 & \textbf{3663} \\
  diabetes & hinge & 553 & 306 & 185 & 75 & 63 & \textbf{59} & \textbf{59} & \textbf{59} & \textbf{59} \\
   & logistic & 474 & 241 & 133 & 53 & 40 & \textbf{39} & \textbf{39} & \textbf{39} & \textbf{39} \\
  heart & hinge & 329 & 217 & 148 & \textbf{35} & 42 & \textbf{35} & \textbf{35} & \textbf{35} & \textbf{35} \\
   & logistic & 376 & 246 & 155 & \textbf{30} & 35 & 32 & 31 & 31 & 31 \\
  ijcnn1 & hinge & 12292 & 3925 & 1198 & \textbf{885} & 1633 & 1327 & 901 & 901 & 901 \\
   & logistic & 15303 & 4473 & 1344 & \textbf{976} & 1798 & 1415 & 1086 & 1086 & 1086 \\
  ionosphere & hinge & 2672 & 1102 & 753 & \textbf{169} & 252 & 211 & 206 & 205 & 205 \\
   & logistic & 5786 & 1897 & 1426 & 240 & 280 & 242 & \textbf{238} & \textbf{238} & \textbf{238} \\
  phishing & hinge & 6752 & 3162 & 1757 & 610 & 635 & 607 & 547 & \textbf{518} & \textbf{518} \\
   & logistic & 22814 & 7394 & 3320 & 1208 & 967 & 890 & 802 & \textbf{767} & \textbf{767} \\
  splice & hinge & 2451 & 777 & 694 & \textbf{243} & 303 & 290 & 277 & 288 & 280 \\
   & logistic & 3014 & 819 & 726 & 183 & 182 & 181 & 179 & 177 & \textbf{175} \\
  w8a & hinge & 349174 & 139920 & 255346 & \textbf{18789} & 34395 & 31966 & 32080 & 31823 & 29661 \\
   & logistic & 86921 & 21095 & 40519 & \textbf{3324} & 4546 & 4230 & 4049 & 3977 & 3865 \\
  abalone & absolute & 12650 & 7395 & 5027 & 1317 & 2194 & \textbf{748} & \textbf{748} & \textbf{748} & \textbf{748} \\
   & squared & 73507 & 44166 & 37398 & 6725 & 7642 & \textbf{6179} & \textbf{6179} & \textbf{6179} & \textbf{6179} \\
  bodyfat & absolute & 319 & 98 & 75 & 30 & 24 & \textbf{23} & \textbf{23} & \textbf{23} & \textbf{23} \\
   & squared & 351 & 37 & 28 & 10 & \textbf{7} & \textbf{8} & \textbf{8} & \textbf{8} & \textbf{8} \\
  cpusmall & absolute & 533948 & 199595 & 182464 & 40537 & 22251 & 14301 & \textbf{14287} & \textbf{14287} & \textbf{14287} \\
   & squared & 12109845 & 2740512 & 3082005 & 561505 & 353329 & \textbf{351253} & 351257 & 351257 & 351257 \\
  housing & absolute & 9979 & 3557 & 3067 & 946 & 949 & 776 & \textbf{746} & \textbf{746} & \textbf{746} \\
   & squared & 154729 & 52053 & 55064 & 20191 & 16103 & \textbf{15973} & 15975 & 15975 & 15975 \\
  mg & absolute & 277 & 110 & 92 & 30 & 40 & \textbf{28} & \textbf{28} & \textbf{28} & \textbf{28} \\
   & squared & 112 & 32 & \textbf{15} & 19 & 17 & 18 & 18 & 18 & 18 \\
  space\_ga & absolute & 1393 & 908 & 523 & 133 & 259 & \textbf{65} & \textbf{65} & \textbf{65} & \textbf{65} \\
   & squared & 1451 & 534 & 528 & \textbf{40} & 75 & 55 & 55 & 55 & 55 \\
   \hline
\end{tabular}
}
\caption{The regret of each algorithm for the various data sets and loss
functions (rounded to whole numbers). Boldface indicates that the regret
is within one unit of the minimum for the row.}
\label{tab:regret table}
\end{table}

\clearpage
\global\pdfpageattr\expandafter{\the\pdfpageattr/Rotate 0}

\subsection{Hypertune Results}
\label{app:hypertune}

In this section we investigate the effect of hyperparameter tuning. Each
of the algorithms that we consider has one free parameter, $\sigma$, for
which the theory advocates tuning it in terms of the (unknown) norm of
the comparator, or the maximal distance from the comparator within the
domain. This theoretical recommendation is what we employed in our
experiments in Section~\ref{sec:experiments}. In contrast, we now ask
what performance one may reach by optimizing the $\sigma$ parameter for
the data in hand. Our approach will be to evaluate all algorithms on a discrete grid of parameter settings. For convenience of comparison between full and coordinate-wise algorithms, we parameterise our grid by the factor by which we scale the theoretically optimal tuning from Section~\ref{sec:experiments}. We include in our grid exponentially small factors $2^j$ for $j=-7, \dots, -3$, followed by a linear grid running from $2^{-3}$ to $3$ with steps of size $1/8$, resulting in a grand total of $28$ grid points. We visualise the entire performance profile for four selected data sets in Figure~\ref{fig:hypertune.plots}. There we see that the optimal tuning for $\sigma$ can be either higher or lower than the theoretical recommendation, and whether it should be higher or lower can be different for different algorithms even on the same data set.

\begin{figure}
  \centering
  \subfigure[\texttt{covtype} with logistic loss]{%
    \includegraphics[width=.5\textwidth]{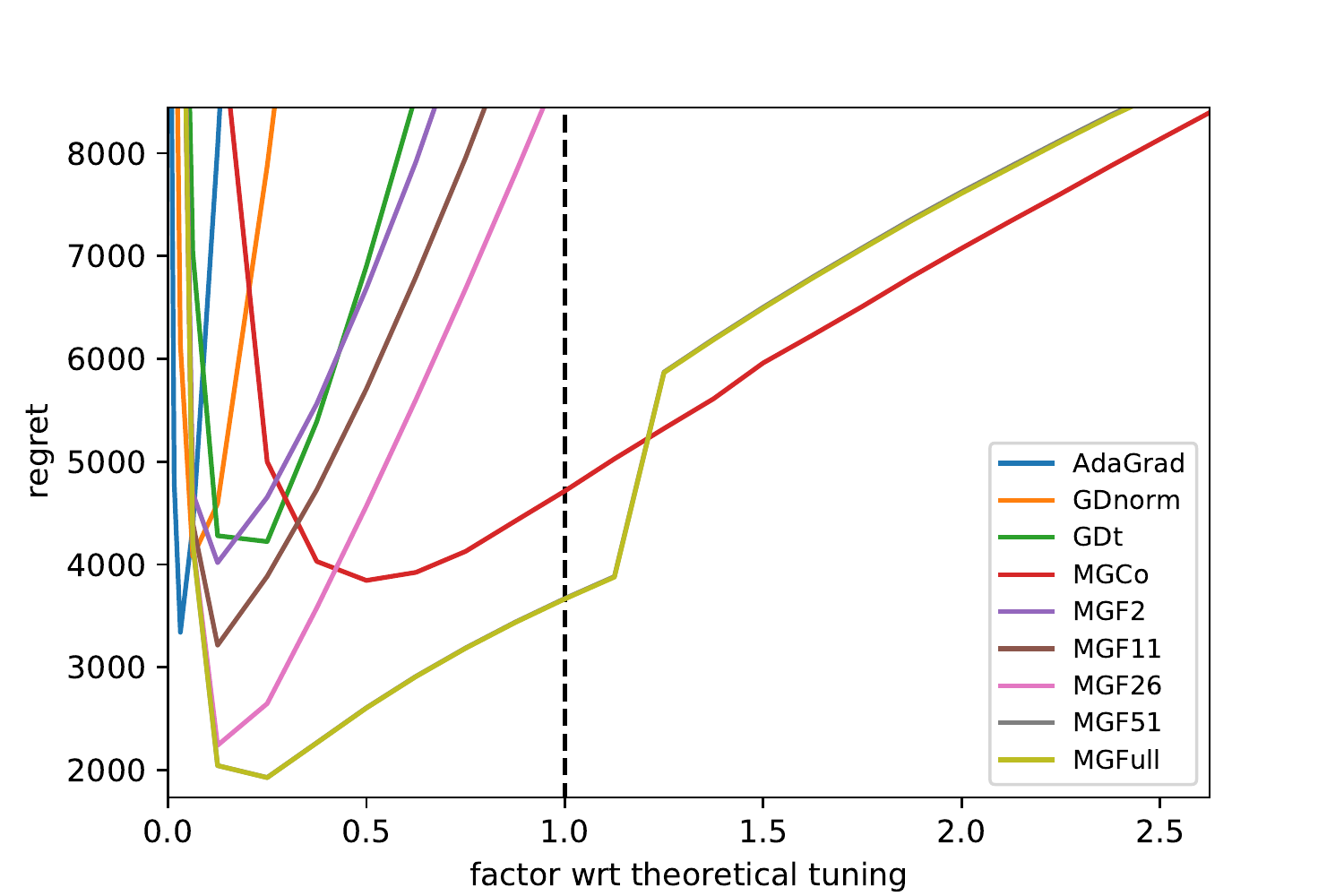}
  }%
  \subfigure[\texttt{housing} with squared loss]{%
    \includegraphics[width=.5\textwidth]{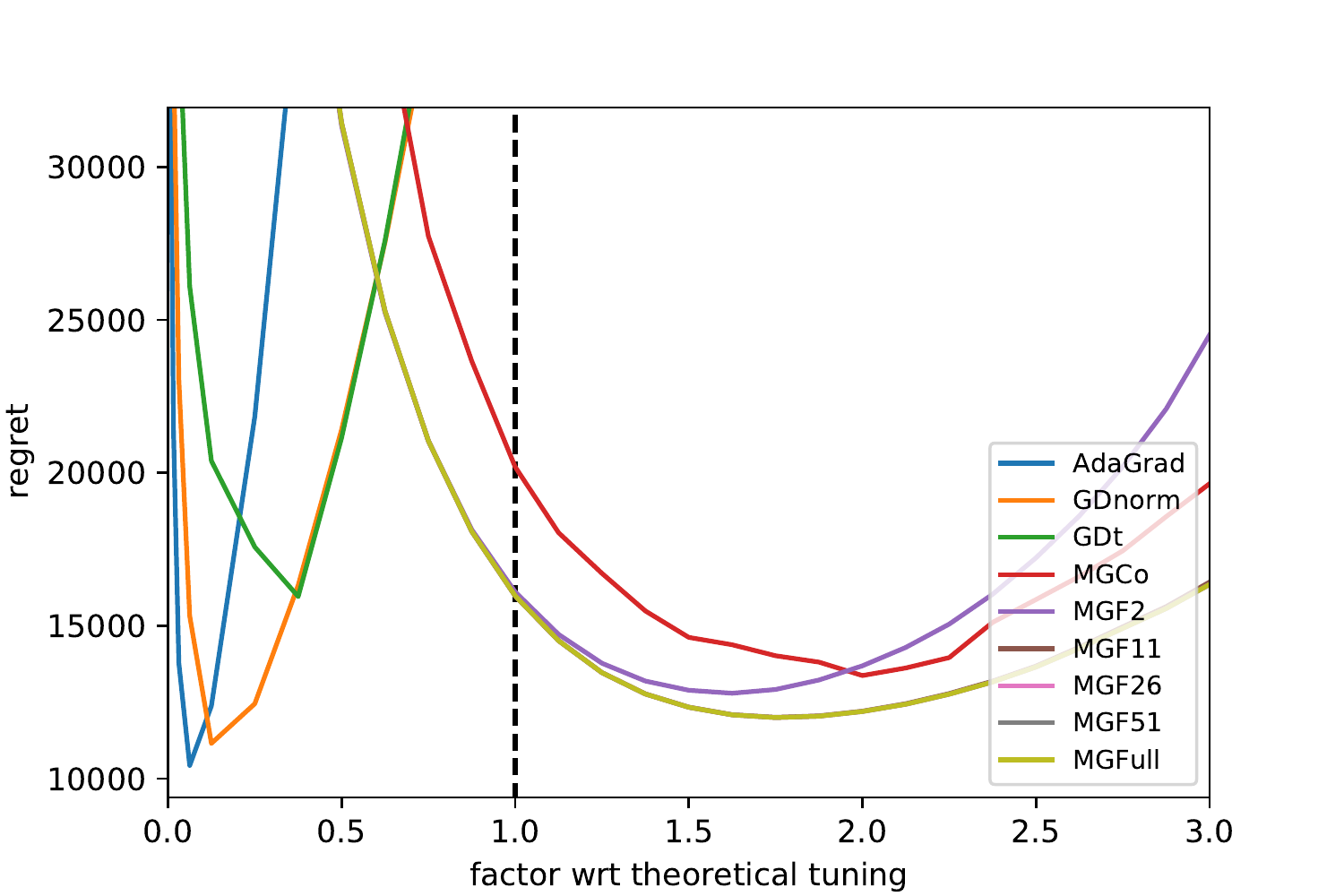}
  }
  \subfigure[\texttt{australian} with hinge loss]{%
    \includegraphics[width=.5\textwidth]{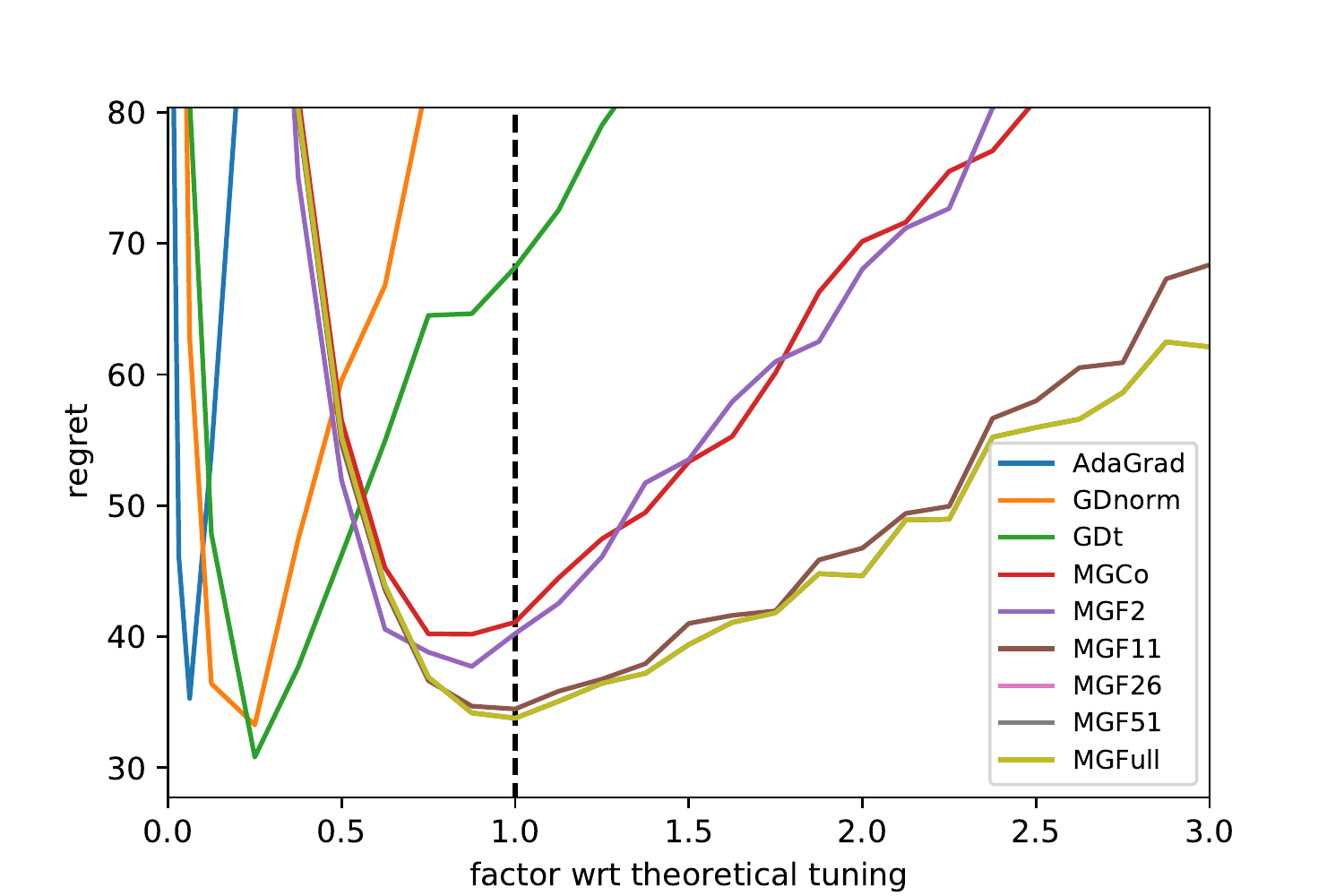}
  }%
  \subfigure[\texttt{space\_ga} with absolute loss]{%
    \includegraphics[width=.5\textwidth]{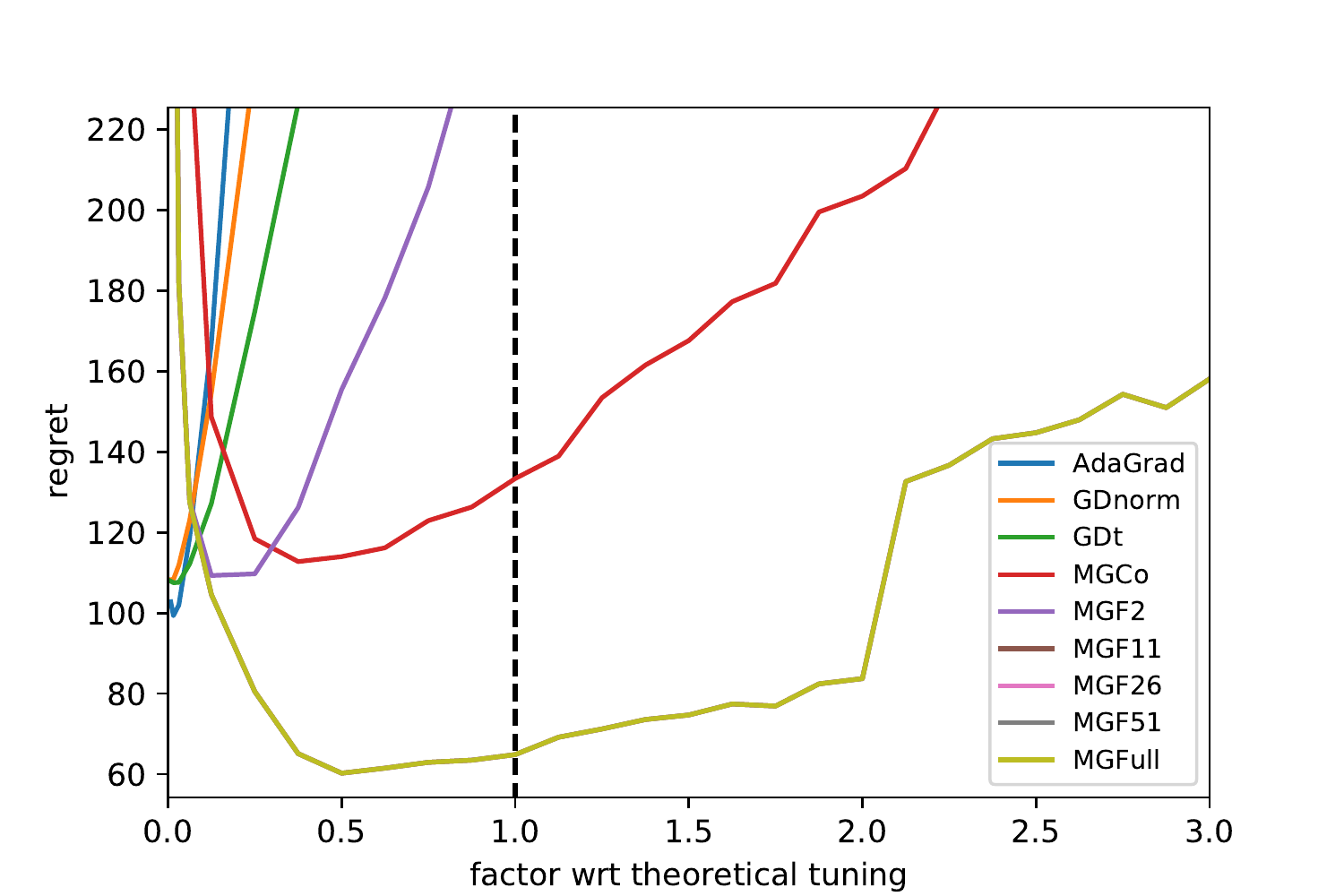}
  }%
  \caption{Performance of all algorithms as a function of the tuning
  parameter $\sigma$ on four selected data sets. We have parameterised
  $\sigma$ by a factor times its theoretically optimal tuning. The
  dotted line indicates the standard tuning (corresponding to factor
  $1$), at which the results from Table~\ref{tab:regret table} are
  reproduced. Note that with different $\sigma$, the algorithms produce
  different iterates $\w_t$, and as a result see different gradients. For MetaGrad, these further affect the set of active experts that are maintained by the master. These effects make the curves interestingly non-smooth.}\label{fig:hypertune.plots}
\end{figure}

We evaluate our algorithms on all data sets. (Recall that a helpful
summary of the properties of each data set can be found in
Table~\ref{tbl:sumdat}). First, in
Table~\ref{tab:hypertune.regret.table}, which parallels
Table~\ref{tab:regret table}, we present the hypertuned regret for each
algorithm on each data set. These results are subsequently summarised by
Table~\ref{tab:hypertune.overview}, which is the hypertuned analogue of
Table~\ref{tbl:sumresults}. Here we compare all algorithms to AdaGrad
instead of OGDt, as it has the best hypertuned performance among prior
existing algorithms. (Interestingly, the theoretical prediction that
OGDnorm dominates OGDt does materialise for the hypertuned regret, while
it did not under the bonafide tuning of Section~\ref{sec:experiments}.)
We can conclude from Table~\ref{tab:hypertune.overview} that AdaGrad,
OGDnorm and MetaGrad, in either full or sketched forms, all have very
similar performance. As discussed in
Section~\ref{sec:additionalexperiments}, this suggests that the
empirical superiority of MetaGrad in the experiments from
Section~\ref{sec:experiments} may be attributed to its ability to better
adapt to the optimal learning rate $\eta$.
We should also remember that the hypertuned performance is not a-priori
indicative of practical results. It is an interesting challenge to
develop new methods that achieve as much of this hypertuned performance
in practice as possible, but which also come with corresponding
theoretical guarantees. In this quest MetaGrad constitutes a solid first step.

\begin{table}[ht]
\centering
\begin{tabular}{lrrr}
  \hline
Algorithm & \# best & \# better than AdaGrad & MedianRatio \\
  \hline
   AdaGrad & 14 & 34 & 1.00 \\ 
  OGDnorm & 8 & 18 & 0.99 \\ 
  OGDt & 8 & 14 & 1.07 \\ 
  MGCo & 1 & 6 & 1.15 \\ 
  MGF2 & 1 & 10 & 1.09 \\ 
  MGF11 & 6 & 18 & 1.01 \\ 
  MGF26 & 4 & 18 & 1.01 \\ 
  MGF51 & 7 & 18 & 1.01 \\ 
  MGFull & 9 & 18 & 1.01 \\ 
   \hline
\end{tabular}
\caption{Comparison of algorithms with AdaGrad, with the $\sigma$
hyperparameter optimized in hindsight for the data. The MedianRatio column
contains the median ratio of the regret of each algorithm over that of
AdaGrad. Columns ``\# best'' and ``\# better than AdaGrad'' count cases where
the algorithm is at most one regret unit above the best algorithm or AdaGrad,
respectively.
}\label{tab:hypertune.overview}
\end{table}

\clearpage
\begin{table}[ht]
  \centering
\resizebox{\linewidth}{!}{%
\begin{tabular}{lllllllllll}
  \hline
   Data set & Loss & AdaGrad & OGDnorm & OGDt & MGCo & MGF2 & MGF11 & MGF26 & MGF51 & MGFull \\
  \hline
a9a & hinge & 1512 & \textbf{484} & 504 & 664 & 627 & 592 & 588 & 583 & 585 \\ 
   & logistic & \textbf{304} & 412 & 472 & 527 & 512 & 484 & 478 & 473 & 473 \\ 
  australian & hinge & 35 & 33 & \textbf{31} & 40 & 38 & 34 & 34 & 34 & 34 \\ 
   & logistic & \textbf{25} & 26 & \textbf{24} & 33 & 36 & 34 & 34 & 34 & 34 \\ 
  breast-cancer & hinge & \textbf{20} & \textbf{20} & \textbf{19} & 23 & 23 & 22 & 22 & 22 & 22 \\ 
   & logistic & 21 & \textbf{19} & 26 & 23 & 24 & 24 & 24 & 24 & 24 \\ 
  covtype & hinge & 8070 & 6382 & 6205 & 9095 & 6811 & 5648 & 5067 & \textbf{4939} & \textbf{4939} \\ 
   & logistic & 3339 & 4077 & 4222 & 3844 & 4017 & 3214 & 2240 & 1927 & \textbf{1926} \\ 
  diabetes & hinge & \textbf{58} & 73 & 76 & 71 & 62 & 59 & 59 & 59 & 59 \\ 
   & logistic & \textbf{36} & 50 & 55 & 49 & 40 & 39 & 39 & 39 & 39 \\ 
  heart & hinge & \textbf{34} & 35 & \textbf{33} & 35 & 35 & \textbf{33} & \textbf{34} & \textbf{34} & \textbf{34} \\ 
   & logistic & \textbf{28} & \textbf{28} & 30 & 30 & 30 & 29 & 29 & 29 & 29 \\ 
  ijcnn1 & hinge & \textbf{419} & 550 & 542 & 597 & 751 & 640 & 502 & 502 & 502 \\ 
   & logistic & \textbf{500} & 663 & 782 & 804 & 1021 & 823 & 715 & 715 & 715 \\ 
  ionosphere & hinge & 106 & \textbf{103} & 110 & 110 & 110 & 108 & 108 & 108 & 108 \\ 
   & logistic & 106 & \textbf{98} & 111 & 106 & 104 & 103 & 103 & 103 & 103 \\ 
  phishing & hinge & \textbf{290} & 471 & 433 & 378 & 326 & 311 & 301 & 303 & 303 \\ 
   & logistic & \textbf{258} & 457 & 492 & 423 & 345 & 335 & 331 & 330 & 330 \\ 
  splice & hinge & 210 & 200 & 211 & 179 & \textbf{175} & 177 & 180 & 178 & 177 \\ 
   & logistic & 150 & 147 & 174 & 137 & 139 & 138 & 137 & \textbf{136} & \textbf{136} \\ 
  w8a & hinge & 3299 & 1458 & 2545 & 935 & 875 & 875 & 875 & \textbf{873} & \textbf{873} \\ 
   & logistic & 1147 & 1123 & 2764 & 1224 & 1159 & 1133 & 1124 & 1121 & \textbf{1117} \\ 
  abalone & absolute & 1038 & 1040 & 1033 & 1211 & 1131 & \textbf{692} & \textbf{692} & \textbf{692} & \textbf{692} \\ 
   & squared & 6204 & 6950 & 7627 & 6698 & 7127 & \textbf{6179} & \textbf{6179} & \textbf{6179} & \textbf{6179} \\ 
  bodyfat & absolute & 23 & \textbf{18} & \textbf{17} & 28 & 24 & 23 & 23 & 23 & 23 \\ 
   & squared & 6 & \textbf{3} & \textbf{4} & 7 & 6 & 6 & 6 & 6 & 6 \\ 
  cpusmall & absolute & 11379 & 10608 & 10577 & 15284 & 10489 & \textbf{9922} & 9976 & 9976 & 9976 \\ 
   & squared & 479645 & 478014 & 694804 & 545240 & 279054 & \textbf{278921} & 278947 & 278947 & 278947 \\ 
  housing & absolute & \textbf{666} & 794 & 795 & 866 & 894 & 776 & 746 & 746 & 746 \\ 
   & squared & \textbf{10425} & 11150 & 15954 & 13368 & 12790 & 12002 & 11995 & 11995 & 11995 \\ 
  mg & absolute & 20 & 15 & \textbf{14} & 30 & 36 & 28 & 28 & 28 & 28 \\ 
   & squared & 9 & 6 & \textbf{4} & 13 & 13 & 12 & 12 & 12 & 12 \\ 
  space\_ga & absolute & 99 & 108 & 108 & 113 & 109 & \textbf{60} & \textbf{60} & \textbf{60} & \textbf{60} \\ 
   & squared & \textbf{40} & 43 & 43 & \textbf{40} & 45 & 45 & 45 & 45 & 45 \\
   \hline
\end{tabular}
}
\caption{The regret of each algorithm for the various data sets and loss
functions, with the $\sigma$ hyperparameter of each method optimized in
hindsight for the data. Boldface indicates the regret differs less than
$1$ from the minimum regret for the row.}
\label{tab:hypertune.regret.table}
\end{table}
\clearpage

\DeclareRobustCommand{\VAN}[3]{#3} %
\bibliography{MG,../bib}

\DeclareRobustCommand{\VAN}[3]{#2} %

\end{document}